\documentclass{article}

\newcommand{\loose}{\looseness=-1}

\usepackage[utf8]{inputenc} %
\usepackage[T1]{fontenc}    %
\usepackage{url}            %
\usepackage{booktabs}       %
\usepackage{amsfonts}       %
\usepackage{nicefrac}       %
\usepackage{microtype}      %

\usepackage{tocloft}            %

\usepackage{enumitem}

\usepackage{breakcites}

\usepackage{etoolbox}
\usepackage{comment}
\newtoggle{draft}
\togglefalse{draft}
\newcommand{\draft}[1]{\iftoggle{draft}{#1}{}}

\usepackage{mathrsfs}

\usepackage{algorithm}
\usepackage{verbatim}
\usepackage[noend]{algpseudocode}
\newcommand{\multiline}[1]{\parbox[t]{\dimexpr\linewidth-\algorithmicindent}{#1}}

\usepackage{multicol}

\usepackage{colortbl}

\usepackage{setspace}

\usepackage{transparent}

\usepackage{inconsolata}
\usepackage[scaled=.90]{helvet}
\usepackage{xspace}

\usepackage{pifont}

\usepackage[letterpaper, left=1in, right=1in, top=1in, bottom=1in]{geometry}

\PassOptionsToPackage{hypertexnames=false}{hyperref}  %
\usepackage{parskip}

\usepackage[dvipsnames]{xcolor}
\usepackage[colorlinks=true, linkcolor=blue!70!black, citecolor=blue!70!black,urlcolor=black,breaklinks=true]{hyperref}
\usepackage{microtype}
\usepackage{hhline}

\makeatletter
\newcommand{\neutralize}[1]{\expandafter\let\csname c@#1\endcsname\count@}
\makeatother

\usepackage{algorithm}

\usepackage{natbib}
\bibpunct{(}{)}{;}{a}{,}{,}

\usepackage{amsthm}
\usepackage{mathtools}
\usepackage{amsmath}
\usepackage{bbm}
\usepackage{amsfonts}
\usepackage{amssymb}

\usepackage{xpatch}

\usepackage{thmtools}
\usepackage{thm-restate}
\declaretheorem[name=Theorem,parent=section]{theorem}
\declaretheorem[name=Lemma,parent=section]{lemma}
\declaretheorem[name=Assumption, parent=section]{assumption}
\declaretheorem[name=Condition, parent=section]{condition}
\declaretheorem[qed=$\triangleleft$,name=Example,style=definition, parent=section]{example}
\declaretheorem[name=Remark,style=definition, parent=section]{remark}
\declaretheorem[name=Proposition, parent=section]{proposition}

\makeatletter
  \renewenvironment{proof}[1][Proof]%
  {%
   \par\noindent{\bfseries\upshape {#1.}\ }%
  }%
  {\qed\newline}
  \makeatother

\theoremstyle{definition}  %

\newtheorem{corollary}{Corollary}[section]

\theoremstyle{plain}
\newtheorem{definition}{Definition}[section]

\xpatchcmd{\proof}{\itshape}{\normalfont\proofnameformat}{}{}
\newcommand{\proofnameformat}{\bfseries}

\usepackage[nameinlink,capitalize]{cleveref}

\newcommand{\pref}[1]{\cref{#1}}
\newcommand{\pfref}[1]{Proof of \pref{#1}}

\renewcommand{\eqref}[1]{\texorpdfstring{\hyperref[#1]{Eq. (\ref*{#1})}}{Eq. (\ref*{#1})}}

\crefformat{equation}{#2(#1)#3}
\Crefformat{equation}{#2(#1)#3}

\Crefformat{figure}{#2Figure #1#3}
\Crefname{assumption}{Assumption}{Assumptions}
\Crefformat{assumption}{#2Assumption #1#3}

\usepackage{crossreftools}
\newcommand{\preft}[1]{\crtcref{#1}}
\newcommand{\creftitle}[1]{\crtcref{#1}}

\usepackage{xparse}

\ExplSyntaxOn
\DeclareDocumentCommand{\XDeclarePairedDelimiter}{mm}
 {
  \__egreg_delimiter_clear_keys: %
  \keys_set:nn { egreg/delimiters } { #2 }
  \use:x %
   {
    \exp_not:n {\NewDocumentCommand{#1}{sO{}m} }
     {
      \exp_not:n { \IfBooleanTF{##1} }
       {
        \exp_not:N \egreg_paired_delimiter_expand:nnnn
         { \exp_not:V \l_egreg_delimiter_left_tl }
         { \exp_not:V \l_egreg_delimiter_right_tl }
         { \exp_not:n { ##3 } }
         { \exp_not:V \l_egreg_delimiter_subscript_tl }
       }
       {
        \exp_not:N \egreg_paired_delimiter_fixed:nnnnn 
         { \exp_not:n { ##2 } }
         { \exp_not:V \l_egreg_delimiter_left_tl }
         { \exp_not:V \l_egreg_delimiter_right_tl }
         { \exp_not:n { ##3 } }
         { \exp_not:V \l_egreg_delimiter_subscript_tl }
       }
     }
   }
 }

\keys_define:nn { egreg/delimiters }
 {
  left      .tl_set:N = \l_egreg_delimiter_left_tl,
  right     .tl_set:N = \l_egreg_delimiter_right_tl,
  subscript .tl_set:N = \l_egreg_delimiter_subscript_tl,
 }

\cs_new_protected:Npn \__egreg_delimiter_clear_keys:
 {
  \keys_set:nn { egreg/delimiters } { left=.,right=.,subscript={} }
 }

\cs_new_protected:Npn \egreg_paired_delimiter_expand:nnnn #1 #2 #3 #4
 {%
  \mathopen{}
  \mathclose\c_group_begin_token
   \left#1
   #3
   \group_insert_after:N \c_group_end_token
   \right#2
   \tl_if_empty:nF {#4} { \c_math_subscript_token {#4} }
 }
\cs_new_protected:Npn \egreg_paired_delimiter_fixed:nnnnn #1 #2 #3 #4 #5
 {
  \mathopen{#1#2}#4\mathclose{#1#3}
  \tl_if_empty:nF {#5} { \c_math_subscript_token {#5} }
 }
\ExplSyntaxOff

\XDeclarePairedDelimiter{\supnorm}{
  left=\lVert,
  right=\rVert,
  subscript=\infty
  }

\DeclarePairedDelimiter{\abs}{\lvert}{\rvert} %
\DeclarePairedDelimiter{\brk}{[}{]}
\DeclarePairedDelimiter{\crl}{\{}{\}}
\DeclarePairedDelimiter{\prn}{(}{)}
\DeclarePairedDelimiter{\nrm}{\|}{\|}

\DeclareMathOperator{\En}{\mathbb{E}}

\DeclareMathOperator*{\argmin}{arg\,min} %
\DeclareMathOperator*{\argmax}{arg\,max}

\newcommand{\wt}[1]{\widetilde{#1}}
\newcommand{\wh}[1]{\widehat{#1}}
\newcommand{\wb}[1]{\widebar{#1}}

\def\ddefloop#1{\ifx\ddefloop#1\else\ddef{#1}\expandafter\ddefloop\fi}
\def\ddef#1{\expandafter\def\csname bb#1\endcsname{\ensuremath{\mathbb{#1}}}}
\ddefloop ABCDEFGHIJKLMNOPQRSTUVWXYZ\ddefloop
\def\ddefloop#1{\ifx\ddefloop#1\else\ddef{#1}\expandafter\ddefloop\fi}
\def\ddef#1{\expandafter\def\csname b#1\endcsname{\ensuremath{\mathbf{#1}}}}
\ddefloop ABCDEFGHIJKLMNOPQRSTUVWXYZ\ddefloop
\def\ddef#1{\expandafter\def\csname sf#1\endcsname{\ensuremath{\mathsf{#1}}}}
\ddefloop ABCDEFGHIJKLMNOPQRSTUVWXYZ\ddefloop
\def\ddef#1{\expandafter\def\csname c#1\endcsname{\ensuremath{\mathcal{#1}}}}
\ddefloop ABCDEFGHIJKLMNOPQRSTUVWXYZ\ddefloop
\def\ddef#1{\expandafter\def\csname h#1\endcsname{\ensuremath{\widehat{#1}}}}
\ddefloop ABCDEFGHIJKLMNOPQRSTUVWXYZ\ddefloop
\def\ddef#1{\expandafter\def\csname hc#1\endcsname{\ensuremath{\widehat{\mathcal{#1}}}}}
\ddefloop ABCDEFGHIJKLMNOPQRSTUVWXYZ\ddefloop
\def\ddef#1{\expandafter\def\csname t#1\endcsname{\ensuremath{\widetilde{#1}}}}
\ddefloop ABCDEFGHIJKLMNOPQRSTUVWXYZ\ddefloop
\def\ddef#1{\expandafter\def\csname tc#1\endcsname{\ensuremath{\widetilde{\mathcal{#1}}}}}
\ddefloop ABCDEFGHIJKLMNOPQRSTUVWXYZ\ddefloop
\def\ddefloop#1{\ifx\ddefloop#1\else\ddef{#1}\expandafter\ddefloop\fi}
\def\ddef#1{\expandafter\def\csname scr#1\endcsname{\ensuremath{\mathscr{#1}}}}
\ddefloop ABCDEFGHIJKLMNOPQRSTUVWXYZ\ddefloop

\newcommand{\ind}{\mathbbm{1}}    %

\newcommand{\veps}{\varepsilon}

\newcommand{\ldef}{\vcentcolon=}

\newcommand{\Mia}{M_{\alpha,i}}
  \newcommand{\Mja}{M_{\alpha,j}}
  \newcommand{\fmia}{f\sups{\Mia}}
  \newcommand{\pimia}{\pi\subs{\Mia}}
  \newcommand{\Mtila}{\wt{M}}
\newcommand{\fmtila}{f\sups{\Mtila}}

\newcommand{\cMab}{\cM^{\alpha,\beta}}%
\newcommand{\cMcup}{\cM^{\mathrm{all}}}%

\newcommand{\Mbarnu}{\Mbar_\nu}
\newcommand{\fmbarnu}{f\sups{\Mbarnu}}
\newcommand{\pimbarnu}{\pi\subs{\Mbarnu}}

\newcommand{\cmb}{\cup\crl{\Mbar}}

\newcommand{\alphaupper}{\wb{\Delta}(\gamma,T)}
\newcommand{\alphalower}{\underline{\Delta}(\gamma,T)}
\newcommand{\alphaupperabs}{\wb{\alpha}(\veps,\gamma)}
\newcommand{\alphalowerabs}{\underline{\alpha}(\veps,\gamma)}

\newcommand{\alphaveps}{\alpha(\veps)}

\newcommand{\rhs}{right-hand side\xspace}

\newcommand{\vz}{\vee{}0}

\newcommand{\pnot}{p_0}%

\newcommand{\cMu}{\cM\cup\crl{\Mbar}}

\newcommand{\alphagamma}{\alpha_{\gamma}}
  \newcommand{\alphaequiv}{\alpha(\veps,\gamma)}

\newcommand{\creg}{c_{\reg}}
\newcommand{\Creg}{C_{\reg}}

\newcommand{\falsex}{false exclusion\xspace}
\newcommand{\Falsex}{False exclusion\xspace}

\newcommand{\reg}{\mathrm{reg}}

\newcommand{\vepsl}{\underline{\veps}(T)}
\newcommand{\vepsu}{\wb{\veps}(T)}

\newcommand{\vepslowerT}{\vepsl}
\newcommand{\vepsupperT}{\vepsu}

\newcommand{\Ceff}{C_{\mathrm{prob}}}
\newcommand{\Cprob}{\Ceff}

\newcommand{\Mtil}{\wt{M}}
\newcommand{\fmtil}{f\sups{\Mtil}}
\newcommand{\pimtil}{\pi\subs{\Mtil}}

\newcommand{\picirc}{\pi_{\circ}}

\newcommand{\Enp}[1][p]{\En_{\pi\sim{}#1}}
\newcommand{\Enq}[1][q]{\En_{\pi\sim{}#1}}

\newcommand{\pdec}{\normalfont{\textsf{p-dec}}}
\newcommand{\rdec}{\normalfont{\textsf{r-dec}}}

\newcommand{\decopac}[1][\gamma]{\pdec^{\mathrm{o}}_{#1}}
\newcommand{\decoreg}[1][\gamma]{\rdec^{\mathrm{o}}_{#1}}
\newcommand{\deccpac}[1][\veps]{\pdec^{\mathrm{c}}_{#1}}
\newcommand{\deccreg}[1][\veps]{\rdec^{\mathrm{c}}_{#1}}

\newcommand{\decoregbayes}[1][\gamma]{\underline{\rdec}^{\mathrm{o}}_{#1}}
\newcommand{\decopacbayes}[1][\gamma]{\underline{\pdec}^{\mathrm{o}}_{#1}}

\newcommand{\decopacr}[1][\gamma]{\pdec^{\mathrm{o,rnd}}_{#1}}
\newcommand{\decoregr}[1][\gamma]{\rdec^{\mathrm{o,rnd}}_{#1}}
\newcommand{\deccpacr}[1][\veps]{\pdec^{\mathrm{c,rnd}}_{#1}}
\newcommand{\deccregr}[1][\veps]{\rdec^{\mathrm{c,rnd}}_{#1}}

\newcommand{\deccpacg}[1][\veps]{\pdec^{\mathrm{c,greedy}}_{#1}}

\newcommand{\deccpacalt}[1][\veps]{\wt{\pdec}^{\mathrm{c}}_{#1}}

\newcommand{\cMall}{\cM^{+}}

\newcommand{\cMtil}[1][\alpha]{\wt{\cM}_{#1}}

\newcommand{\RegDM}{\Reg_{\mathsf{DM}}(T)}

\newcommand{\RiskDM}{\mathrm{\mathbf{Risk}}_{\mathsf{DM}}(T)}

\newcommand{\Empi}[1][M]{\En^{\sss{#1},\pi}}

\renewcommand{\pm}[1][M]{p_{\sss{#1}}}
\newcommand{\Ct}{C(T)}

\newcommand{\gmo}{g\sups{M_1}}

\renewcommand{\c}{\mathrm{c}}

\newcommand{\fmp}{f\sups{M'}}

\newcommand{\gammastar}{\gamma^{\star}}

\newcommand{\qmbar}{
  q_{\sMbar}}

\newcommand{\pmbar}{p_{\sMbar}}

\newcommand{\qbar}{\bar{q}}

\renewcommand{\emptyset}{\varnothing}

\newcommand{\filt}{\mathscr{F}}

\newcommand{\hist}{\mathfrak{H}}

\newcommand{\Asig}{\mathscr{P}}
\newcommand{\Rsig}{\mathscr{R}}
\newcommand{\Osig}{\mathscr{O}}

\newcommand{\SA}{\mathscr{A}}
\newcommand{\SB}{\mathscr{B}}

\newcommand{\Hspace}{\Omega}
\newcommand{\Hsig}{\filt}

\newcommand{\abscont}{V(\cM)}

\newcommand{\Framework}{Decision Making with Structured Observations\xspace}
\newcommand{\FrameworkShort}{DMSO\xspace}

\newcommand{\learner}{learner\xspace}

\newcommand{\Cloc}{C_{\mathrm{loc}}}

\newcommand{\act}{\pi}
\newcommand{\Act}{\Pi}

\newcommand{\obs}{o}
\newcommand{\Obs}{\mathcal{\cO}}
\newcommand{\ObsSpace}{\mathcal{\cO}}

\newcommand{\RewardSpace}{\cR}

\newcommand{\Rspace}{\RewardSpace}

\newcommand{\CompText}{Decision-Estimation Coefficient\xspace}
\newcommand{\CompAbbrev}{DEC\xspace}
\newcommand{\CompShort}{\CompAbbrev}

\newcommand{\etdtext}{Estimation-to-Decisions\xspace}

\newcommand{\etdp}{\textsf{E\protect\scalebox{1.04}{2}D}$^{+}$\xspace}

\newcommand{\M}[1]{^{{\scriptscriptstyle M}}}  %

\newcommand{\sMbar}{\sss{\Mbar}}
 
\newcommand{\sups}[1]{^{{\scriptscriptstyle#1}}}
\newcommand{\subs}[1]{_{{\scriptscriptstyle#1}}}

\newcommand{\sss}[1]{{\scriptscriptstyle#1}}

\newcommand{\Enm}[1][M]{\En^{\sss{#1}}}

\newcommand{\bbPm}[1][M]{\bbP\sups{#1}}
\newcommand{\bbPmbar}[1][\Mbar]{\bbP\sups{#1}}

\newcommand{\fm}[1][M]{f\sups{#1}}
\newcommand{\pim}[1][M]{\pi_{\sss{#1}}}

\newcommand{\gm}{g\sups{M}}
\newcommand{\gmbar}{g\sups{\Mbar}}

\newcommand{\fmbar}{f\sups{\Mbar}}
\newcommand{\pimbar}{\pi\subs{\Mbar}}

\newcommand{\fmstar}{f\sups{\Mstar}}
\newcommand{\pimstar}{\pi\subs{\Mstar}}

\newcommand{\pihat}{\wh{\pi}}

\newcommand{\cMloc}[1][\alpha]{\cM_{#1}}
\newcommand{\cMinf}[1][\alpha]{\cM^{\infty}_{#1}}

\newcommand{\Mbar}{\wb{M}}

\newcommand{\Rm}[1][M]{R\sups{#1}}
\newcommand{\Pm}[1][M]{P\sups{#1}}

\newcommand{\Reg}{\mathrm{\mathbf{Reg}}}

\newcommand{\Est}{\mathrm{\mathbf{Est}}_{\mathsf{H}}}
\newcommand{\EstBar}{\widebar{\mathrm{\mathbf{Est}}}_{\mathsf{H}}}

\newcommand{\EstHel}{\mathrm{\mathbf{Est}}_{\mathsf{H}}(T)}

\newcommand{\EstProbHel}{\mathrm{\mathbf{Est}}_{\mathsf{H}}(T,\delta)}

  \newcommand{\AlgEst}{\mathrm{\mathbf{Alg}}_{\mathsf{Est}}}

\newcommand{\Mhat}{\wh{M}}
\newcommand{\Mstar}{M^{\star}}

\newcommand{\algcommentlight}[1]{\textcolor{blue!70!black}{\transparent{0.5}\footnotesize{\texttt{\textbf{//\hspace{2pt}#1}}}}}
\newcommand{\algcommentbig}[1]{\textcolor{blue!70!black}{\footnotesize{\texttt{\textbf{/*
          #1~*/}}}}}

\newcommand{\approxleq}{\lesssim}
\newcommand{\approxgeq}{\gtrsim}

\renewcommand{\ind}[1]{^{{\scriptscriptstyle#1}}}

\newcommand{\bigoh}{O}
\newcommand{\bigoht}{\wt{O}}
\newcommand{\bigom}{\Omega}
\newcommand{\bigomt}{\wt{\Omega}}
\newcommand{\bigthetat}{\wt{\Theta}}

\newcommand{\indic}{\mathbb{I}}

\newcommand{\Dhel}[2]{D_{\mathsf{H}}\prn*{#1,#2}}

\newcommand{\Dhels}[2]{D^{2}_{\mathsf{H}}\prn*{#1,#2}}
\newcommand{\hell}[2]{D^{2}_{\mathsf{H}}\prn*{#1,#2}}
\newcommand{\tvd}[2]{D_{\mathsf{TV}}\prn*{#1,#2}}

\newcommand{\Dtv}[2]{D_{\mathsf{TV}}\prn*{#1,#2}}
\newcommand{\Dtvs}[2]{D^2_{\mathsf{TV}}\prn*{#1,#2}}

\newcommand{\DhelsX}[3]{D^{2}_{\mathsf{H}}\prn[#1]{#2,#3}}

\newcommand{\Ber}{\mathrm{Ber}}

\newcommand{\conv}{\mathrm{co}}

\newcommand{\Qstar}{Q^{\star}}

\newcommand{\unif}{\mathrm{Unif}}

\newcommand{\astar}{a^{\star}}

\newcommand{\mathand}{\quad\text{and}\quad}

\def\multiset#1#2{\ensuremath{\left(\kern-.3em\left(\genfrac{}{}{0pt}{}{#1}{#2}\right)\kern-.3em\right)}}

\renewcommand{\emptyset}{\varnothing}

\newcommand{\pbar}{\bar{p}}
\newcommand{\phat}{\wh{p}}

\newcommand{\pir}{\pi_{\circ}}%

\newcommand{\nc}{\newcommand}
\nc{\DMO}{\DeclareMathOperator}
\newcount\Comments  %
\DMO{\prox}{prox}
\DMO{\Span}{span}
\DMO{\UCB}{UCB}
\DMO{\LCB}{LCB}
\nc{\br}[2]{{\rm br}^{#1}({#2})}
\nc{\depth}[1]{{\rm d}({#1})}
\nc{\child}[2]{{\rm ch}_{#1}({#2})}
\nc{\parent}[1]{{\rm pa}({#1})}
\nc{\dg}{\dagger}
\nc{\indsig}[2]{\mathcal{I}_{#1}({#2})}
\nc{\total}{{\rm fin}}
\nc{\early}{{\rm pre}}
\nc{\zsink}{z_{\rm sink}}
\nc{\lowv}{{\rm low}}
\nc{\ol}{\overline}
\nc{\madec}[3]{\texttt{ma-dec}_{#1}({#2}, {#3})}
\nc{\madeco}[1]{\texttt{ma-dec}_{#1}}
\nc{\madecd}[3]{\texttt{ma-dec}^{\texttt{d}}_{#1}({#2}, {#3})}
\nc{\mainf}{\texttt{ma-inf}}
\nc{\maexo}{\texttt{ma-exo}}
\nc{\dec}{\texttt{dec}}
\nc{\decc}{\texttt{dec}^{\texttt{c}}}
\nc{\deccp}{\texttt{dec}^{\texttt{c-pac}}}
\nc{\deccr}{\texttt{dec}^{\texttt{c-reg}}}
\nc{\Alg}{{\rm\bf Alg}}
\nc{\co}{{\rm co}}
\nc{\BV}{\mathbb{V}}
\nc{\ham}[2]{d_{\rm Ham}({#1}, {#2})}

\nc{\gamvec}{\gamma}
\nc{\til}{\widetilde}
\nc{\td}{\tilde}
\nc{\todo}[1]{\ifnum\Comments=1 {\color{red}  [TODO: #1]}\fi}
\nc{\old}[1]{\ifnum\Comments=1 {\color{brown}  [OLD: #1]}\fi}
\newcommand{\noah}[1]{\ngcomment{#1}}
\nc{\BP}{\mathbb{P}}
\nc{\BI}{\mathbb{I}}

\nc{\fools}[3]{\MF_{#3}({#1}, {#2})}
\nc{\fool}[2]{\MF({#1},{#2})}
\nc{\clip}[2]{{\rm clip}\left[ \left. {#1} \right| {#2} \right]}
\nc{\imax}{\omega}
\nc{\CF}{\mathscr{F}}
\nc{\CG}{\mathscr{G}}
\nc{\CA}{\mathscr{A}}
\nc{\MH}{\mathcal{H}}
\nc{\MV}{\mathcal{V}}
\nc{\MC}{\mathcal{C}}
\nc{\MI}{\mathcal{I}}
\nc{\MQ}{\mathcal{Q}}
\nc{\st}{\star}
\nc{\lng}{\langle}
\nc{\rng}{\rangle}
\DMO{\OOPT}{opt}
\nc{\dopt}[2]{\ell_{\OOPT}({#1},{#2})}
\nc{\MG}{\mathcal{G}}
\nc{\MP}{\mathcal{P}}
\nc{\PP}{\mathbb{P}}
\nc{\TT}{\mathbb{T}}
\nc{\TTmax}{\TT_{\max}}
\DMO{\REG}{Reg}
\DMO{\WREG}{wReg}
\nc{\wreg}[2]{{\Delta}^{\rm w}_{{#1}}({#2})}
\nc{\wReg}[2]{{\WREG}_{{#1}}({#2})}
\DMO{\Ham}{Ham}
\DMO{\Gap}{Gap}
\DMO{\GD}{GD}
\DMO{\GDA}{GDA}
\DMO{\EG}{EG}
\DMO{\OGDA}{OGDA}
\DMO{\Unif}{Unif}
\DMO{\Tr}{Tr}
\nc{\Qu}{\ul{Q}}
\nc{\Qo}{\ol{Q}}
\nc{\Ro}{\ol{R}}
\nc{\Vu}{\ul{V}}
\nc{\Vo}{\ol{V}}
\nc{\RanQ}{\Delta Q}
\nc{\RanV}{\Delta V}
\nc{\clipQ}{\Delta \breve{Q}}
\nc{\frzQ}{\Delta \mathring{Q}}
\nc{\clipV}{\Delta \breve{V}}
\nc{\clipdelta}{\breve{\delta}}
\nc{\cliptheta}{\breve{\theta}}
\nc{\delmin}{\Delta_{{\rm min}}}
\nc{\delmins}[1]{\Delta_{{\rm min},{#1}}}
\nc{\gapfinal}[1]{\max \left\{ \frac{\frzQ_{{#1}}^{k^\st}(x,a)}{2H}, \frac{\delmin}{4H} \right\}}
\nc{\post}[2]{R({#1}; {#2})}
\nc{\posts}[3]{R_{#3}({#1}; {#2})}
\nc{\pstr}{{\rm po}}
\nc{\prior}{{\rm pr}}

\nc{\algnst}[1]{\begin{align*}#1\end{align*}}
\nc{\algn}[1]{\begin{align}#1\end{align}}
\nc{\matx}[1]{\left(\begin{matrix}#1\end{matrix}\right)}
\renewcommand{\^}[1]{^{\scriptscriptstyle#1}}

\nc{\nuu}{\nu}

\nc{\bel}[1]{\mathbf{b}({#1})}
\nc{\nbel}[1]{\bar{\mathbf{b}}({#1})}
\nc{\sbel}[2]{\mathbf{b}'_{#1}({#2})}
\nc{\nsbel}[2]{\bar{\mathbf{b}}'_{#1}({#2})}

\nc{\bone}{\mathbf{1}}

\nc{\MO}{\mathcal O}
\nc{\MU}{\mathcal{U}}
\nc{\ME}{\mathcal{E}}
\nc{\MN}{\mathcal{N}}
\nc{\MK}{\mathcal{K}}
\nc{\MM}{\mathcal{M}}
\nc{\ML}{\mathcal{L}}
\nc{\MS}{\mathcal{S}}
\nc{\MT}{\mathcal{T}}
\nc{\BF}{\mathbb F}
\nc{\BQ}{\mathbb Q}
\nc{\MX}{\mathcal{X}}
\nc{\MA}{\mathcal{A}}
\nc{\MD}{\mathcal{D}}
\nc{\MB}{\mathcal{B}}
\nc{\MZ}{\mathcal{Z}}
\nc{\MJ}{\mathcal{J}}
\nc{\MW}{\mathcal{W}}
\nc{\MR}{\mathcal{R}}
\nc{\MY}{\mathcal{Y}}
\nc{\BZ}{\mathbb Z}
\nc{\BN}{\mathbb N}
\nc{\ep}{\epsilon}
\nc{\vep}{\varepsilon}
\nc{\gapfn}[1]{\varepsilon_{#1}}
\nc{\ggapfn}[2]{\varphi_{#1}({#2})}
\nc{\epsahk}{\gapfn{0}}
\nc{\BH}{\mathbb H}
\nc{\BG}{\mathbb{G}}
\nc{\D}{\Delta}
\nc{\MF}{\mathcal{F}}
\nc{\One}[1]{\mathbbm{1}\left\{{#1}\right\}}
\nc{\bOne}{\mathbf{1}}
\nc{\Aopt}{\mathcal{A}^{\rm opt}}
\nc{\Amul}{\mathcal{A}^{\rm mul}}
\nc{\CM}{\mathscr{M}}
\nc{\CO}{\mathscr{O}}
\nc{\CR}{\mathsscr{R}}

\nc{\SP}{\mathsf P}
\nc{\SQ}{\mathsf Q}
\nc{\SC}{\mathscr{C}}
\nc{\SD}{\mathscr{D}}
\nc{\SE}{\mathscr{E}}
\nc{\SG}{\mathscr{G}}

\nc{\DO}{\accentset{\circ}{\D}}
\nc{\mf}{\mathfrak}
\nc{\mfp}{\mathfrak{p}}
\nc{\mfq}{\mf{q}}
\nc{\Sp}{\mbox{Spec}}
\nc{\Spm}{\mbox{Specm}}
\nc{\hookuparrow}{\mathrel{\rotatebox[origin=c]{90}{$\hookrightarrow$}}}
\nc{\hookdownarrow}{\mathrel{\rotatebox[origin=c]{-90}{$\hookrightarrow$}}}
\nc{\hra}{\hookrightarrow}
\nc{\tra}{\twoheadrightarrow}
\nc{\sgn}{{\rm sgn}}
\nc{\aut}{{\rm Aut}}
\nc{\Hom}{{\rm Hom}}
\nc{\img}{{\rm Im}}
\DMO{\id}{Id}
\DMO{\KL}{KL}
\nc{\kld}[2]{\KL({#1}||{#2})}
\nc{\ren}[3]{D_{#3}({#1}||{#2})}
\nc{\chisq}[2]{\chi^2({#1},{#2})}
\nc{\dvg}[2]{D({#1} \| {#2})}
\DMO{\BSS}{BSS}
\DMO{\BES}{BES}
\DMO{\BGS}{BGS}
\nc{\indep}{\perp}
\DMO{\sink}{sink}
\nc{\fp}[1]{\MP_1({#1})}
\nc{\BO}{\mathbb{O}}
\nc{\BT}{\mathbb{T}}

\nc{\RR}{\mathbb{R}}
\nc{\Gradient}{\nabla}
\nc{\norm}[1]{\left \lVert #1 \right \rVert}
\nc{\EE}{\mathbb{E}}

\DMO{\PR}{Pr}

\nc{\E}{\mathbb{E}}
\nc{\ra}{\rightarrow}

\nc{\opo}{\texttt{opo}}

\input{widebar}

\usepackage[suppress]{color-edits}
 \addauthor{df}{ForestGreen}
 \addauthor{dk}{magenta}
 \addauthor{jq}{yellow!60!black}
 \addauthor{ng}{purple}
 \draft{

 \usepackage[textsize=tiny]{todonotes}
 
 }
\usepackage[final]{showlabels}

\makeatletter
\let\OldStatex\Statex
\renewcommand{\Statex}[1][3]{%
  \setlength\@tempdima{\algorithmicindent}%
  \OldStatex\hskip\dimexpr#1\@tempdima\relax}
\makeatother

 \usepackage{accents}

 \addtocontents{toc}{\protect\setcounter{tocdepth}{0}}

\let\oldparagraph\paragraph
\renewcommand{\paragraph}[1]{\oldparagraph{#1.}}

\newcommand{\paragraphi}[1]{\par\noindent\emph{#1.}}

\title{Improved Guarantees for Interactive Decision Making\\with the
  Decision-Estimation Coefficient}
\title{Tight Guarantees for Interactive Decision Making\\with the
  Decision-Estimation Coefficient}
\author{%
  Dylan J. Foster\\{\small \texttt{dylanfoster@microsoft.com}} \and Noah Golowich\thanks{Supported by a Fannie \& John Hertz Foundation Fellowship and an NSF Graduate Fellowship.}\\{\small \texttt{nzg@mit.edu}}\\ \and Yanjun Han\\{\small \texttt{yjhan@mit.edu}}
}

\date{January 19, 2023}

\begin{document}
\maketitle
\begin{abstract}
A foundational problem in reinforcement learning and interactive
decision making is to understand what modeling assumptions lead to
sample-efficient learning guarantees, and what algorithm design
principles achieve optimal sample
complexity. Recently, \citet{foster2021statistical} introduced the
\CompText (\CompShort), a
measure of statistical complexity which leads to upper and lower bounds on the
optimal sample complexity for a general class of problems encompassing bandits and reinforcement learning with function approximation.
In
this paper, we introduce a new variant of the \CompShort, the
\emph{Constrained \CompText}, and use it to derive new lower bounds that
improve upon prior work on three fronts:
\begin{itemize}
        \item they hold in expectation, with no restrictions on the
          class of algorithms under consideration.
        \item they hold globally, and do not rely on the notion of
          \emph{localization} used by \citet{foster2021statistical}.
        \item most interestingly, they allow the \emph{reference
            model} with respect to which the
          \CompShort is defined to be \emph{improper}, establishing that
          improper reference models play a fundamental role.
        \end{itemize}
We provide upper bounds on regret that scale with the same
quantity, thereby closing all but one of the gaps between upper
and lower bounds in \citet{foster2021statistical}. Our results
apply to both the regret framework and PAC framework, and make use of
several new analysis and algorithm design techniques that we
anticipate will find broader use.

\end{abstract}

  \addtocontents{toc}{\protect\setcounter{tocdepth}{2}}
  {\hypersetup{hidelinks}
    \tableofcontents
  }

\section{Introduction}
\label{sec:intro}

Sample efficiency in reinforcement learning and decision
making is a fundamental challenge. State-of-the-art algorithms \citep{lillicrap2015continuous,mnih2015human,silver2016mastering} often require
millions of rounds of interaction to achieve human-level performance, which is
prohibitive in practice and constitutes a barrier to reliable
deployment. For continued progress on challenging real-world domains
where the agent must navigate high-dimensional state and observation spaces, it is critical to design
algorithms that can take advantage of users' domain knowledge
(via modeling and function approximation) to enable generalization and improved
sample efficiency. As such, a foundational question is to understand what
modeling assumptions lead to sample-efficient learning
guarantees, and what algorithms achieve optimal sample complexity.

The non-asymptotic theory of reinforcement learning is rich with sufficient
conditions under which sample-efficient learning is
possible \citep{dean2020sample,yang2019sample,jin2020provably,modi2020sample,ayoub2020model,krishnamurthy2016pac,du2019latent,li2009unifying,dong2019provably,zhou2021nearly}, as well as structural properties that aim to unify these
conditions
\citep{russo2013eluder,jiang2017contextual,sun2019model,wang2020provably,du2021bilinear,jin2021bellman},
but conditions that are \emph{necessary}
  have generally been elusive. Recently though, \citet{foster2021statistical} introduced the \CompText (\CompShort), a
unified notion of statistical complexity that leads to both upper and
\emph{lower} bounds on the optimal sample complexity in a general decision
making framework. The results of \citet{foster2021statistical} show that the \CompShort plays a role
for interactive decision making analogous to that of the VC dimension and its relatives in statistical learning, but leave room for tighter
quantitative guarantees. In this paper, we introduce a
new variant of the \CompShort, the \emph{Constrained \CompText}, and
use it to close several gaps between the upper and lower bounds in \citet{foster2021statistical}.

\subsection{Interactive Decision Making}

We consider the \emph{\Framework} (\FrameworkShort) framework of
\citet{foster2021statistical}, a general setting for
interactive decision making that encompasses bandits problems
(including structured and
contextual bandits) and reinforcement learning with function
approximation. 

The protocol consists of $T$ rounds. For each round $t=1,\ldots,T$:
  \begin{enumerate}
  \item The \learner selects a \emph{decision} $\act\ind{t}\in\Act$,
    where $\Act$ is the \emph{decision space}.
  \item The learner receives a reward $r\ind{t}\in\cR\subseteq\bbR$
    and observation $o\ind{t}\in\cO$ sampled via
    $(r\ind{t},o\ind{t})\sim{}\Mstar(\pi\ind{t})$, where
    $\Mstar:\Pi\to\Delta(\cR\times\cO)$ is the underlying \emph{model}.
    
  \end{enumerate}
  We refer to $\cR$ as the \emph{reward space} and $\cO$ as the
  \emph{observation space}.
    The model $\Mstar$, which we formalize as a conditional
    distribution, plays the role of the underlying
    environment, and is unknown to the learner. However, the learner
    is assumed to have access to a \emph{model class} $\cM\subset(\Pi\to\Delta(\cR\times\cO))$ that
    contains $\Mstar$.
    \begin{assumption}[Realizability]
  \label{ass:realizability}
  The learner has access to model class $\cM$ that contains the true model $\Mstar$.
\end{assumption}
The model class $\cM$ represents the learner's prior knowledge of the
underlying environment. For structured bandit problems, where models
  correspond to reward distributions, it encodes
structure in the reward landscape (smoothness, linearity, convexity,
etc.), and for reinforcement learning problems, where models correspond to
  Markov decision processes (MDPs), it typically encodes
structure in the transition probabilities or value functions. In more detail:
    \begin{itemize}
    \item \textbf{Bandits.} In bandit problems, $\cM$ is a reward
      distribution, $\act\ind{t}$ is
     referred to as an \emph{action} or \emph{arm} and $\Act$
    is referred to as the \emph{action space}; there are no
    observations beyond rewards ($\ObsSpace=\crl*{\emptyset}$). By
    choosing $\cM$ so that the mean reward function $\fm$ has
    appropriate structure, one can capture bandit problems with continuous or infinite action spaces and
structured rewards, including linear bandits
    \citep{dani2007price,abernethy2008competing,bubeck2012towards},
    bandit convex optimization
    \citep{kleinberg2004nearly,flaxman2005online,bubeck2017kernel,lattimore2020improved},
    and nonparametric bandits
    \citep{kleinberg2004nearly,bubeck2011x,magureanu2014lipschitz}.

  \item \textbf{Reinforcement learning.} In episodic reinforcement learning,
    each model $M\in\cM$ specifies a non-stationary horizon-$H$  Markov decision process
  $M=\crl*{\crl{\cS_h}_{h=1}^{H}, \cA, \crl{\Pm_h}_{h=1}^{H}, \crl{\Rm_h}_{h=1}^{H},
    d_1}$, where $\cS_h$ is the state space for layer $h$, $\cA$ is the action space,
  $\Pm_h:\cS_h\times\cA\to\Delta(\cS_{h+1})$ is the probability transition
  kernel for layer $h$, $\Rm_h:\cS_h\times\cA\to\Delta(\bbR)$ is
  the reward distribution for layer $h$, and $d_1\in\Delta(\cS_1)$ is the initial
  state distribution. Decisions 
  $\pi=\crl{\pi_h:\cS_h\to\Delta(\cA)}_{h=1}^{H}$ are \emph{policies} (mappings
    from states to actions). Given a policy $\pi$, an episode proceeds
    as follows (beginning from
  $s_1\sim{}d_1$). For $h=1,\ldots,H$:
  \begin{itemize}
  \item $a_h\sim\pi_h(s_h)$.
  \item $r_h\sim{}R_h\sups{M}(s_h,a_h)$ and $s_{h+1}\sim{}P_h\sups{M}(\cdot\mid{}s_h,a_h)$.
  \end{itemize}
  This process leads to $(r,o)\sim{}M(\pi)$, where
  $r=\sum_{h=1}^{H}r_h$ is the
    cumulative reward in the episode, and the observation $o=(s_1,a_1,r_1),\ldots(s_H,a_H,r_H)$
    is the episode's trajectory (sequence of observed states, actions,
    and rewards). By choosing $\cM$ appropriately, one can encompass standard classes
  of MDPs (e.g., tabular MDPs or linear systems) \citet{dean2020sample,yang2019sample,jin2020provably,modi2020sample,ayoub2020model,krishnamurthy2016pac,du2019latent,li2009unifying,dong2019provably}, as well as more
  general structural conditions \citep{jiang2017contextual,sun2019model,wang2020provably,du2021bilinear,jin2021bellman}.

  \end{itemize}

For a model $M\in\cM$, $\Empi[M]\brk*{\cdot}$ denotes expectation
  under the process $(r,\obs)\sim{}M(\pi)$,
  $\fm(\pi)\ldef{}\Empi[M]\brk*{r}$ denotes the mean reward function,
  and $\pim\ldef{}\argmax_{\act\in\Act}\fm(\act)$ denotes the optimal decision.

We consider two types of guarantees for interactive decision making: regret
  guarantees and PAC
  (Probably Approximately Correct) guarantees. For regret guarantees,
  we are concerned with the \emph{cumulative suboptimality} given by
  \begin{equation}
    \label{eq:regret}
    \RegDM\ldef\sum_{t=1}^{T}\En_{\pi\ind{t}\sim{}p\ind{t}}\brk*{\fmstar(\pimstar)-\fmstar(\pi\ind{t})},
  \end{equation}
  where $p\ind{t}$ is the learner's randomization distribution
  (conditional distribution over $\pi\ind{t}$) for
  round $t$. For PAC guarantees, we are only concerned with \emph{final suboptimality}. After rounds $t=1,\ldots,T$
  complete, the learner can use all of the data collected to select a
  final decision $\pihat$ (which may be randomized according to a distribution $p \in \Delta(\Pi)$), and we measure performance via
  \begin{equation}
    \label{eq:pac}
    \RiskDM\ldef{}\En_{\pihat\sim{}p}\brk*{\fmstar(\pimstar) - \fmstar(\pihat)}.
  \end{equation}

  We refer to \citet{foster2021statistical,foster2022complexity} for
  further background and measure-theoretic details, as well as
  additional examples and discussion.

  \subsection{Background: Decision-Estimation Coefficient}

  To motivate our results, let us first recall the \CompText of
  \citet{foster2021statistical}; for this discussion, we restrict our attention to regret
  guarantees. Define the squared Hellinger distance for probability measures $\bbP$ and $\bbQ$ with a common
dominating measure $\nu$ by
\begin{equation}
  \label{eq:hellinger}
  \Dhels{\bbP}{\bbQ}=\int\prn[\bigg]{\sqrt{\frac{d\bbP}{d\nu}}-\sqrt{\frac{d\bbQ}{d\nu}}}^{2}d\nu.
\end{equation}

For a model class $\cM$, reference model $\Mbar:\Pi\to\Delta(\cR\times\cO)$, and scale
parameter $\gamma>0$, the \CompText is given by
\begin{align}
  \label{eq:dec_offset}
  \decoreg(\cM,\Mbar)=
  \inf_{p\in\Delta(\Pi)}\sup_{M\in\cM}\En_{\pi\sim{}p}\brk*{
  \fm(\pim) - \fm(\pi)
  - \gamma\cdot\Dhels{M(\pi)}{\Mbar(\pi)}
  }.
\end{align}
Here, we depart slightly from the notation in
\citet{foster2021statistical} and append the prefix $\textsf{r-}$ (indicating ``regret'') and the
superscript ``o'' (indicating ``offset'') to distinguish from other
variants that will be introduced shortly.

\citet{foster2021statistical} (see also
\citet{foster2022complexity}) use the \CompShort to provide upper and
lower bounds on the optimal regret for interactive decision
making:
\begin{itemize}
  \item On the upper bound side, there exists an algorithm
    (\emph{Estimation-to-Decisions}) that obtains
    \begin{equation}\En\brk*{\RegDM}\approxleq{}\inf_{\gamma}\max\crl[\bigg]{\sup_{\Mbar\in\conv(\cM)}\decoreg(\cM,\Mbar)\cdot{}T+\gamma\cdot\EstHel}.\label{eq:upper_old_simple}
    \end{equation}
    Here, the term $\EstHel$ represents the
sample complexity required to perform statistical estimation with the
class, and has $\EstHel\leq\log\abs{\cM}$ for the special case of finite classes.
  \item On the lower bound side, \emph{any algorithm} must have
    \begin{equation}\En\brk*{\RegDM}\approxgeq\inf_{\gamma}\max\crl[\bigg]{\sup_{\Mbar\in\cM}\decoreg(\cM',\Mbar)\cdot{}T+\gamma},
      \label{eq:lower_old_simple}
      \end{equation}
  where $\cM'$ is a certain ``localized''
  subclass of $\cM$ (roughly speaking, models for which $\nrm{\fm-\fmbar}_{\infty}\approxleq\frac{\gamma}{T}$).
\end{itemize}
While the results of \citet{foster2021statistical} show that the \CompShort characterizes
learnability for various classes of models (for instance, convex
classes with low ``estimation complexity''),
the quantitative rates leave room for improvement. The aim of this paper is to address the following gaps:
\begin{itemize}
\item To provide lower bounds that hold in expectation (as opposed to
  low probability) the lower bound \pref{eq:lower_old_simple}
  restricts to a localized subclass of $\cM'\subseteq\cM$, which may have lower complexity than the original
  class. This yields reasonable results for many special
  cases, but can be arbitrarily loose in general.
    \item The lower bound \pref{eq:lower_old_simple} restricts to reference models
  $\Mbar\in\cM$, yet the upper bound \pref{eq:upper_old_simple}
  maximizes over reference
  models $\Mbar\in\conv(\cM)$, potentially leading to larger regret. For
  example, Proposition 3.1 of \citet{foster2021statistical} gives a
  model class $\cM$ for which this distinction leads to an arbitrarily
  large gap between the upper and lower bounds.
\end{itemize}
A third gap, which we do not address, is the presence of the
estimation complexity $\EstHel$ in \eqref{eq:upper_old_simple}, which
is not present in the lower bound. See \citet{foster2021statistical,foster2022note} for
further discussion around this issue.

    \subsection{Constrained Decision-Estimation Coefficient and
      Overview of Results}

    To address the issues in the prequel, we introduce a new
    complexity measure, the
    \emph{Constrained \CompText}. For a reference model
    $\Mbar:\Pi\to\Delta(\cR\times\cO)$, define
    , defined via
\begin{align}
  \label{eq:dec_constrained}
  \deccreg(\cM,\Mbar)=
  \inf_{p\in\Delta(\Pi)}\sup_{M\in\cM}\crl*{\En_{\pi\sim{}p}\brk*{
  \fm(\pim) - \fm(\pi)}
  \mid\En_{\pi\sim{}p}\brk*{\Dhels{M(\pi)}{\Mbar(\pi)}}\leq\veps^2
  },
\end{align}
with the convention that the value above is zero whenever the set
\begin{align}
  \label{eq:hellinger_ball}
\cH_{p,\veps}(\Mbar)\ldef{}\crl*{M\in\cM\mid{}\En_{\pi\sim{}p}\brk*{\Dhels{M(\pi)}{\Mbar(\pi)}}\leq\veps^2}
\end{align}
is empty; the superscript ``c'' indicates ``constrained'', and
distinguishes from the offset counterpart. Our main quantity of
interest is
\begin{equation}
  \label{eq:dec_constrained_max}
  \deccreg(\cM) = \sup_{\Mbar\in\conv(\cM)}\deccreg(\cM\cup\crl{\Mbar},\Mbar),
\end{equation}%
where $\conv(\cM)$ denotes the convex hull of the class $\cM$.

The constrained and offset \CompShort differ only in how
the quantity $\En_{\pi\sim{}p}\brk*{\Dhels{M(\pi)}{\Mbar(\pi)}}$,
representing information gain, is incorporated. The constrained
\CompShort places a hard constraint on the information gain, while the
offset \CompShort treats the information
gain as a penalty for the max-player, amounting to a soft
constraint. At first glance, one might expect these quantities to be
equivalent via Lagrangian duality.
It is indeed the case that the offset \CompShort upper bounds the
constrained \CompShort:
\[
\deccreg(\cM,\Mbar)\leq\inf_{\gamma}\crl{\decoreg(\cM,\Mbar)\vee{}0+\gamma\veps^2},\quad\forall{}\Mbar,
\]
but strong duality fails, and the complexity measures are not equivalent
in general; detailed discussion is given in
\pref{sec:properties}. 

\paragraph{Main results for regret framework}
The constrained DEC
possesses a number of useful properties not shared by the offset
DEC, including \emph{implicitly} enforcing a form of localization in
an adaptive fashion (cf. \pref{sec:properties}). By leveraging these
properties, we provide improved lower and upper bounds that close all but one of the gaps between the
bounds in \citet{foster2021statistical}.
Our main result is as follows.
\newtheorem*{thm:informal1}{Theorem (informal)}%
\begin{thm:informal1}
  For any model class $\cM$:
  \begin{itemize}
  \item Lower bound: For a worst-case model in $\cM$, any algorithm must have
    \[
      \En\brk*{\RegDM} \geq \bigomt(1)\cdot\deccreg[\vepsl](\cM)\cdot{}T
    \]
    for $\vepsl=\wt{\Theta}\prn[\big]{\sqrt{1/T}}$,
  \item Upper bound: There exists an algorithm
    (Estimation-to-Decisions$^{+}$\hspace{-1pt}) that achieves
    \[
      \En\brk*{\RegDM} \leq \bigoht(1)\cdot\deccreg[\vepsu](\cM)\cdot{}T
    \]
    for $\vepsu=\wt{\Theta}\prn[\big]{\sqrt{\EstHel/T}} \leq \wt{\Theta}\prn[\big]{\sqrt{\log\abs{\cM}/T}}$.
  \end{itemize}

\end{thm:informal1}
This lower and upper bound are always tighter than the respective
bounds in prior work, and the lower bound in particular improves upon 
\citet{foster2021statistical} on two fronts:
\begin{itemize}
        \item It holds globally, and removes the notion of
        \emph{localization} used by
        \citet{foster2021statistical}. In addition, it holds in expectation, with no restriction on the
        class of algorithms under consideration.
      \item More interestingly, it scales with
        $\deccreg(\cM)=\sup_{\Mbar\in\conv(\cM)}\deccreg(\cM\cup\crl{\Mbar},\Mbar)$
        as opposed to, say, $\sup_{\Mbar\in\cM}\deccreg(\cM,\Mbar)$, showing (in
        tandem with the upper bound) that
        \emph{improper} reference models $\Mbar\in\conv(\cM)$ play a
        fundamental role, and are not simply an artifact of the upper
        bounds in \citet{foster2021statistical}.
      \end{itemize}
      Together, our upper and lower bounds form an important
      step toward building a simple, user-friendly, and unified theory for interactive
      decision making based on the \CompText. The only gap our results
      leave open is the role of the estimation error
      $\EstHel$, a deep issue that necessitates future research.\loose

On the technical side, our lower bounds make use of several
new analysis techniques and structural results that depart sharply
from previous approaches
\citep{foster2021statistical,foster2022complexity}. Our upper bounds
build upon the \etdtext paradigm introduced in
\citet{foster2021statistical}, but the design and analysis are
substantially more sophisticated, as certain properties of the constrained
\CompShort that aid in proving lower bounds lead to non-trivial
challenges in deriving upper bounds. We anticipate that our techniques will
  find broader use, and to this end we provide user-friendly
  tools for working with the constrained \CompShort in \pref{sec:properties}.

  \begin{remark}
    At this point, the reader may wonder why the definition
    \pref{eq:dec_constrained_max} incorporates the set
    $\cM\cup\crl{\Mbar}$, where $\Mbar$ may lie outside the class
    $\cM$. We show in \pref{sec:properties} that this modification,
    despite not being required for the offset \CompShort, plays a
    central role for the constrained \CompShort.
  \end{remark}

  \paragraph{Main results for PAC framework}
Moving from the regret framework to PAC, we work with the following
PAC counterparts to the
offset and constrained DEC:
\begin{align}
  \label{eq:dec_offset_pac}
  \decopac(\cM,\Mbar)=
  \inf_{p,q\in\Delta(\Pi)}\sup_{M\in\cM}\crl*{\En_{\pi\sim{}p}\brk*{
  \fm(\pim) - \fm(\pi)}
  - \gamma\cdot\En_{\pi\sim{}q}\brk*{\Dhels{M(\pi)}{\Mbar(\pi)}}
  }
\end{align}
and
\begin{align}
  \label{eq:dec_constrained_pac}
  \deccpac(\cM,\Mbar)=
  \inf_{p,q\in\Delta(\Pi)}\sup_{M\in\cM}\crl*{\En_{\pi\sim{}p}\brk*{
  \fm(\pim) - \fm(\pi)}
  \mid\En_{\pi\sim{}q}\brk*{\Dhels{M(\pi)}{\Mbar(\pi)}}\leq\veps^2
  },
\end{align}
with the convention that the value in \eqref{eq:dec_constrained_pac} is
zero when $\cH_{q,\veps}(\Mbar)=\emptyset$. These definitions parallel \pref{eq:dec_offset} and
\pref{eq:dec_constrained}, but allow the min-player to select a
separate \emph{exploration distribution} $q$ under which the
information gain (Hellinger distance) is evaluated, and \emph{exploitation distribution} $p$ under
which regret is evaluated. Since one can always choose $p=q$, it is immediate that
$\decopac(\cM,\Mbar)\leq\decoreg(\cM,\Mbar)$, and likewise
$\deccpac(\cM,\Mbar)\leq\deccreg(\cM,\Mbar)$. Defining\footnote{Compared
  the definition of the constrained \CompShort for regret, the definition
  $\deccpac(\cM)=\sup_{\Mbar\in\conv(\cM)}\deccpac(\cM,\Mbar)$ does
  not apply the \CompShort to the class $\cMu$. See
  \pref{sec:reference_regret} for a detailed explanation.}
\begin{equation}
\deccpac(\cM)=\sup_{\Mbar\in\conv(\cM)}\deccpac(\cM,\Mbar),\label{eq:dec_pac_max}
\end{equation}
  we show
that the PAC \CompShort leads to the following lower and upper bounds
on PAC sample complexity.
\newtheorem*{thm:informal2}{Theorem (informal)}%
\begin{thm:informal2}
  For any model class $\cM$:
  \begin{itemize}
  \item Lower bound: For a worst-case model in $\cM$, any PAC algorithm with $T$ rounds of interaction must have
    \[
      \En\brk*{\RiskDM} \geq \bigomt(1)\cdot{}\deccpac[\vepsl](\cM)
    \]
    for $\vepsl=\bigthetat\prn[\big]{\sqrt{1/T}}$.
  \item Upper bound: There exists an algorithm
    (Estimation-to-Decisions$^{+}$\hspace{-1pt}) that achieves
    \[
      \En\brk*{\RiskDM} \leq \bigoht(1)\cdot{}\deccpac[\vepsu](\cM)
    \]
    for $\vepsu=\wt{\Theta}\prn[\big]{\sqrt{\EstHel/T}} \leq \wt{\Theta}\prn[\big]{\sqrt{\log\abs{\cM}/T}}$.
  \end{itemize}

\end{thm:informal2}

\subsection{Preliminaries}
\paragraph{Interactive decision making}
We assume throughout the paper that $\cR=\brk*{0,1}$ unless otherwise stated. We let
$\cMall=\crl{M:\Pi\to\Delta(\brk{0,1}\times\cO)}$ denote the space of
all possible models with rewards in $\brk{0,1}$. We adopt the shorthand
$\gm(\pi)=\fm(\pim) - \fm(\pi)$.

We adopt the same formalism for probability spaces as in
\citet{foster2021statistical,foster2022complexity}. Decisions are
associated with a measurable space $(\Act,\Asig)$, rewards are
associated with the space $(\Rspace,\Rsig)$, and observations are associated with the space $(\Obs,\Osig)$.  The history up to time $t$ is denoted by  $\hist\ind{t}
=(\act\ind{1},r\ind{1},\obs\ind{1}),\ldots,(\act\ind{t},r\ind{t},\obs\ind{t})$. We define
\[
\Hspace\ind{t}=\prod_{i=1}^{t}(\Act\times\Rspace\times\Obs),\mathand\Hsig\ind{t}=\bigotimes_{i=1}^{t}(\Asig\otimes{}\Rsig\otimes\Osig)
\]
so that $\hist\ind{t}$ is associated with the space
$(\Hspace\ind{t},\Hsig\ind{t})$. For a class $\cM$, we define
\[\abscont\ldef\sup_{M,M'\in\cM}\sup_{\act\in\Act}\sup_{A\in{}\Rsig\otimes\Osig}\crl[\big]{\tfrac{M(A\mid\act)}{M'(A\mid{}\act)}}\vee{}e;\]
finiteness of this quantity is not necessary for our any of our
results, but improves our lower bounds (\pref{thm:pac_lower,thm:regret_lower})
by a $\log(T)$ factor.

\paragraph{Divergences}
  For probability distributions $\bbP$ and $\bbQ$ over a measurable space
  $(\Omega,\filt)$ with a common dominating measure, we define the total variation distance as
  \[
    \Dtv{\bbP}{\bbQ}=\sup_{A\in\filt}\abs{\bbP(A)-\bbQ(A)}
    = \frac{1}{2}\int_\Omega\abs{d\bbP-d\bbQ}.
  \]
\subsection{Organization}
\pref{sec:lower} presents our main lower bounds for the regret and PAC
frameworks, and \pref{sec:upper} presents complementary upper
bounds. \pref{sec:properties} establishes structural results
concerning the constrained \CompShort and offset \CompShort, and \pref{sec:related} provides a detailed comparison to
bounds from prior work. We close with additional examples
(\pref{sec:examples}). Unless otherwise stated, proofs are deferred to
the appendix.

\paragraph{Additional notation} For an integer $n\in\bbN$, we let $[n]$ denote the set
  $\{1,\dots,n\}$. For a set $\cZ$, we let
        $\Delta(\cZ)$ denote the set of all probability distributions
        over $\cZ$, and let $\cZ^{\c}$ denote the complement. We adopt standard
        big-oh notation, and write $f=\bigoht(g)$ to denote that $f =
        \bigoh(g\cdot{}\max\crl*{1,\mathrm{polylog}(g)})$. We use $\approxleq$ only in informal statements to emphasize the most notable elements of an inequality.

\section{Lower Bounds}
\label{sec:lower}

In this section, we provide lower bounds based on the constrained
\CompText for the PAC and regret frameworks, then prove the lower bound for PAC,
highlighting the most salient ideas behind the result.

\subsection{Main Results}

We provide minimax lower bounds for interactive decision making, which show for any model class $\cM$
and horizon $T\in\bbN$,
the worst-case regret (resp. PAC sample complexity) for any algorithm is lower
bounded by the constrained DEC for an appropriate choice of the radius parameter
$\veps>0$. Our lower bounds for PAC and regret take a nearly identical form, and differ only in the use of
$\deccpac(\cM)$ versus $\deccreg(\cM)$. To state the results, define $\Ct\ldef\log(T\wedge{}\abscont)$.

\begin{theorem}[Main Lower Bound: PAC]
  \label{thm:pac_lower}
  Let $\vepslowerT \ldef c\cdot\frac{1}{\sqrt{T\Ct}}$, where $c>0$ is
  a sufficiently small numerical constant and $C\ldef{}48\sqrt{2}$. For all $T\in\bbN$ such
  that the condition
  \begin{align}
    \deccpac[\vepslowerT](\MM)\geq{}C\cdot\vepslowerT\label{eq:pacdec-lb}
  \end{align}
  is satisfied, it holds that for any PAC algorithm, there exists a model
  in $\cM$ such that
    \begin{align}
\E\left[\RiskDM\right] \geq \bigom(1) \cdot
      \sup_{\Mbar\in\cMall}\deccpac[\vepslowerT](\cMu,\Mbar)
      \geq \bigom(1) \cdot \deccpac[\vepslowerT](\MM).
  \end{align}
\end{theorem}
\begin{theorem}[Main Lower Bound: Regret]
  \label{thm:regret_lower}
There exist universal constants $C, C' > 0$ and $c, c' > 0$ such that the following holds.  Let $\vepslowerT \ldef c\cdot\frac{1}{\sqrt{T\Ct}}$. %
For all $T\in\bbN$ such
that the condition
    \begin{align}
\label{eq:regret_lower_condition}
\deccreg[\vepslowerT](\MM)\geq{}C \cdot \vepslowerT
  \end{align}
  is satisfied, it holds that for any regret minimization algorithm, there exists a model
  in $\cM$ such that
  
    \begin{align}
    \label{eq:regret_lower}
      \En\brk*{\RegDM}
      &\geq{}
        c'\cdot\sup_{\Mbar\in\cMall}\deccreg[\vepslowerT](\cMu,\Mbar)\cdot{}\frac{T}{\log
        T} - C' \cdot \frac{\sqrt{T}}{\log(T)} \\
      &\geq{}
        c' \cdot\deccreg[\vepslowerT](\cM)\cdot{}\frac{T}{\log
        T}
        - C' \cdot \frac{ \sqrt{T}}{\log(T)}.\nonumber
    \end{align}
  \end{theorem}
  
\pref{thm:pac_lower,thm:regret_lower} show that the constrained
\CompShort is a fundamental limit for interactive decision making. Importantly, these lower bounds remove the notion of
localization required by prior work, and show that the
\CompShort remains a lower bound even if one allows for improper reference
models $\Mbar\in\conv(\cM)$; as a result, they are always tighter than
those found in \citet{foster2021statistical,foster2022complexity}. We
will show in a moment (\pref{sec:upper}) that our lower bounds can achieved
algorithmically, up to a difference in radius that depends on the
estimation capacity for $\cM$ ($\sqrt{\log\abs{\cM}/T}$ versus
$1/\sqrt{T}$ for the case of finite classes). We defer a
detailed comparison to prior work to \pref{sec:related}, and take this time to build
intuition as to the behavior of the lower bounds and give an overview
of our proof techniques.

\begin{remark}
  \label{rem:abscont}
  Whenever $\abscont=\bigoh(1)$, we have
  $\vepslowerT\propto1/\sqrt{T}$ in
  \pref{thm:pac_lower,thm:regret_lower}. In the general case where
  $\abscont$ is not bounded, we have
  $\vepslowerT\propto{}1/\sqrt{T\log(T)}$, and the lower bounds lose a
  logarithmic factor. For most standard model classes, one has
  $\abscont=\infty$, but there exists a subclass $\cM'\subseteq\cM$
  with $V(\cM')=\bigoh(1)$ and $\deccpac(\cM')\approxgeq\deccpac(\cM)$
  (resp. $\deccreg(\cM')\approxgeq\deccreg(\cM)$). In this case, one
  can derive a tighter lower bound with $\vepslowerT\propto1/\sqrt{T}$
  by applying \pref{thm:pac_lower} (resp. \pref{thm:regret_lower}) to
  $\cM'$. See \citet{foster2021statistical} for examples.
\end{remark}

\begin{remark}
  \pref{thm:pac_lower} (resp. \pref{thm:regret_lower}) scales with the
  quantity
  $\sup_{\Mbar\in\cMall}\deccpac(\cMu,\Mbar)\geq\deccpac(\cM)$
  (resp. $\sup_{\Mbar\in\cMall}\deccreg(\cMu,\Mbar)\geq\deccreg(\cM)$),
  which allows for arbitrary reference models
  $\Mbar\notin\conv(\cM)$. We show in \pref{sec:convexity} that
  maximizing over reference models $\Mbar\in\cMall$ does not increase the value of the
  \CompShort beyond what is attained by $\Mbar\in\conv(\cM)$, so this
  result does not contradict our upper bounds. We state the lower
  bounds in this form because 1) our proof works with $\Mbar\in\cMall$
  directly, and does not use the structure of $\conv(\cM)$, and 2)
  allowing for $\Mbar\in\cMall$ often simplifies calculations.
\end{remark}

\paragraph{Understanding the lower bounds}
Let us give a sense for how the lower bounds behave for standard model
classes. We focus on regret, since this is the main
object of interest for prior work.
\begin{itemize}
\item \emph{$\sqrt{T}$-rates.} For the most well-studied
  classes found throughout the literature on bandits and reinforcement
  learning, we have
  \[
    \deccreg(\cM)\propto\veps\cdot\sqrt{\Ceff},
  \]
  where $\Ceff>0$ is a problem-dependent constant that reflects some
  notion of intrinsic complexity. In this case, the condition
  \pref{eq:regret_lower_condition} is satisfied whenever $\Ceff$ is
  larger than some numerical constant, and \pref{thm:regret_lower}
  gives\footnote{This requires $\abscont=\bigoh(1)$ so that $\vepsupperT\propto1/\sqrt{T}$; see \pref{rem:abscont}.}
  \[
\En\brk*{\RegDM} \geq\bigomt\prn[\big]{\sqrt{\Ceff{}\cdot{}T}}.
  \]
  Examples (cf. \pref{sec:examples}) include multi-armed bandits
  with $A$ actions, where $\Ceff\geq{}A$ (leading to $\En\brk*{\RegDM}\geq\bigomt(\sqrt{AT})$), linear bandits in dimension $d$,
  where $\Ceff\geq{}d$ (leading to $\En\brk*{\RegDM}\geq\bigomt(\sqrt{dT})$), and tabular reinforcement learning with $S$ states,
  $A$ actions, and horizon $H$, where $\Ceff\geq{}HSA$ (leading to $\En\brk*{\RegDM}\geq\bigomt(\sqrt{HSAT})$).
\item \emph{Nonparametric rates.} For nonparametric model classes, where the
  optimal regret is of larger order than $\sqrt{T}$, one typically has
\[
  \deccreg(\cM)\propto\veps^{1-\rho}
\]
for some $\rho\in(0,1)$. In this case, the condition in \eqref{eq:regret_lower_condition} is
satisfied whenever $T$ is a sufficiently large constant, and \pref{thm:regret_lower}
gives \[\En\brk*{\RegDM} \geq\bigomt(T^{\frac{1+\rho}{2}}).\] A
standard example is Lipschitz bandits over $\brk{0,1}^{d}$, where we have
$\deccreg(\cM)\propto\veps^{1-\frac{d}{d+2}}$, leading to
$\En\brk*{\RegDM} \geq\bigomt\prn[\big]{T^{\frac{d+1}{d+2}}}$.
\item \emph{Fast rates.} For problems with low noise, such as
  noiseless bandits, the \CompShort
  typically exhibits threshold behavior, with
  \[
    \deccreg(\cM)\propto\indic\crl[\big]{\veps\geq{}1/\sqrt{\Ceff}},
  \]
  where $\Ceff$ is a problem-dependent parameter. For example, if
  $\cM$ consists of multi-armed bandit instances with $\Pi=\crl{1,\ldots,A}$ and
  noiseless, binary rewards, one can take $\Ceff\propto{}A$. For such
  settings, the condition \pref{eq:regret_lower_condition} is satisfied whenever
  $T=\wt{O}\prn*{\Ceff}$, and \pref{thm:regret_lower} gives \[\En\brk*{\RegDM}\geq\bigomt\prn[\big]{\min\crl{\Ceff,T}}.\]
\end{itemize}
For PAC guarantees, the situation is similar to regret, but the lower bounds are
scaled by a $1/T$ factor due to normalization. Briefly:
\begin{itemize}
  \item Whenever
    $\deccpac(\cM)\propto\veps\cdot\sqrt{\Ceff}$,
  \pref{thm:pac_lower} gives
$\En\brk*{\RiskDM} \geq\bigomt\prn*{\sqrt{\frac{\Ceff{}}{T}}}$,
which implies that
$\bigomt\prn*{
    \frac{\Ceff}{\veps^2}
    }$
samples are required to learn an $\veps$-optimal policy.
  
\item Whenever $\deccpac(\cM)\propto\veps^{1-\rho}$
for $\rho\in(0,1)$, \pref{thm:pac_lower}
gives $\En\brk*{\RiskDM} \geq\bigomt(T^{-\frac{(1-\rho)}{2}})$, which
implies that $\bigomt(\veps^{-\frac{2}{1-\rho}})$ samples are required
to learn a $\veps$-optimal policy.
\item Whenever $\deccpac(\cM)\propto\indic\crl{\veps\geq{}1/\sqrt{\Ceff}}$,
  where $\Ceff$ is a problem-dependent parameter,
  \pref{thm:pac_lower} gives
  $\En\brk*{\RiskDM}\geq\bigom(1)$ until $T=\bigomt(\Ceff)$; that
  is, at least $\bigomt(\Ceff)$ samples are required to learn
  beyond constant suboptimality.
\end{itemize}
We refer to \pref{sec:examples} for further examples and details.

\paragraph{Proof techniques: PAC} We now highlight some of the key ideas behind the
proof of \pref{thm:pac_lower}. The high-level structure of the proof is as follows. For any
  algorithm, one can construct a ``hard'' pair of models $M_1,M_2\in\cM$ such
  that:
  \begin{enumerate}
    \item The joint laws of the decisions $\pi\^t$ and observations
    $(r\^t, o\^t)$ induced by the algorithm under $M_1$ and $M_2$ are close in total variation (i.e.,
    $\Dtv{\bbP\sups{M_1}}{\bbP\sups{M_2}}\leq{}1/4$, where $\bbP^{M}$
    denotes the law of $\hist\ind{T}$ under $M$)
    \item Any
    algorithm with risk much smaller
    than the DEC must query
    substantially different decisions in $\Pi$ depending on whether
    the underlying model is $M_1$ or $M_2$.
  \end{enumerate}
  Since  $\bbP\sups{M_1}$ are 
  $\bbP\sups{M_2}$ close enough (in total variation) that the algorithm will fail to
  distinguish the models with constant probability, and since the optimal
  decisions for the models are (approximately) exclusive, the lower
  bound follows.
  
We select the pair of models $(M_1, M_2)$ in an \emph{adversarial}
fashion based on the algorithm under consideration, in a way that
generalizes the approach taken in \citet{foster2021statistical,foster2022complexity}.
We fix an arbitrary model $M_1\in\cM$, then, letting
$q\subs{M_1}\ldef\En\sups{M_1}\brk*{\frac{1}{T}\sum_{t=1}^{T}q\ind{t}(\cdot\mid\hist\ind{t-1})}$
and $p\subs{M_1}\ldef\En\sups{M_1}\brk*{p(\cdot\mid\hist\ind{T})}$
denote the learner's average play under this model, choose $M_2$ as
the model that attains the maximum in \eqref{eq:dec_constrained_pac} with $(p\subs{M_1},q\subs{M_1})$ plugged
in. This approach suffices to prove lower bounds that scale with
$\sup_{\Mbar\in\cM}\deccpac(\cM,\Mbar)$, but is not sufficient to
incorporate improper reference models and prove a lower bound that scales with
$\deccpac(\cM)=\sup_{\Mbar\in\conv(\cM)}\deccpac(\cMu,\Mbar)$. Indeed,
unless the class $\MM$ is convex, there is no reason why an algorithm
with low risk for models in $\cM$ should have low risk for improper
\emph{mixtures} $\Mbar \in \co(\MM)$. Thus, naively choosing $M_1 =
\Mbar \in \co(\MM)$ is problematic, as we have no way to relate the
algorithm's risk under $M_1$ to that for models in the class.

To circumvent this issue, we iterate the process above: Given an
arbitrary, potentially improper reference model $\Mbar\in\cMall$, we first obtain $M_1$ by finding
the model that attains the maximum in \eqref{eq:dec_constrained_pac} with $(p\subs{\Mbar},\qmbar)$ plugged
in. With this model in hand, we obtain $M_2$ in a similar fashion, but
condition on the event the learner behaves near-optimally for
$M_1$. That is, we find the maximizer for \eqref{eq:dec_constrained_pac} with the distribution
$(p\subs{\Mbar}(\cdot\mid{}\cE_1),\qmbar)$ plugged in, where
$\cE_1$ is the set of near-optimal decisions for $M_1$. This argument
leads to lower bounds that scale with $\sup_{\Mbar\in\cMall}\deccpac(\cM,\Mbar)\geq\deccpac(\cM)$ because the reference model $\Mbar$ acts only as a midpoint between
$M_1$ and $M_2$ (whose existence allows us to control the total
variation between the models), and as a result is not required to live in $\cM$.

\paragraph{Proof techniques: Regret}
The proof of our lower bound for regret (\pref{thm:regret_lower})
follows a similar approach to \pref{thm:pac_lower}. However, non-trivial
difficulties arise in applying the iterative conditioning scheme
in the preceding discussion
because there are no longer separate distributions for exploration
($q\subs{M}$) and exploitation ($\pm$), causing the analysis of regret
and information to be coupled. To address this, we adopt a somewhat
different two-part scheme.
\begin{enumerate}
\item First, we prove a lower bound that is similar to \Cref{thm:regret_lower},
  but qualitatively weaker in the sense that the quantity
  $\deccreg[\vepslowerT](\MM) = \sup_{\Mbar \in \cMall}
  \deccreg[\vepslowerT](\MM, \Mbar)$ in \eqref{eq:regret_lower} is
  replaced by $\sup_{\Mbar \in \cM} \deccreg[\vepslowerT](\MM,
  \Mbar)$; that is, the lower bound restricts to proper reference
  models. This is proven using a similar approach to our lower bound
  for PAC.
\item Then, we upgrade this weaker result to the full claim of
  \Cref{thm:regret_lower}, using the following algorithmic result: for any
  model class $\MM$ and any $\Mbar \in \cMall$ (not necessarily in
  $\MM$) satisfying mild technical conditions,
  if there exists an algorithm that achieves expected regret at
  most $R$ with respect to the class $\cM$, then there exists an
  algorithm that achieves regret at most $O(R \cdot \log T)$ with
  respect to the enlarged class $\cM \cup \{ \Mbar \}$. By then choosing
  $\Mbar$ appropriately and applying the proper lower bound from Part
  1 to a
  slightly modified version of the class $\cM \cup \{ \Mbar \}$, we are able to establish \Cref{thm:regret_lower}. 
\end{enumerate}
See \pref{app:lower} for the proof.

\paragraph{Benefits of the constrained \CompShort}

The main technical advantage gained by working with the constrained \CompShort over
  the offset \CompShort is that, by placing a hard constraint on the
  Hellinger distance between models under consideration, we can appeal
  to stronger change-of-measure arguments than those considered in
  prior work; this is key to deriving in-expectation (as opposed to
  low probability) lower bounds. In particular, the radius
  $\vepslowerT\approx{}1/\sqrt{T}$ is the largest possible choice such
  that for any algorithm, one can find a worst-case pair of models for which the total variation distance
  $\Dtv{\bbP\sups{M_1}}{\bbP\sups{M_2}}$ is a \emph{small constant} (say,
  $1/4$). Whenever the total variation distance between the induced
  laws is constant, the algorithm must fail to distinguish
  $M_1$ and $M_2$ with constant probability, which entails large
  regret if the optimal decisions for $M_1$ and $M_2$ are
  significantly different.

\citet{foster2021statistical} emphasize that the \CompText can be
thought of as interactive counterpart to the modulus of continuity in
statistical estimation
\citep{donoho1987geometrizing,donoho1991geometrizingii,donoho1991geometrizingiii}. We
find the constrained \CompShort to be a more direct analogue
than the offset \CompShort: The modulus of continuity places a hard
constraint on Hellinger distance for similar technical reasons, and lower bounds
based on the modulus make use of the same $1/\sqrt{T}$ radius.

\subsection{Proof of PAC Lower Bound (\Cref{thm:pac_lower})}
\label{sec:pac_lower}
In what follows, we prove the PAC lower bound
(\pref{thm:pac_lower}). The proof for regret (\pref{thm:regret_lower}), which is
similar but carries some important technical differences, is deferred to
\pref{app:lower}.

\paragraph{Preliminaries}
Formally, for $T \in \BN$, an \emph{algorithm} for the PAC framework is a
collection of mappings $(p,q) = \prn*{\{ q\^t(\cdot \mid \cdot) \}_{t=1}^T, p(\cdot \mid \cdot
) }$ that (adaptively) draws decisions $\pi\^t \sim
q\^t(\cdot \mid \hist\^{t-1})$ (for $t \in [T]$), and then outputs the final
decision $\wh \pi \sim p(\cdot \mid \hist\^T)$ conditioned on the history
$\hist\^T$. We define $\BP^{\sss{M}, (p,q)}$ as the law of $\hist\^T$
when the underlying model is $M$ and the algorithm is
$(p,q)$. Throughout the proof, we will use the elementary property $\Dtv{\bbP}{\bbQ}\leq\Dhel{\bbP}{\bbQ}$.

\begin{proof}[\pfref{thm:pac_lower}]
  For later reference, we define a constant
  \begin{align}
c_0 = 1/16\nonumber,
  \end{align}
and define $\Ct=2^{8}\cdot\log(T\wedge\abscont)$. Fix $T \in \BN$ and an algorithm $(p,q) = \{ q\^t(\cdot \mid \cdot),
  p(\cdot \mid \cdot ) \}_{t=1}^T$. For each model $M\in\cMall$, we use the abbreviation
  $\BP\sups{M}\equiv\bbP^{\sss{M},(p,q)}$, and write $\E\sups{M}$ for
  the corresponding expectation. In addition, we define \[\pm = \Enm
    \left[  p(\cdot \mid \hist\^T) \right],\mathand q\subs{M} = \Enm \left[ \frac 1T \sum_{t=1}^T q\^t(\cdot \mid \hist\^{t-1}) \right].\]

  \paragraph{Choosing a hard pair of models}
  Fix an arbitrary reference model $\Mbar\in\cMall$ and define $\vep := \frac{1}{10 \sqrt{C(T) \cdot T}}$
  and $\vepslowerT := \vep/\sqrt{2}$. We will prove a lower bound in terms
  of $\deccpac[\vepslowerT](\cM,\Mbar)$. We abbreviate $\delta \ldef \deccpac[\vepslowerT](\MM, \Mbar)$, so that the assumption \Cref{eq:pacdec-lb} gives $\delta \geq 48 \vep$. %
  
  To begin, choose any model $M_1\in\cM$ satisfying:
  \begin{align}
M_1 \in \left\{ M \in \MM \ : \ \E_{\pi \sim q_{\Mbar}}\left[ \hell{M(\pi)}{\Mbar(\pi)} \right] \leq \vep^2 \ \wedge \ \E_{\pi \sim \pmbar} \left[\hell{M(\pi)}{\Mbar(\pi)} \right] \leq \vep^2 \right\}.\label{eq:pac_lb_set}
  \end{align}
We will make use of the fact that, by \pref{lem:pac_constrained_alt}, we have that for
all $p\in\Delta(\Pi)$,
  \begin{align}
    &\dec_{\vep/\sqrt 2}(\MM, \Mbar)  \notag\\
    &\leq \sup_{M \in \MM} \left\{ \E_{\pi \sim p} \left[ \gm(\pi) \right] \ | \ \E_{\pi \sim q_{\Mbar}} \left[\hell{M(\pi)}{\Mbar(\pi)} \right] \leq \vep^2 \ \wedge \ \E_{\pi \sim p} \left[\hell{M(\pi)}{\Mbar(\pi)} \right] \leq \vep^2 \right\} \label{eq:chisq-dec-lb},
  \end{align}
  which implies (along with \eqref{eq:pacdec-lb}) that the set in \eqref{eq:pac_lb_set} is
  non-empty. For any model $M\in\cMall$, define
  \begin{align}
\ME\sups{M} := \left\{ \pi \in \Pi : g^{M}(\pi) \geq c_0 \cdot \delta \right\},\nonumber
  \end{align}
  and define $\MA_1 := \ME\sups{M_1}$. 
  Let $p' \ldef \pmbar(\cdot \mid \MA_1^\c)$, and set
  \begin{align}
M_2 := \argmax_{M \in \MM} \left\{ \E_{\pi \sim p'} \left[ \gm(\pi) \right] \ | \ \E_{\pi \sim q_{\Mbar}} \left[ \hell{M(\pi)}{\Mbar(\pi)} \right] \leq \vep^2 \ \wedge \ \E_{\pi \sim p'} \left[ \hell{M(\pi)}{\Mbar(\pi)}\right] \leq \vep^2 \right\};
  \end{align}
  as with $M_1$, \pref{lem:pac_constrained_alt} implies that this set is non-empty.
  Finally, define $\MA_2 \ldef \ME\sups{M_2} \cap \MA_1^{\c}$. 

  \paragraph{Lower bounding the algorithm's risk}
We now recall Lemma A.13 from \citet{foster2021statistical},
which states that for all models $M$,
\begin{equation}
  \label{eq:hellinger_bound}
    \Dhels{\bbPm}{\bbPmbar}
    \leq \Ct\cdot{}T\cdot\En_{\pi\sim\qmbar}\brk*{\Dhels{M(\pi)}{\Mbar(\pi)}}.
  \end{equation}
  By the data
  processing inequality, this further implies that 
  \begin{align}
    \Dtvs{\pm}{\pmbar}\leq
\Ct\cdot{}T\cdot\En_{\pi\sim\qmbar}\brk*{\Dhels{M(\pi)}{\Mbar(\pi)}}.
  \end{align}
  Since $\E_{\pi \sim q_{\Mbar}} \left[ \hell{M_i(\pi)}{\Mbar(\pi)}
  \right] \leq \vep^2$ for $i \in \{1,2\}$, our choice $\vep \leq
  \frac{1}{10 \cdot \sqrt{T \cdot C(T)}}$ implies
  that \[\tvd{p\subs{M_i}}{\pmbar} \leq
    \frac{1}{10},\quad\text{for}\quad i \in \{1,2\}.\] As a result,
  for each $i \in \{1,2\}$, we have
  \begin{align}
    \E_{\pi \sim p\subs{M_i}} \left[ g\sups{M_i}(\pi) \right] &\geq  c_0 \delta \cdot p\subs{M_i}(\pi \in \ME\sups{M_i}) \\&\geq c_0 \delta \cdot \left( \pmbar(\pi \in \ME\sups{M_i}) - \tvd{\pmbar}{p\subs{M_i}} \right)\nonumber\\
    &\geq c_0 \delta \cdot (\pmbar(\pi \in \ME\sups{M_i}) - 1/10 )\label{eq:mbar-tenth}.
  \end{align}
  Thus, to prove the theorem, it suffices to lower
  bound $\pmbar(\pi \in \ME\sups{M_i})$ by $1/4$ for at least one of $i \in
  \{1,2\}$, which will show that the quantity in \eqref{eq:mbar-tenth}
  is at least $\frac{3c_0\delta}{20}$ (in fact, any constant lower
  bound greater than $1/10$ suffices).

  We assume henceforth that
  $\pmbar(\MA_1^\c) \geq 1/2$, as otherwise we have
  $\pmbar(\ME\sups{M_1}) = \pmbar(\MA_1) \geq 1/2$, in which case the
  result immediately follows from \eqref{eq:mbar-tenth}.
  Before continuing, we note that since $\E_{\pi \sim \pmbar} \left[
    \hell{M_1(\pi)}{\Mbar(\pi)} \right] \leq \vep^2$ and
  \begin{align}
    \E_{\pi \sim \pmbar}\left[ \One{\pi \in \MA_1^{\c}} \cdot \hell{M_2(\pi)}{\Mbar(\pi)} \right] \leq \E_{\pi \sim p'} \left[ \hell{M_2(\pi)}{\Mbar(\pi)} \right] \leq \vep^2,\nonumber
  \end{align}
  the triangle inequality for Hellinger distance implies that
$
\E_{\pi \sim \pmbar} \left[ \One{\pi \in \MA_1^{\c}} \cdot  \hell{M_1(\pi)}{M_2(\pi)} \right] \leq 4\vep^2.
$
Hence, by \eqref{lem:hellinger_to_value} and Jensen's inequality, we have
\begin{align}
  \E_{\pi \sim \pmbar} \left[ \One{\pi \in \MA_1^{\c}} \cdot | f\sups{M_1}(\pi) - f\sups{M_2}(\pi)| \right] \leq \sqrt{\E_{\pi \sim \pmbar} \left[ \One{\pi \in \MA_1^{\c}} \cdot \hell{M_1(\pi)}{M_2(\pi)}\right]} \leq  2\vep.\label{eq:diff-m1m2-q}
  \end{align}

  \paragraph{Lower bounding the gap} To proceed, we will first
  establish that
  \begin{align}
    \label{eq:pac_lb_gap}
    f\sups{M_2}(\pi\subs{M_2}) \geq f\sups{M_1}(\pi\subs{M_1}) + \frac{\delta}{2} \cdot (1-4c_0) - 2\vep.
  \end{align}
   To do this, we note that from the definition of $M_2$,
  \begin{align}
\E_{\pi \sim \pmbar} \left[ \One{\pi \in \MA_1^\c} \cdot g\sups{M_2}(\pi) \right] \geq \pmbar(\MA_1^\c) \cdot \E_{\pi \sim p'} \left[ g\sups{M_2}(\pi) \right] \geq \frac 12 \cdot \deccpac[\vep/\sqrt 2](\MM, \Mbar) = \frac{\delta}{2}\nonumber,
  \end{align}
  where we have used the assumption that $\pmbar(\MA_1^\c) \geq \frac 12$. Thus,
  \begin{align}
    \E_{\pi \sim \pmbar} \left[ \One{\pi \in \MA_2} \cdot g\sups{M_2}(\pi) \right] &= \E_{\pi \sim \pmbar} \left[ \One{\pi \in \MA_1^\c} \cdot g\sups{M_2}(\pi)\right] - \E_{\pi \sim \pmbar} \left[ \One{\pi \in \MA_1^\c \cap (\ME\sups{M_2})^\c} \cdot g\sups{M_2}(\pi) \right] \nonumber\\
    &\geq  \E_{\pi \sim \pmbar} \left[ \One{\pi \in \MA_1^\c} \cdot g\sups{M_2}(\pi)\right] - c_0\delta \geq \frac{\delta}{2} \cdot \left(1 - 2c_0 \right)\label{eq:lb-gm2-a2},
  \end{align}
where the first inequality follows because $g\sups{M_2}(\pi) < c_0\delta$ for $\pi \in (\ME\sups{M_2})^\c$. 

Next, we compute%
  \begin{align}
    c_0 \delta &\geq   \E_{\pi \sim \pmbar} \left[ \One{\pi \in \MA_2} \cdot \left( f\sups{M_1}(\pi\subs{M_1}) - f\sups{M_1}(\pi) \right) \right] \nonumber\\
    & = \E_{\pi \sim \pmbar} \left[ \One{\pi \in \MA_2} \cdot \left(f\sups{M_2}(\pi\subs{M_2}) - f\sups{M_1}(\pi) \right) \right] + \pmbar(\MA_2) \cdot \left(f\sups{M_1}(\pi\subs{M_1}) - f\sups{M_2}(\pi\subs{M_2})\right)\nonumber\\
    & \geq \E_{\pi \sim \pmbar} \left[ \One{\pi \in \MA_2} \cdot \left( f\sups{M_2}(\pi\subs{M_2}) - f\sups{M_2}(\pi) \right) \right] - 2\vep + \pmbar(\MA_2) \cdot \left(f\sups{M_1}(\pi\subs{M_1}) - f\sups{M_2}(\pi\subs{M_2})\right)\nonumber\\
    & \geq \frac{\delta}{2} \cdot (1-2c_0) - 2 \vep + \pmbar(\MA_2) \cdot \left( f\sups{M_1}(\pi\subs{M_1}) - f\sups{M_2}(\pi\subs{M_2}) \right)\nonumber,
  \end{align}
  where the first inequality follows because $g\sups{M_1}(\pi) < c_0\delta$ for $\pi \in \MA_2 \subset \MA_1^{\c} = (\ME\sups{M_1})^{\c}$, the
  second-to-last inequality follows from \eqref{eq:diff-m1m2-q} and the fact that $\MA_2 \subset \MA_1^{\c}$, and the final inequality follows by \eqref{eq:lb-gm2-a2}. Rearranging and using that $\pmbar(\MA_2) \in [0,1]$, we obtain
  \begin{align}
    f\sups{M_2}(\pi\subs{M_2}) - f\sups{M_1}(\pi\subs{M_1}) & \geq \pmbar(\MA_2) \cdot \left( f\sups{M_2}(\pi\subs{M_2}) - f\sups{M_1}(\pi\subs{M_1}) \right)\nonumber\\
    & \geq \frac{\delta}{2} \cdot (1-4c_0) - 2\vep\nonumber.
  \end{align}

  \paragraph{Lower bounding the failure probability and concluding}
  To finish the proof, we use the inequality \pref{eq:pac_lb_gap} to bound the probability $\pmbar((\ME\sups{M_2})^{\c} \cap \MA_1^{\c})$ as follows:
  \begin{align}
    \pmbar((\ME\sups{M_2})^\c \cap \MA_1^{\c}) \cdot \left( \frac{\delta}{2} \cdot (1-4c_0) - 2\vep \right) & \leq  \E_{\pi \sim \pmbar} \left[ \One{(\ME\sups{M_2})^\c \cap \MA_1^\c} \cdot (f\sups{M_2}(\pi\subs{M_2}) - f\sups{M_1}(\pi\subs{M_1})) \right]\nonumber\\
    & \leq  \E_{\pi \sim \pmbar} \left[ \One{(\ME\sups{M_2})^\c \cap \MA_1^\c} \cdot \left(( f\sups{M_2}(\pi\subs{M_2})  - f\sups{M_2}(\pi)) - (f\sups{M_1}(\pi\subs{M_1}) - f\sups{M_1}(\pi))\right)\right] + 2\ep\nonumber\\
    & \leq  \E_{\pi \sim \pmbar} \left[ \One{(\ME\sups{M_2})^\c \cap \MA_1^\c} \cdot g\sups{M_2}(\pi) \right] + 2\vep\nonumber\\
    & \leq  c_0 \delta + 2\vep\nonumber,
  \end{align}
  where the first inequality uses \eqref{eq:pac_lb_gap}, the second
  inequality uses \eqref{eq:diff-m1m2-q} and the fact that $(\ME\sups{M_2})^\c \cap \MA_1^\c \subset \MA_1^\c$, the third inequality uses
  that $\gmo\geq{}0$, and the final inequality uses the definition of
  $\cE\sups{M_2}$.  Since we
  have assumed that $\vep \leq \frac{\delta}{48}$ and $c_0 \leq 1/16$, it
  follows that $\pmbar((\ME\sups{M_2})^{\c} \cap \MA_1^\c) \leq
  \frac{\delta/8}{\delta/4} = 1/2$. Thus, we have established that\[\pmbar(\MA_1) + \pmbar(\ME\sups{M_2}) \geq \pmbar(\MA_1 \cup \ME\sups{M_2}) \geq 
  1/2,\] which implies that either $\pmbar(\ME\sups{M_1}) =
\pmbar(\MA_1) \geq 1/4$ or $\pmbar(\ME\sups{M_2}) \geq 1/4$. As a
consequence, by \eqref{eq:mbar-tenth}, we have that for some $i \in \{1,2\}$, $\E_{\pi \sim p\subs{M_i}} \left[ g\sups{M_i}(\pi)\right] \geq \frac{3c_0\delta}{20}$, thus establishing the desired lower bound.

To wrap up, we note that the lower bound we have established holds for an
arbitrary reference model $\Mbar\in\cMall$, so we are free to choose
$\Mbar\in\cMall$ to maximize $\deccpac[\veps/\sqrt{2}](\cM,\Mbar)$.
 
\end{proof}

\begin{remark}
    The structure of the proof \pref{thm:pac_lower} bears some
    superficial similarities to that of the classical
  two-point method
\citep{donoho1987geometrizing,donoho1991geometrizingii,donoho1991geometrizingiii,yu1997assouad,tsybakov2008introduction}
in statistics and information theory, but has a number of fundamental differences.
\begin{enumerate}
  \item First, in our lower bound, the pair of models
  $(M_1, M_2)$ is chosen in an \emph{adversarial} fashion based on the
  algorithm under consideration, while the classical approach selects
  a pair of models obliviously. When considering only two models, choosing the models adversarially is
  critical to capture the complexity of classes $\cM$ that require
  distinguishing between many distinct decisions. For example, even in
  the simple special case of multi-armed bandits, this is necessary to make the number of actions $A$ appear in the lower
  bound. %
\item Second, and perhaps more importantly, the classical two-point
  argument cannot be directly applied because the function $\gm(\pi)$
  does not enjoy the metric structure required by this approach. In
  particular, the classical ``separation condition'', which takes the form
$g\sups{M_1}(\pi) + g\sups{M_2}(\pi) \geq \delta\;\forall{}\pi$ when
applied to our setting, does not hold. Instead, the crux of the proof
is to show that, as a consequence of our choice for $M_1$ and $M_2$,
we have
\[
  \En_{\pi\sim{}p\subs{M_1}}\brk*{g\sups{M_1}(\pi)} + \En_{\pi\sim{}p\subs{M_2}}\brk*{g\sups{M_2}(\pi)}\approxgeq{}\delta.
\]
To establish this inequality, we take advantage of the fact that 
1) $\deccpac[\vepslowerT](\cM)\approxgeq{}\vepslowerT$, by assumption, and 2) rewards $r$
are observed and lie in the range $\brk*{0,1}$; the latter 
implies that $\abs{\fm(\pi)-\fmbar(\pi)}\leq\Dhel{M(\pi)}{\Mbar(\pi)}$
for all $M,\Mbar\in\cMall$.
\end{enumerate}

\end{remark}

\section{Upper Bounds}
\label{sec:upper}

We now give algorithms that obtain upper bounds on sample complexity
that complement the lower bounds in \pref{sec:lower}. First, in
\pref{sec:upper_pac} we give an algorithm and upper bound for the PAC
framework. We then prove this result in \pref{sec:upper_pac_proof},
highlighting the key algorithm design principles and analysis ideas. Then,
in \pref{sec:upper_regret}, we give an algorithm and upper bound for
regret. Both of our algorithms use the
\emph{Estimation-to-Decisions} paradigm of
\citet{foster2021statistical}. Our algorithm for regret builds upon
our algorithm for PAC, but is more
technically involved.

\subsection{Online Estimation} Our algorithms and regret bounds use
the primitive of an \emph{online estimation oracle}, denoted by
$\AlgEst$, which is an algorithm that, given knowledge of some class $\MM$ containing the true model $M^\st$, estimates the underlying model $\Mstar$ from
data on the fly. At each round $t$, given the data
$\hist\ind{t-1}=(\pi\ind{1},r\ind{1},o\ind{1}),\ldots,(\pi\ind{t-1},r\ind{t-1},o\ind{t-1})$
observed so far an estimation oracle builds an estimate
\[
\Mhat\ind{t}=\AlgEst\prn*{ \crl*{(\act\ind{i},
r\ind{i},\obs\ind{i})}_{i=1}^{t-1} }
\]
for the true model $\Mstar$. We measure the oracle's estimation
performance in terms of cumulative Hellinger error, which we assume is
bounded as follows.
\begin{assumption}[Estimation oracle for $\cM$]
    \label{ass:hellinger_oracle}
	At each time $t\in[T]$, an online estimation oracle
        $\AlgEst$ for $\cM$ returns,
        given  $$\hist\ind{t-1}=(\pi\ind{1},r\ind{1},o\ind{1}),\ldots,(\pi\ind{t-1},r\ind{t-1},o\ind{t-1})$$
        with $(r\ind{i},o\ind{i})\sim\Mstar(\pi\ind{i})$ and
        $\pi\ind{i}\sim p\ind{i}$, an estimator
        $\Mhat\ind{t}\in\conv(\cM)$ such that whenever $\Mstar\in\cM$,
        \begin{align}
              \EstHel \ldef{}
          \sum_{t=1}^{T}\En_{\act\ind{t}\sim{}p\ind{t}}\brk*{\Dhels{\Mstar(\pi\ind{t})}{\Mhat\ind{t}(\pi\ind{t})}}
          \leq \EstProbHel,
        \end{align}
	with probability at least $1-\delta$, where $\EstProbHel$ is a
        known upper bound.
      \end{assumption}

      Algorithms satisfying \pref{ass:hellinger_oracle} can be obtained via online conditional
      density estimation (that is, online learning with the
      logarithmic loss). Typically, the best possible
      estimation rate $\EstProbHel$ will reflect the statistical capacity of the class
      $\cM$. Standard examples include finite classes, where the
      exponential weights algorithm (also known as Vovk's
      aggregating algorithm) achieves
      $\EstProbHel\leq\bigoh(\log(\abs{\cM}/\delta))$, and parametric
      classes in $\bbR^{d}$, where one can achieve
      $\EstProbHel\leq\bigoht(d)$. See Section 4 of
      \citet{foster2021statistical} for further background.

      \subsection{Algorithm and Upper Bound for PAC}
      \label{sec:upper_pac}
            \begin{algorithm}[ht]
    \setstretch{1.3}
     \begin{algorithmic}[1]
       \State \textbf{parameters}:
       \Statex[1] Number of rounds $T\in\bbN$.
       \Statex[1] Failure probability $\delta>0$.
       \Statex[1] Online estimation oracle $\AlgEst$.
       \State Define $L \ldef \lceil \log 2/\delta\rceil$, $J \ldef
       \frac{T}{L+1}$, and $\EstBar\ldef\Est\left( \frac{2T}{\lceil\log 2/\delta \rceil},
  \frac{\delta}{4\lceil \log 2/\delta\rceil}\right)$.
       \State Set $\vepsupperT \ldef 8\sqrt{\frac{\lceil \log 2/\delta \rceil}{T} \cdot
  \EstBar}$.
       \Statex[0] \algcommentbig{Exploration phase}
  \For{$t=1, 2, \cdots, J$}
  \State Obtain estimate $\Mhat\ind{t} = \AlgEst\prn[\Big]{ \crl*{(\act\ind{i},
      r\ind{i},\obs\ind{i})}_{i=1}^{t-1} }$.  \label{line:pac-compute-mhat}
\State\multiline{Compute
\label{line:ptqt-pac}
    \begin{align}
  \hspace{-1cm}    (p\^t,q\^t) := \argmin_{p,q \in \Delta(\Pi)} \sup_{M \in  \MH_{q,\vepsupperT}(\wh M\^t)} \E_{\pi \sim p} [ \fm(\pim) - \fm(\pi)],\nonumber
    \end{align}
with the convention that the value is zero if $\cH_{q,\vepsupperT}(\Mhat\ind{t})=\emptyset$.}
    \State{}Sample decision $\act\ind{t}\sim{}q\ind{t}$ and update estimation
    oracle $\AlgEst$ with $(\act\ind{t},r\ind{t}, \obs\ind{t})$.
    \EndFor
    \Statex[0] \algcommentbig{Exploitation phase}
    \State Sample $L$ indices $t_1, \ldots, t_L \sim\unif(
    [J])$ independently. \label{line:pac-sample-linds}
    \State For each $\ell \in [L]$, draw $J$ independent samples
    $\pi_{\ell}\^1, \ldots, \pi_{\ell}\^J \sim q\^{t_\ell}$, and
    observe $(\pi_{ \ell}\^j, r_{ \ell}\^j, o_{\ell}\^j)$ for each $j
    \in [J]$.%
    \State For each $\ell \in [L]$ and $j \in [J]$, compute
    \begin{align}
\til M_{\ell}\^j := \AlgEst \left( \left\{ (\pi_{\ell}\^i, r_\ell\^i, o_\ell\^i) \right\}_{i=1}^{j-1} \right),\nonumber
    \end{align}
    and let $\til M_\ell := \frac{1}{J} \sum_{j=1}^J \til M_\ell\^j$. \hfill\algcommentlight{$\Mtil_{\ell}$ is a high-quality
      estimate for $\Mstar$ under $q\ind{t_\ell}$.}
    \label{line:tilm-pac-define}
    \State \textbf{output:}
    Set $\phat\ldef{} p\^{t_{\wh\ell}}$ and output
    $\pihat\sim{}\phat$, where $\wh\ell\ldef{}\argmin_{\ell\in\brk{L}}\E_{\pi \sim q\^{t_{\ell}}} \brk[\big]{ \DhelsX{\big}{\wh M\^{t_{\ell}}(\pi)}{\til M_{\ell}(\pi)}}$. \label{line:choose-lhat} %
\end{algorithmic}
\caption{Estimation-to-Decisions (\etdp) for PAC}
\label{alg:pac}
\end{algorithm}

\pref{alg:pac} displays our main algorithm for the PAC framework, \etdp. The
algorithm is built upon the Estimation-to-Decisions paradigm of
\citet{foster2021statistical}, which uses the following scheme for
each round $t$:
\begin{itemize}
\item Obtain an estimator $\Mhat\ind{t}\in\conv(\cM)$ for $\Mstar$ from the
  online estimation oracle $\AlgEst$.
\item Sample $\pi\ind{t}$ from a decision distribution obtained by solving the min-max
  optimization problem that defines the \CompShort, with
  $\Mhat\ind{t}$ plugged in as the reference model.
\end{itemize}
\etdp follows this template, but incorporates non-trivial changes 
that are tailored to i) the constrained (as opposed to offset) DEC and
ii) PAC guarantees (as opposed to regret). Briefly, the algorithm
consists of two phases, an \emph{exploration phase} and an
\emph{exploitation phase}, which we outline below.

\paragraph{Exploration phase}
In the \emph{exploration phase}, which consists of rounds
$t=1,\ldots,J$, where $J=\bigomt(T)$, \pref{alg:pac} repeatedly obtains an estimator $\Mhat\ind{t}$ by querying
  the estimation oracle $\AlgEst$ with the current dataset
  $\hist\ind{t-1}=(\pi\ind{1},r\ind{1},o\ind{1}),\ldots,(\pi\ind{t-1},r\ind{t-1},o\ind{t-1})$ (\cref{line:pac-compute-mhat}),
  then computes the pair of distributions $(p\ind{t},q\ind{t})$ that
  solve the min-max problem that defines the PAC DEC
  (\eqref{eq:dec_constrained_pac}) with
  $\Mhat\ind{t}$
  plugged in as the reference model (\cref{line:ptqt-pac}):
  \begin{align}
    \label{eq:pac_dec_alg}
    (p\^t,q\^t) := \argmin_{p,q \in \Delta(\Pi)} \sup_{M \in  \MH_{q,\vepsupperT}(\wh M\^t)} \E_{\pi \sim p} [ \fm(\pim) - \fm(\pi)].
  \end{align}
  By definition, the value above is always bounded by $\deccpac[\vepsupperT](\cM,\Mhat\ind{t})$.
  Recall that $p\ind{t}$ may be though of as an \emph{exploitation
  distribution}, and that $q\ind{t}$ may be thought of as an
  \emph{exploration distribution}. With these distributions in hand, \pref{alg:pac} samples
  $\pi\ind{t}\sim{}q\ind{t}$ from the exploration distribution
  $q\ind{t}$. The distribution $p\ind{t}$ is not used in this phase,
  but is retained for the exploitation phase that follows.
\paragraph{Exploitation phase}
  In the \emph{exploitation phase} (or, post-processing phase), we
  aim to identify an exploitation distribution
  $p\ind{t}\in\crl{p\ind{1},\ldots,p\ind{J}}$ from the collection
  computed during the exploration phase that is ``good'' in the
    sense that it has sufficiently low suboptimality under $\Mstar$. The motivation behind
    this stage is as follows. From the expression
    \pref{eq:pac_dec_alg} and the definition of the constrained
    \CompShort, we are guaranteed that $p\ind{t}$ has
    \[
     \En_{\pi\sim{}p\ind{t}}\brk*{\fmstar(\pimstar)-\fmstar(\pi\ind{t})}\leq{}\deccpac[\vepsupperT](\cM,\Mhat\ind{t})
    \]
for    any round $t\in\brk{J}$ where
    $\Mstar\in\cH_{q\ind{t},\vepsupperT}(\Mhat\ind{t})$. The challenge
    here is
    that the estimation oracle ensures only that the \emph{cumulative}
    estimation
    error for the estimators $\Mhat\ind{1},\ldots,\Mhat\ind{J}$ is
    low. There is no guarantee that the per-round estimation error
    $\En_{\pi\ind{t}\sim{}q\ind{t}}\brk*{\Dhels{\Mstar(\pi\ind{t})}{\Mhat\ind{t}(\pi\ind{t})}}$
        will decrease with time, and for any fixed round $t$ of
        interest, this quantity might be trivially large. This can lead
    to a problem we term \emph{\falsex}, where
    $\Mstar\notin\cH_{q\ind{t},\vepsupperT}(\Mhat\ind{t})$. 
    \Falsex is problematic because we have no control over the suboptimality
    under $\Mstar$ for rounds $t\in\brk{J}$ where it occurs. The good news is that
    while the online
    estimation oracle may lead to
    \falsex for some rounds, Markov's inequality implies
    that (for our choice of $\vepsupperT$, which depends on $\EstProbHel$) the true model 
    $\Mstar$ will be included in
    $\cH_{q\ind{t},\vepsupperT}(\Mhat\ind{t})$ for \emph{at least half
      of the rounds $t\in\brk{J}$}. Hence, the exploitation phase
    proceeds by sampling (on \cref{line:pac-sample-linds}) a small number of rounds $t_1,\ldots,t_L\in\brk{J}$---a logarithmic
    number suffices to ensure that at least one is good with high
    probability---and performing a test to identify a good distribution
    $p\ind{t_{\wh{\ell}}}$ within the set
    $\crl{p\ind{t_1},\ldots,p\ind{t_L}}$, which is then returned by
    the algorithm.

    In more detail, the exploitation phase
    (\pref{line:pac-sample-linds} through \pref{line:choose-lhat})
    proceeds by gathering many (namely, $\Theta(T/\log T)$) samples
    from $q\ind{t_{\ell}}$
    for each of the $L$ rounds $\crl{t_{\ell}}_{\ell\in\brk{L}}$ and then, for each
    $\ell\in\brk{L}$, using the estimation oracle $\AlgEst$ to produce
    an estimated model $\til M_\ell\in\cM$ based on these
    samples. Since many samples are used to produce the estimate
    $\til M_\ell$, it is guaranteed to be close to the true model
    $M^\st$ under $q\ind{t_{\ell}}$. This means that by choosing a round
    $t_{\wh \ell}$ that minimizes the Hellinger distance between $\wh
    M\^{t_{\wh \ell}}$ and $\til M_{\wh \ell}$
    (\cref{line:choose-lhat}), we have that with high probability,
    $M^\st \in \MH_{q\^{t_{\wh \ell}}, \vep}(\wh M\^{t_{\wh \ell}})$,
    thus solving the \falsex problem, and ensuring that the
    exploitation distribution $p\ind{t_{\wh
        \ell}}$ has low risk. \dfcomment{typesetting of
      $\cH_{q\ind{t_{\wh \ell}}}$ is brutal}

\paragraph{Main result}
We show that \etdp enjoys the following guarantee for PAC.
\begin{theorem}[Main Upper Bound: PAC]
  \label{thm:pac_upper}
Fix $\delta \in \left(0, \frac{1}{10} \right)$ and $T \in \BN$. Suppose
that \Cref{ass:realizability,ass:hellinger_oracle} hold, and let
$\EstBar\ldef\Est\left( \frac{2T}{\lceil\log 2/\delta \rceil},
  \frac{\delta}{4\lceil \log 2/\delta\rceil}\right)$. With
$\vepsupperT \ldef 8\sqrt{\frac{\lceil \log 2/\delta \rceil}{T} \cdot
  \EstBar}$, \Cref{alg:pac} guarantees that with probability at least $1-\delta$,
  \begin{align}
    \RiskDM \leq \deccpac[\vepsupperT](\MM) \nonumber.
  \end{align}
  Thus, the expected risk achieved by \Cref{alg:pac} is bounded by
  \begin{align}
\E[\RiskDM] \leq \deccpac[\vepsupperT](\MM)  + \delta.\label{eq:pac-expected-risk}
  \end{align}
\end{theorem}
This result matches the lower bound in \pref{thm:pac_lower}, with the
only gap being the choice of radius $\veps>0$ for the constrained
\CompShort: The lower bound (\pref{thm:pac_lower}) has
$\vepslowerT\propto\sqrt{\frac{1}{T}}$, while the upper bound
(\pref{thm:pac_upper}) has
$\vepsupperT\propto\sqrt{\frac{\log(2/\delta)\cdot\EstBar}{T}}$. As discussed in \citet{foster2021statistical}, understanding when the
estimation complexity $\EstBar$ can be removed or weakened is subtle
issue, as there are some classes $\cM$ for which this term is
necessary, and others for which it is superfluous. This is the main
question left open by our research.

Let us instantiate \pref{thm:pac_upper} for some standard examples, focusing on the special case
where $\cM$ is finite and $\EstProbHel\leq\log(\abs{\cM}/\delta)$ for simplicity.
\begin{itemize}
  \item Whenever
    $\deccpac(\cM)\propto\veps\cdot\sqrt{\Ceff}$,
  \pref{thm:pac_upper} gives
$\En\brk*{\RiskDM} \leq\bigoht\prn*{\sqrt{\frac{\Ceff{}\log\abs{\cM}}{T}}}$,
which translates into
$\bigoht\prn*{
    \frac{\Ceff\log\abs{\cM}}{\veps^2}
    }$
samples to learn an $\veps$-optimal policy.
  
\item Whenever $\deccpac(\cM)\propto\veps^{1-\rho}$
for $\rho\in(0,1)$, \pref{thm:pac_upper}
gives $\En\brk*{\RiskDM} \leq\bigoht((\log\abs{\cM}/T)^{\frac{(1-\rho)}{2}})$, which
translates into $\bigoht(\log\abs{\cM}\cdot\veps^{-\frac{2}{1-\rho}})$ samples to learn
a $\veps$-optimal policy.

\item Whenever $\deccpac(\cM)\propto\indic\crl{\veps\geq{}1/\sqrt{\Ceff}}$,
  where $\Ceff$ is a problem-dependent parameter,
  \pref{thm:pac_upper} gives
  $\RiskDM=0$ with high probability whenever
  $T\geq\bigomt(\Ceff\cdot\log\abs{\cM})$.
\end{itemize}

\subsection{Proof of PAC Upper Bound (\Cref{thm:pac_upper})}
\label{sec:upper_pac_proof}

We now prove \pref{thm:pac_upper}. Before proceeding with the proof,
we give some brief background on the notion of online-to-batch
conversion, which is used in the exploitation phase and its analysis.

\paragraph{Background: Online-to-batch conversion}
As discussed in the prequel, we assume access to an online oracle
$\AlgEst$. We use the
\emph{online} guarantee the oracle the algorithm provides---namely,
that it ensures that the cumulative estimation error is bounded for an adaptively chosen
sequence of decisions---in a
non-trivial fashion during the exploration phase, but our analysis additionally
makes use the fact that online oracles can be used to provide
guarantees for \emph{offline} estimation.

For offline estimation, we consider a setting in which 
        there is some $p \in \Delta(\Pi)$ so that $p\^t = p$ for all
        $t$, and the algorithm must output a single model estimate
        $\wh M$ such that
        $\En_{\pi\sim{}p}\brk*{\Dhels{\Mhat(\pi)}{\Mstar(\pi)}}\leq\veps^2$. We
        can obtain such a guarantee using an online estimation oracle via the
        following online-to-batch conversion process:
        \begin{itemize}
        \item For each $t=1,\ldots,T$, obtain $\Mhat\ind{t}$ by running $\AlgEst$ on samples
          $\crl{(\pi\^i, r\^i, o\^i)}_{t=1}^{t-1}$, where $\pi\^t \sim
          p$ and $(r\ind{t},o\ind{t})\sim\Mstar(\pi\ind{t})$.
        \item Let $\wh M := \frac 1T \sum_{t=1}^T \wh M\^t \in \co(\MM)$.
        \end{itemize}
        It is evident from  \cref{ass:hellinger_oracle} and
        convexity of the squared Hellinger distance that with
        probability at least $1-\delta$, the estimator $\Mhat$
        constructed above satisfies
        \begin{align}
          \E_{\pi \sim p} \left[ \hell{M^\st(\pi)}{\wh M(\pi)}\right] &= \E_{\pi \sim p} \left[ \hell{M^\st(\pi)}{\frac 1T \sum_{t=1}^T \wh M\^t(\pi)}\right]\nonumber\\
          &\leq  \frac 1T \sum_{t=1}^T \E_{\pi \sim p} \left[ \hell{M^\st(\pi)}{\wh M\^t(\pi)}\right] \leq \frac{\Est(T, \delta)}{T}\label{eq:offline-estimation}.
        \end{align}

This is the strategy employed in \pref{line:tilm-pac-define} of
\pref{alg:pac}. Standard offline estimation algorithms (e.g., MLE) can
be used to derive similar guarantees, but we make use of online-to-batch in order to
keep notation light, since \pref{alg:pac} already requires the online
estimation algorithm for the exploration phase.

\begin{proof}[\pfref{thm:pac_upper}] We begin by analyzing the
  exploitation phase.
e  Recall that we set $J := \frac{T}{\lceil \log 2/\delta \rceil+1} \geq \frac{T}{2L}$. 
  By \Cref{ass:hellinger_oracle}, we have that with probability at
  least $1-\frac{\delta}{4L}$, 
  \begin{align}
\sum_{t=1}^{J} \E_{\pi\^t \sim q\^t} \left[ \hell{M^\st(\pi\^t)}{\wh M\^t(\pi\^t)} \right] \leq \Est\left(J, \frac{\delta}{4L}\right) \leq \EstBar\nonumber.
  \end{align}
We denote this event by $\scrE_0$, and condition on it going
forward. Since we have $\vepsupperT^2 \geq  \frac{32}{J} \cdot \EstBar$
by definition, it follows from Markov's inequality that if $s \in [J]$ is chosen uniformly at random, then with probability at least $1/2$,
  \begin{align}
\E_{\pi\^s \sim q\^s} \left[ \hell{M^\st(\pi\^s)}{\wh M\^s(\pi\^s)} \right] \leq \frac{\vepsupperT^2}{16}\label{eq:s-hell-bound-2}.
  \end{align}
Going forward, our aim is to show that the exploitation phase
identifies such an index $s\in\brk{J}$. Indeed, for any $s\in\brk{J}$
such that the inequality \pref{eq:s-hell-bound-2} holds, we have
$\Mstar\in\cH_{q\ind{s},\vepsupperT}(\Mhat\ind{s})$, and hence
\[
\En_{\pi\sim{}p\ind{s}}\brk*{\fmstar(\pimstar) - \fmstar(\pi)}\leq\deccpac[\vepsupperT](\cM,\Mhat\ind{s})\leq \deccpac[\vepsupperT](\cM)
\]
from the expression
    \pref{eq:pac_dec_alg} and the definition of the constrained
    \CompShort.

To proceed, first observe that for the uniformly sampled indices $t_1,
\ldots, t_L \in [J]$, a standard confidence boosting argument implies
that with probability at least $1-2^{-L} \geq 1-\frac{\delta}{2}$,
there is some $\ell \in [L]$ so that \eqref{eq:s-hell-bound-2} is
satisfied with $s = t_{\ell}$. We denote this event by $\scrF$.

Next, recall the definition $\til M_\ell = \frac 1J \sum_{j=1}^J
\til M_\ell\^j$ in \eqref{line:tilm-pac-define} of
\Cref{alg:pac}. Using \pref{ass:hellinger_oracle} together with
\eqref{eq:offline-estimation}, we have that for each $\ell \in [L]$,
there is an event that occurs with probability at least
$1-\frac{\delta}{4L}$, denoted by
$\scrE_\ell$, such that under $\scrE_\ell$ we have
  \begin{align}
\E_{\pi \sim q\^{t_\ell}} \left[ \hell{M^\st(\pi)}{\til M_{\ell}(\pi)}\right] \leq \frac{\Est(J, \delta/4L)}{J} \leq \frac{\EstBar}{J} \leq  \frac{\vepsupperT^2}{32}\nonumber,
  \end{align}
  where the second inequality uses our choice for $\vepsupperT$. Define
  $\scrE \ldef \bigcap_{\ell=0}^L \scrE_\ell$, so that $\scrE$ occurs with
  probability at least $1-\frac{(L+1)\delta}{4L} \geq
  1-\frac{\delta}{2}$. We define $\scrE = \scrF \cap \bigcap_{\ell=0}^L \scrE_\ell$, so that $\scrE$ occurs with probability at least $1-\frac{(L+1)\delta}{4L} - \frac{\delta}{2} \geq 1-\delta$.
  
We now show that the exploitation phase succeeds whenever the event
$\scrE$ holds. By the triangle inequality for Hellinger distance,
letting $\ell\in\brk{L}$ be any index such that \eqref{eq:s-hell-bound-2} is
satisfied with $s=t_\ell$, we have
  \begin{align}
\E_{\pi \sim q\^{t_\ell}} \left[ \hell{\til M_\ell(\pi)}{\wh M\^{t_\ell}(\pi)} \right] \leq & 2 \cdot \left( \E_{\pi \sim q\^{t_\ell}} \left[ \hell{M^\st(\pi)}{\til M_\ell(\pi)} + \hell{M^\st(\pi)}{\wh M\^{t_\ell}(\pi)} \right]\right)\leq \frac{\vepsupperT^2}{4}\nonumber.
  \end{align}
  From the definition on \Cref{line:choose-lhat}, the index
 $\wh \ell\in\brk{L}$ satisfies
  \begin{align}
\E_{\pi \sim q\^{t_{\wh \ell}}} \left[ \hell{\til M_{\wh \ell}(\pi)}{\wh M\^{t_{\wh \ell}}(\pi)} \right] \leq \E_{\pi \sim q\^{t_\ell}} \left[ \hell{\til M_\ell(\pi)}{\wh M\^{t_\ell}(\pi)} \right] \leq \frac{\vepsupperT^2}{4}\nonumber.
  \end{align}
  Using the triangle inequality for Hellinger distance once more, we obtain that under the event $\scrE$,
  \begin{align}
\E_{\pi \sim q\^{t_{\wh \ell}}} \left[ \hell{\wh M\^{t_{\wh \ell}}(\pi)}{M^\st(\pi)} \right] \leq 2 \cdot \left( \frac{\vepsupperT^2}{4} + \frac{\vepsupperT^2}{32} \right) < \vepsupperT^2\nonumber,
  \end{align}
  which means that $M^\st\in \MH_{ q\^{t_{\wh \ell}},\vepsupperT}(\wh
  M\^{t_{\wh \ell}})$. It follows that whenever $\scrE$ holds,
  \begin{align}
   \RiskDM =  \E_{\pi \sim p\^{t_{\wh \ell}}} \left[
    \fmstar(\pimstar)- \fmstar(\pi) \right] & \leq  \sup_{M \in \MH_{ q\^{t_{\wh \ell}},\vepsupperT}(\wh M\^{t_{\wh \ell}})} \E_{\pi \sim p\^{t_{\wh \ell}}}[\fm(\pim) - \fm(\pi)]\nonumber\\
    &= \inf_{p,q \in \Delta(\Pi)} \sup_{M \in \MH_{q,\vepsupperT}(\wh M\^{t_{\wh \ell}})} \E_{\pi \sim p} [\fm(\pim) - \fm(\pi)]\nonumber\\
    &= \deccpac(\MM, \wh M\^{t_{\wh \ell}})\nonumber\\
    &\leq \sup_{\Mbar \in \co(\MM)} \deccpac[\vepsupperT](\MM, \Mbar)=\deccpac[\vepsupperT](\cM)\nonumber,
  \end{align}
  where the first equality follows from the choice of $p\^{t_{\wh \ell}}, q\^{t_{\wh \ell}}$ in \Cref{line:ptqt-pac}.
  The conclusion in \eqref{eq:pac-expected-risk} of the theorem
  statement follows from the observation that $\RiskDM$ is bounded
  above by 1, as we assume that rewards lie in $[0,1]$. 
  \end{proof}

\subsection{Algorithm and Upper Bound for Regret}
\label{sec:upper_regret}

      \begin{algorithm}[tp]
    \setstretch{1.3}
     \begin{algorithmic}[1]
       \State \textbf{parameters}:
       \Statex[1] Number of rounds $T\in\bbN$ and error probability $\delta \in (0,1)$.
       \Statex[1] Failure probability $\delta>0$.
       \Statex[1] Online estimation oracle $\AlgEst$.
       \Statex[1] Constants $C_1 > 1$. \algcommentlight{Specified in \cref{app:upper}.}
       \State Define $N := \lceil \log T \rceil$ and $L \ldef \lceil \log 1/\delta \rceil$.
     \State Set $\vep_N \ldef\sqrt{\frac{C_1\cdot \Est(T, \delta)\cdot  L}{T}}$ each $\vep_i \ldef \sqrt{2^{N-i}} \cdot \vep_N$ for $i \in [N]$. 
\State Split $[T]$ into $2N$ contiguous blocks $\ME_1 \cup \MR_1 \cup \cdots \cup \ME_N \cup \MR_N$, where
$
    |\MR_i|, |\ME_i| \in \left[ \frac{2^i}{4}, \frac{2^i}{2} \right]  %
    $
    for all $i \in [N]$.
  \State Set $\MM_1 := \MM$.
  \For{$i \in [N]$}
  \Statex[0] \algcommentbig{Exploration for epoch $i$}
  \State Initialize an instance of $\AlgEst$, with time horizon
  $|\ME_i|$, failure probability $\delta$, and model class $\MM_i$.
  \For{$t \in \ME_i$}
  \State Obtain estimate $\wh M\^t := \AlgEst( \{ (\pi\^s, r\^s, o\^s)\}_{s\in \ME_i, s < t })$, where $\wh M\^t \in \co(\MM_i)$. 
    \State       Compute %
    \label{line:regret_algo_dec} 
      \begin{align}
        p\^t := \argmin_{p \in \Delta(\Pi)} \sup_{M \in \MH_{p,\vep_i}(\wh M\^t) \cup \{\wh M\^t\}} \E_{\pi \sim p}[\fm(\pim) - \fm(\pi)]\nonumber.
      \end{align}
      \State Sample decision $\act\ind{t}\sim{}p\ind{t}$ and update estimation
      oracle $\AlgEst$ with $(\act\ind{t},r\ind{t}, \obs\ind{t})$.
      \EndFor
      \Statex[0] \algcommentbig{Refinement for epoch $i$}
  \State Sample a subset $\MS_i \subset \ME_i$ of size $L$ uniformly at random (with replacement). \label{line:sample-subset-si}
  \State Initialize $s_i^{\mathrm{tmp}} = \perp$. \algcommentlight{\emph{With high probability,~$s_i^{\mathrm{tmp}}$ will be updated in the {for} loop.}}
  \For{$s \in \MS_i$}
  \State \multiline{Initialize an instance of $\AlgEst$ with horizon $J_i := |\MR_i|/L$, failure probability $\delta$, and model class $\MM_i$.  }
  \For{$1 \leq j \leq J_i$}
  \State Sample a decision $\pi\^j_s \sim p\^s$ and update $\AlgEst$ with $(\pi\^j_s, r\^j_s, o\^j_s)$.\label{line:sample-pisj-ps}
  \State Obtain estimate $\til M_s\^j := \AlgEst \left( \{( \pi_s\^k, r_s\^k, o_s\^k)_{k=1}^{j-1} \} \right)$, where $\til M_s\^j \in \co(\MM_i)$.
  \If{$\sum_{k=1}^j \E_{\pi \sim p\^s} \left[\hell{\wh M\^s(\pi)}{\til M_s\^k(\pi)}\right] > \frac{J_i \vep_i^2}{4}$}
  \State Define $J_{i,s} := j$, and \textbf{break} out of the loop over $j$ (i.e., proceed to the next value in $\MS_i$).\label{line:choose-jis}
  \EndIf

  \If{$j = J_i$}
  \State Set $s_i^{\mathrm{tmp}} = s$ and $J_{i,s} = J_i$.\footnotemark
  \label{line:assign-si}
  \EndIf
\EndFor
\EndFor
\State Set $s_i = s_i^{\mathrm{tmp}}$ if $s_i^{\mathrm{tmp}} \neq
\perp$, else choose $s_i \in \MS_i$ arbitrarily. Set $\wh M_i := \frac{1}{J_{i,s}} \sum_{j=1}^{J_{i,s}} \til M_{s_i}\^j$ and $\wh p_i = p\^{s_i}$. \label{line:set-si-perm}
\For{any remaining rounds $t$ in $\MR_i$}
\State Play $\pi\^t \sim p\^{s_i}$ (the rewards and observations can be ignored). \label{line:play-psi}
\EndFor
  \State Set
    \begin{align}
\MM_{i+1} := \left\{ M \in \MM \mid  \E_{\pi \sim \wh p_i} \left[ \hell{M(\pi)}{\wh M_i(\pi)} \right] \leq \frac{\Est(J_i, \delta)}{J_i} \right\}\nonumber.
    \end{align}
    \EndFor
\end{algorithmic}
\caption{Estimation-to-Decisions (\etdp) for Regret}
\label{alg:regret}
\end{algorithm}

\footnotetext{\pref{alg:regret} will continue to work if we modify it to break out of the loop over $s\in\cS_i$ when this if statement is reached, which is somewhat more natural. We use the version presented here because it slightly simplifies the proof.}

We now present a variant of \etdp for the regret framework,
\pref{alg:regret}, which attains a regret bound that
scales with $\deccreg[\vepsupperT](\cM)\cdot{}T$ for an appropriate
radius $\vepsupperT>0$.

\paragraph{Online estimation}
\pref{alg:regret} makes use of an online estimation oracle in the same
fashion as \pref{alg:pac}, but we require a slightly stronger oracle
capable of incorporating \emph{constraints} on the model class,
specified via a subset $\cM'\subseteq\cM$.
\begin{assumption}[Constrained estimation oracle for $\cM$]
    \label{ass:hellinger_oracle_constrained}
    A constrained estimation oracle for $\cM$ takes as input a
    constraint set $\cM'\subseteq\cM$. At each time $t\in[T]$, the online estimation oracle
        $\AlgEst$ returns,
        given  $$\hist\ind{t-1}=(\pi\ind{1},r\ind{1},o\ind{1}),\ldots,(\pi\ind{t-1},r\ind{t-1},o\ind{t-1})$$
        with $(r\ind{i},o\ind{i})\sim\Mstar(\pi\ind{i})$ and
        $\pi\ind{i}\sim p\ind{i}$, an estimator
        $\Mhat\ind{t}\in\conv(\cM')$ such that whenever $\Mstar\in\cM'$,
        \begin{align}
              \EstHel \ldef{}
          \sum_{t=1}^{T}\En_{\act\ind{t}\sim{}p\ind{t}}\brk*{\Dhels{\Mstar(\pi\ind{t})}{\Mhat\ind{t}(\pi\ind{t})}}
          \leq \EstProbHel,
        \end{align}
	with probability at least $1-\delta$, where $\EstProbHel$ is a
        known upper bound.
      \end{assumption}
This assumption is identical to \pref{ass:hellinger_oracle}, except
that i) the oracle takes a \emph{constraint set} $\cM'\subseteq\cM$ as an input
before the learning process begins, and ii) the oracle is required to
produce $\Mhat\ind{t}\in\conv(\cM')$; that is, the estimator is
required to lie in the convex hull of the constraint set
$\cM'$. All estimation algorithms that we are aware of can achieve
this guarantee with minor or no modifications, including the
exponential weights algorithm, which satisfies
\pref{ass:hellinger_oracle_constrained} with $\EstProbHel\leq\bigoh(\log(\abs{\cM}/\delta))$.

      \paragraph{Overview of algorithm}
      \pref{alg:regret} employs the Estimation-to-Decisions
principle of \citet{foster2021statistical} but, like
\pref{alg:pac}, incorporates substantial modifications tailored to the
constrained (as opposed to offset) \CompShort. The core of the algorithm
is \pref{line:regret_algo_dec}, which---at each round $t$---obtains an
estimator $\Mhat\ind{t}\in\conv(\cM)$, then computes
an exploratory distribution $p\ind{t}$ by solving the min-max problem
that defines the regret \CompShort (\eqref{eq:dec_constrained}) with
$\Mhat\ind{t}$ plugged in:
  \begin{align}
    \label{eq:regret_dec_alg}
    p\ind{t}:= \argmin_{p\in \Delta(\Pi)} \sup_{M \in  \MH_{p,\veps}(\wh M\^t)\cup\crl{\Mhat\ind{t}}} \E_{\pi \sim p} [ \fm(\pim) - \fm(\pi)],
  \end{align}
  for an appropriate choice of $\veps>0$ that depends on $t$.

As with \pref{alg:pac}, the main challenge \pref{alg:regret} needs to
overcome is \emph{\falsex}: Whenever
$\Mstar\in\cH_{p\ind{t},\veps}(\Mhat\ind{t})$, it is immediate
from the definition above that
\[
  \En_{\pi\ind{t}\sim{}p\ind{t}}\brk*{\fmstar(\pimstar)-\fmstar(\pi\ind{t})}\leq{}
  \deccreg[\veps](\cM\cup\crl{\Mhat\ind{t}},\Mhat\ind{t}),
\]
but what happens if
$\Mstar\notin\cH_{p\ind{t},\veps}(\Mhat\ind{t})$? For PAC, the
solution employed by \pref{alg:pac} is simple: The online estimation
guarantee for $\AlgEst$ ensures that $\Mstar$ is correctly
included for at least half the rounds, and we only need to identify a
single round where it is included. For regret, this reasoning no
longer suffices: We cannot simply ignore the rounds in which
$\Mstar$ is excluded, as the regret for these rounds must be
controlled. 

\paragraph{Epoch scheme}
To address the issue of \falsex, \pref{alg:regret} breaks the rounds $1,\ldots,T$ into epochs
$1,\ldots,N$ of doubling length, with each epoch $i$ consisting of a
contiguous set of \emph{exploration rounds} $\cE_i\subseteq\brk{T}$ and
\emph{refinement rounds} $\cR_i\subseteq\brk{T}$. Each
epoch proceeds in a similar fashion to \pref{alg:pac}, but takes
advantage of the
data collected in previous epochs to explore in a fashion that is
robust to false exclusions:
\begin{itemize}
  \item Each exploration phase gathers data using a sequence of exploratory
  distributions $\crl{p\ind{t}}_{t\in\cE_i}$ computed by solving
  the DEC with the estimated models $\crl{\Mhat\ind{t}}_{t\in\cE_i}$,
  following \eqref{eq:regret_dec_alg}. However, the estimated models
  $\Mhat\ind{t}$ are restricted to lie in $\conv(\cM_i)$, where
  $\cM_i$ is a confidence set computed using data from the previous
  epoch.
  \item The purpose of the refinement phase in epoch $i$ is to use
  the distributions generated in the exploration phase to compute a
  confidence set $\cM_{i+1}$ for the \emph{next epoch} that satisfies
  a certain invariant that allows us to translate low regret with
  respect to models in the confidence set to low regret under
  $\Mstar$.
\end{itemize}

To describe the phases for an epoch $i\in\brk{N}$ in greater detail,
we define
\[
  \alpha_i\ldef{} C_0 \cdot \deccreg[\veps_i](\cM) + 64 \vep_i,
\]
for a constant $C_0 > 1$ whose value is specified in the full proof (\pref{app:upper}).  

\paragraph{Exploration phase} For each step $t$ within the exploration phase $\cE_i$, we
  obtain an estimator $\Mhat\ind{t}\in\conv(\cM_i)$ from the estimation oracle and
  solve
        \begin{align}
        p\^t := \argmin_{p \in \Delta(\Pi)} \sup_{M \in
          \MH_{p,\vep_i}(\wh M\^t)\cup\crl{\Mhat\ind{t}}} \E_{\pi \sim
          p}[\fm(\pim) - \fm(\pi)].
          \label{eq:regret_alg_epoch}
        \end{align}
        where $\cM_i\subseteq\cM$ is the confidence set produced by the
        refinement phase in epoch $i-1$. The set $\cM_i$ is
        constructed so that $\Mstar\in\cM_i$ with high probability,
        but in addition, the following \emph{localization property}
        can be shown (using that $\wh M\^t \in \co(\MM_i)$ for all $t
        \in \ME_i$; see the proof of \cref{lem:compare-arb-distr}): %
        \begin{equation}
          \fmstar(\pimstar) \leq
          f\sups{\Mhat\ind{t}}(\pi\subs{\Mhat\ind{t}}) + \frac{\alpha_{i-1}}{2}\quad\forall t
        \in \ME_i.\label{eq:localization_regret}
        \end{equation}
        The condition \pref{eq:localization_regret} implies that we are always in a favorable position with
        respect to regret:
        \begin{itemize}
          \item If
          $\Mstar \in \MH_{p\ind{t},\vep_i}(\wh M\^t)$---
          that is, the true model is not falsely excluded---we have
          $\En_{\pi\sim{}p\ind{t}}\brk*{\fmstar(\pimstar)-\fmstar(\pi)}\leq\deccreg[\veps_i](\cM)\leq\alpha_i$
          by definition.
          \item Even if $\Mstar$ is falsely
          excluded, \eqref{eq:localization_regret} implies that up to
          certain nuisance terms, we have (as shown in \cref{lem:epoch-optimal-pol}): 
          \[
            \En_{\pi\sim{}p\ind{t}}\brk[\big]{\fmstar(\pimstar)-\fmstar(\pi)}
            \approxleq{}
            \En_{\pi\sim{}p\ind{t}}\brk[\big]{f\sups{\Mhat\ind{t}}(\pi\subs{\Mhat\ind{t}})-f\sups{\Mhat\ind{t}}(\pi)}
            \approxleq\alpha_{i-1},
          \]
          where the latter inequality uses the definition
          \pref{eq:regret_alg_epoch}, which implies that          \[\En_{\pi\sim{}p\ind{t}}\brk[\big]{f\sups{\Mhat\ind{t}}(\pi\subs{\Mhat\ind{t}})-f\sups{\Mhat\ind{t}}(\pi)}\leq\deccreg[\veps_i](\cM\cup\crl{\Mhat\ind{t}},\Mhat\ind{t})\leq\alpha_i.\]
        \end{itemize}
        In both situations, we incur no more than $\bigoh(\alpha_i)$
        regret per round. Due to the doubling epoch schedule,
        the total contribution to regret for all exploration rounds is
        no more than $\bigoht\prn[\big]{\deccreg[\vepsupperT](\MM)\cdot{}T}$.
        \paragraph{Refinement phase}
        For the analysis of the
        exploration phase in epoch $i+1$ to succeed, the refinement phase at
        epoch $i$ must construct a confidence set $\cM_{i+1}$ so that
        the localization property \pref{eq:localization_regret} is
        satisfied with scale $\alpha_{i}$.
        To achieve this, we use a similar
        approach to the exploitation phase in \pref{alg:pac}. For all the rounds $t\in\cE_i$ for which $\Mstar$ is not
        falsely excluded, we are guaranteed that i)
        $\Mstar\in\cH_{p\ind{t},\veps_i}(\Mhat\ind{t})$, and ii) all
        $M\in\cH_{p\ind{t},\veps_i}(\Mhat\ind{t})$ satisfy
        \[
          \En_{\pi\sim{}p\ind{t}}\brk[\big]{\fm(\pim)-\fm(\pi)}\leq \deccreg[\vep_i](\MM) \leq \frac{\alpha_i}{4},
        \]
        which can be shown to imply the localization property in
        \eqref{eq:localization_regret} holds at epoch $i$ (the proof of this fact uses
        that $\Mhat\ind{t}\in\conv(\cM_{i+1})$ for all $t\in\cE_{i+1}$). In addition, the estimation
        guarantee for the oracle (\pref{ass:hellinger_oracle_constrained}) implies
        that $\Mstar$ is included in
        $\cH_{p\ind{t},\veps_i}(\Mhat\ind{t})$ for at least half of
        the rounds $t\in\cE_i$. Hence, to construct $\cM_{i+1}$, we sample a small number of rounds $t_1,\ldots,t_L\in\cE_i$ (\cref{line:sample-subset-si}; a logarithmic 
    number $L$ suffices) and perform a test based on Hellinger distance to identify a ``good''
    distribution $p\ind{t_{\ell}}$ from the set
    $\crl{p\ind{t_1},\ldots,p\ind{t_L}}$ with the property that
    $\Mstar\in\cH_{p\ind{t_\ell},\veps_i}(\Mhat\ind{t_\ell})$. We
    then set
    $\cM_{i+1}=\cH_{p\ind{t_{\ell}},\veps_i}(\Mhat\ind{t_\ell})$,
     ensuring that $\Mstar\in\cM_{i+1}$ and the localization
    property is satisfied.

\paragraph{Main result}
We now present the main regret guarantee for \etdp (\pref{alg:regret}). To state the result in the simplest form
possible, we assume that the regret \CompShort satisfies a mild
growth condition.
\begin{assumption}[Regularity of DEC]
  \label{asm:dec}
  The class $\MM$ satisfies the following:
   for some constant $\Creg > 1$ and all $\vep \in (0,2)$, we have
    \begin{align}
\deccreg(\MM) \leq \Creg^2 \cdot \deccreg[\vep/C_\reg](\MM)\nonumber.
    \end{align}
  \end{assumption}
  This condition asserts that the \CompShort does not shrink too
  quickly as a function of the parameter $\veps$. It is automatically satisfied whenever
$\deccreg(\cM)\propto\veps^{\rho}$ for $\rho\leq{}2$,
with $\rho\leq{}1$ corresponding
to the ``$\sqrt{T}$-regret or greater'' regime; regret bounds that
hold under more general assumptions are given in \pref{app:upper}.
\begin{theorem}[Main Upper Bound: Regret]
  \label{thm:regret_upper}
  Fix %
  $T \in \BN$ and $\delta \in (0,1)$. Suppose
that \Cref{ass:realizability,ass:hellinger_oracle_constrained,asm:dec} hold, and let
$\EstBar\ldef\Est(T, \delta^2)$. Then \pref{alg:regret}, with $C_1>0$
chosen appropriately, ensures
that with probability at least $1-\delta$,
  \begin{align}
  \RegDM \leq  \deccreg[\vepsupperT](\MM)\cdot{}O\prn*{T \log(T)}  +  O\prn[\Big]{\sqrt{T \log(1/\delta) \cdot \EstBar}} \nonumber,
  \end{align}
  where $\vepsupperT \ldef C\cdot\sqrt{\frac{ \EstBar\cdot\log(1/\delta)}{T}}$
  for a numerical constant $C>0$.
  \end{theorem}
The proof of this result is deferred to \pref{app:upper}. As in the PAC setting, this upper bound matches the corresponding lower bound (\pref{thm:regret_lower}) up to
the radius for the constrained \CompShort
($\vepslowerT\approx\sqrt{\frac{1}{T}}$ for \pref{thm:regret_lower} versus
$\vepsupperT\approx\sqrt{\frac{\EstBar}{T}}$
for \pref{thm:regret_upper}), and cannot be improved beyond logarithmic
factors without further
assumptions.

\begin{remark}[Relaxing the regularity condition]
It follows immediately from the proof of \pref{thm:regret_upper} that the following holds. Suppose that in place of \cref{asm:dec}, we assume that there is some function $\mathsf{r}(\vep)$ so that, for all $\vep \in (0,2)$, we have: (1) $\deccreg[\vep](\MM) \leq \mathsf{r}(\vep)$  and (2) $\mathsf{r}(\vep) \leq \Creg^2 \cdot \mathsf{r}(\vep/\Creg)$. Then the regret bound of \pref{thm:regret_upper} holds with $\deccreg[\vepsupperT](\MM)$ replaced by $\mathsf{r}(\vepsupperT)$. %
\end{remark}

Examples of \pref{thm:regret_upper}, under the assumption that
$\cM$ is finite and $\EstProbHel\leq\bigoh\prn[\big]{\log(\abs{\cM}/\delta)}$, include:
\begin{itemize}
  \item Whenever
    $\deccreg(\cM)\propto\veps\cdot\sqrt{\Ceff}$,
  \pref{thm:regret_upper} gives
$\En\brk*{\RegDM} \leq\bigoht\prn[\big]{\sqrt{\Ceff{}\cdot{}T\cdot{}\log\abs{\cM}}}$.
  
\item Whenever $\deccreg(\cM)\propto\veps^{1-\rho}$
for $\rho\in(0,1)$, \pref{thm:regret_upper}
gives $\En\brk*{\RegDM}
\leq\bigoht\prn[\big]{T^{\frac{(1+\rho)}{2}}\cdot\log^{\frac{1-\rho}{2}}\abs{\cM} + T^{\frac{1}{2}} \log^{\frac{1}{2}}|\MM|}$.
\end{itemize}

\section{\CompText: Structural Properties}
\label{sec:properties}

The lower and upper bounds in \pref{sec:lower,sec:upper}, which are
stated in terms of the constrained \CompText,  are tight up
to dependence on the model estimation error $\EstProbHel$ (recall that
the lower bounds use scale $\vepslowerT=\bigomt\prn[\big]{\sqrt{1/T}}$,
while the upper bounds use scale
$\vepsupperT=\bigoht\prn[\big]{\sqrt{\EstProbHel/T}}$). It is natural to ask how
these results are related to the lower and upper bounds in
\citet{foster2021statistical,foster2022complexity}, which are stated
in terms of the offset \CompShort, and at first glance are not
obviously comparable. Toward developing such an understanding, this
section establishes a number of structural properties for the
\CompShort.
\begin{itemize}
\item In \pref{sec:constrained_offset}, we show that the constrained
  and offset \CompShort are nearly equivalent for PAC. For regret, we
  show that it is always possible to bound the constrained
  \CompShort by the offset \CompShort in a tight fashion, but the converse
  is not true in general.
\item In \pref{sec:localization}, we show that the constrained \CompShort
  implicitly enforces a form of localization, and uncover a tighter
  relationship between the constrained and offset variants of the
  regret \CompShort for
  localized classes.
\end{itemize}
\pref{sec:reference_regret,sec:convexity} investigate the
role of the reference model $\Mbar\in\cMall$.
\begin{itemize}
\item First, in
  \pref{sec:reference_regret}, we show that the definition
  $\deccreg(\cM)=\sup_{\Mbar\in\conv(\cM)}\deccreg(\cM\cup\crl{\Mbar},\Mbar)$,
  which incorporates suboptimality under the reference model $\Mbar$, is
  critical to obtain tight upper and lower bounds. This is a
  fundamental difference from
  PAC, where we show that it suffices to use the definition
  $\deccpac(\cM)=\sup_{\Mbar\in\conv(\cM)}\deccpac(\cM,\Mbar)$, which
  does not incorporate suboptimality under $\Mbar$.
  \item Then, in
  \pref{sec:convexity}, we show that allowing for arbitrary,
  potentially improper reference models $\Mbar\in\cMall$ never increases
  the value of the \CompShort beyond what is achieved by reference
  models $\Mbar\in\conv(\cM)$. This illustrates the fundamental role of convexity. In addition, we show that a similar
  equivalence holds for variants of the DEC that incorporate randomized
  mixture estimators.
\end{itemize}

\subsection{Relationship Between Constrained and Offset DEC}
\label{sec:constrained_offset}

It is immediate that one
can bound the constrained \CompShort by the offset \CompShort in a
tight fashion. Focusing on regret for
concreteness, we can use the method of Lagrange
multipliers to show that for any $\Mbar\in\cM$ and $\veps>0$,
\begin{align}
  \deccreg(\cM,\Mbar) &= \inf_{p\in\Delta(\Pi)}\sup_{M\in\cM}\crl*{
  \En_{\pi\sim{}p}\brk*{\gm(\pi)}\mid{}\Enp\brk*{\Dhels{M(\pi)}{\Mbar(\pi)}}\leq\veps^2
                        } \notag\\
  &= \inf_{p\in\Delta(\Pi)}\sup_{M\in\cM}\inf_{\gamma>0}\crl*{
  \En_{\pi\sim{}p}\brk*{\gm(\pi)} - \gamma\prn*{\Enp\brk*{\Dhels{M(\pi)}{\Mbar(\pi)}}-\veps^2}
    } \notag\\
      &\leq \inf_{\gamma>0}\inf_{p\in\Delta(\Pi)}\sup_{M\in\cM}\crl*{
  \En_{\pi\sim{}p}\brk*{\gm(\pi)} - \gamma\prn*{\Enp\brk*{\Dhels{M(\pi)}{\Mbar(\pi)}}-\veps^2}
        }\notag\\
        &= \inf_{\gamma>0}\crl*{
          \decoreg(\cM,\Mbar) + \gamma\veps^2
  }. \label{eq:constrained_to_offset}
\end{align}
The same approach yields an analogous inequality for PAC. For general
models $\Mbar\notin\cM$, the following slightly looser version of \eqref{eq:constrained_to_offset} holds.\footnote{An inequality that replaces the right-hand
  side of \eqref{eq:constrained_offset_regret1} with
  $\decoreg(\cMu,\Mbar)$ follows immediately by applying
  \eqref{eq:constrained_to_offset} to the class $\cMu$. The inequality
  \pref{eq:constrained_offset_regret1} is stronger, and the proof is
  more involved.}
\begin{proposition}
  \label{prop:constrained_offset_union}
For all $\Mbar\in\cMall$ and $\veps>0$, we have
\begin{align}
  \label{eq:constrained_offset_regret1}
  \deccreg(\cMu,\Mbar)
  \leq{} 
  8\cdot{}\inf_{\gamma>0}\crl*{\decoreg(\cM,\Mbar)\vz + \gamma\veps^2}  +
  7\veps.
\end{align}
\end{proposition}
Examples include:
\begin{itemize}
\item Whenever $\decoreg(\cM)\propto\frac{\Cprob}{\gamma}$, \pref{prop:constrained_offset_union} yields $\deccreg(\cM)\approxleq\veps\sqrt{\Cprob}$. In this
    case, both our results and the bounds in
    \citet{foster2021statistical} lead to
    $\En\brk*{\RegDM}\leq\bigoht\prn[\big]{\sqrt{\Cprob\cdot{}T\cdot\EstProbHel}}$.
  \item More
    generally, whenever
    $\decoreg(\cM)\propto\prn*{\frac{\Cprob}{\gamma}}^{\rho}$ for some
    $\rho<1$, then \pref{prop:constrained_offset_union} yields $\deccreg(\cM)\approxleq \Cprob^{\frac{\rho}{1+\rho}}\cdot\veps^{\frac{2\rho}{1+\rho}}$.
    \end{itemize}
    In what follows we
    investigate when and to what extent the constrained and offset
variants of the \CompShort can be related in the opposite direction to
\eqref{eq:constrained_to_offset}---both for regret and PAC.
  
\subsubsection{PAC DEC: Constrained Versus Offset}

For the PAC setting, the following result shows that the constrained and offset
\CompShort are equivalent up to logarithmic factors in most parameter regimes.

\begin{proposition}
  \label{prop:constrained_offset_pac}
  For all $\veps>0$ and $\Mbar\in\cMall$, we have
  \begin{align}
    \label{eq:constrained_offset_pac1}
    \deccpac(\cM,\Mbar) \leq \inf_{\gamma\geq{}0}\crl*{
          \decopac(\cM,\Mbar)\vz + \gamma\veps^2
  }.
  \end{align}
  On the other hand, %
  for all $\gamma\geq{}1$ and $\Mbar\in\cMall$, letting
  $L := 2\lceil \log 2\gamma \rceil$, we have
  \begin{align}
\decopac[\gamma \cdot (4L+1)](\MM, \Mbar) \leq & \frac{2}{\gamma} +  \sup_{\vep > 0}  \left\{ \deccpac(\MM, \Mbar)  - \frac{\gamma \vep^2}{4} \right\}.     \label{eq:constrained_offset_pac2}
  \end{align}
\end{proposition}
Whenever $\deccpac(\cM)\propto\veps^{\rho}$ for $\rho\leq{}1$, one
loses only logarithmic factors by passing to the offset \CompShort
using \eqref{eq:constrained_offset_pac1} and using
\eqref{eq:constrained_offset_pac2} to pass back to the constrained \CompShort. Yet, in the case
where the constrained \CompShort has ``fast'' behavior of the form
$\deccpac(\cM)\propto\Cprob\cdot\veps^{2}$ or
$\deccpac(\cM)\propto\indic\crl{\veps\geq1/\sqrt{\Cprob}}$, this
process is lossy (due to the $\frac{1}{\gamma}$ term in
\eqref{eq:constrained_offset_pac2}), and spoils the prospect of a
faster-than-$\sqrt{T}$ rate. This is why we present our results for
PAC in terms of the constrained \CompShort, and why we use a
dedicated algorithm tailored to the constrained \CompShort (\pref{alg:pac}) as opposed to a direct
adaptation of the algorithm based on the offset \CompShort in
\citet{foster2021statistical}; taking the latter approach and
combining it with \eqref{eq:constrained_offset_pac2} would not lead to
fast rates.

\subsubsection{Regret DEC: Constrained Versus Offset}

In light of the near-equivalence between constrained and offset
\CompShort for
the PAC setting, one might hope that a similar equivalence
would hold for regret. However, a naive adaptation of the techniques used to
prove \pref{prop:constrained_offset_pac} only leads to the following,
quantitatively weaker converse to \pref{prop:constrained_offset_union}.
\begin{proposition}
  \label{prop:constrained_offset_basic}
  For all $\gamma>0$ and $\Mbar\in\cMall$,
    \begin{align}
    \label{eq:constrained_offset_regret2}
    \decoreg(\cM,\Mbar) \leq 
          \deccreg[\gamma^{-1/2}](\cM,\Mbar).
  \end{align}
\end{proposition}
The bound on the offset \CompShort in
\eqref{eq:constrained_offset_regret2} is unsatisfying due to
the scale $\veps=\gamma^{-1/2}$ on the \rhs. For
example, in the case of the multi-armed bandit with $A$ actions, we
have $\deccreg(\cM)\propto\veps\sqrt{A}$ and
$\decoreg(\cM)\propto\frac{A}{\gamma}$, yet
\eqref{eq:constrained_offset_regret2} only yields the inequality
$\deccreg(\cM)\approxleq\sqrt{\frac{A}{\gamma}}$. This leads to
a suboptimal $A^{1/3}T^{2/3}$-type regret bound when plugged into the
upper bounds from \citet{foster2021statistical}
(cf. \eqref{eq:upper_old}). Unfortunately, the following result
shows that \pref{prop:constrained_offset_basic} is tight in general, even when $\Mbar\in\cM$.

\begin{proposition}
  \label{prop:constrained_offset_counterexample1}
  For all $\gamma \geq 1$, there exists a model class $\cM\subseteq\cMall$ such that for all $\vep \in (0,1)$,
  \begin{align}
\deccreg(\MM) = \sup_{\Mbar \in \co(\MM)} \deccreg(\MM \cup \{ \Mbar \}, \Mbar ) \leq O\left( \vep^2 \gamma^{1/2} \right)\nonumber,
  \end{align}
  so in particular $\deccreg(\MM) \leq O(\vep)$ for $\vep \leq \gamma^{-1/2}$. 
  Yet, there exists $\Mbar \in \MM$ such that %
  \begin{align}
\decoreg[\gamma](\MM, \Mbar) \geq \Omega \left( \gamma^{-1/2}\right)\nonumber.
  \end{align}
  \end{proposition}
This rules out the possibility of an inequality tighter than 
\eqref{eq:constrained_offset_regret2}, and shows that the constrained and
offset \CompShort have fundamentally different behavior for regret.

  \subsection{Localization}
  \label{sec:localization}

  For a scale parameter $\alpha>0$ and reference model
  $\Mbar\in\cMall$, define the following \emph{localized} subclass of
  $\cM$:
  \begin{align}
    \cMloc(\Mbar)\ldef\crl{M\in\cM\mid{}\fm(\pim)\leq\fmbar(\pimbar) + \alpha}.\label{eq:localization_one_sided}
  \end{align}
The tightest upper and lower bounds on
  regret in \citet{foster2021statistical} are stated in terms of the offset
  \CompShort for the localized class
  \pref{eq:localization_one_sided}, for an appropriate
  choice of $\alpha>0$ that depends on $T$ (see \pref{sec:related} for
  precise statements). Our
  bounds based on the constrained \CompShort avoid the explicit use of
  localization, but in what follows, we show that the constrained \CompShort \emph{implicitly} enforces
  a form of localization.
\begin{proposition}[Localization for PAC \CompShort]
  \label{prop:constrained_implies_localization_pac}
For all $\veps>0$ and $\Mbar\in\cMall$, letting
$\alphaveps\ldef\sqrt{3}\veps+\deccpac[\sqrt{6}\veps](\cM,\Mbar)$, we have
\begin{align*}
  \deccpac(\cM,\Mbar)
  \leq{} \deccpac[\sqrt{3}\veps](\cM_{\alphaveps}(\Mbar),\Mbar).
\end{align*}
\end{proposition}
A similar result holds for regret, but we require that the DEC
satisfies a slightly stronger version of the regularity condition
\pref{asm:dec}.
\begin{definition}[Strong regularity of \CompShort]
\label{def:growth}
  For $\Mbar\in\cMall$, the constrained
  DEC is said to satisfy the strong regularity condition relative to $\Mbar$ if
  there exist constants $\Creg\geq\sqrt{2}$ and $\creg<\Creg$ such
  that for all $\veps>0$,
\begin{align}
  \label{eq:regret_localization_growth}
\deccreg[\Creg\cdot\veps](\cM,\Mbar)\leq{}\creg^2\cdot\deccreg[\veps](\cM,\Mbar).
\end{align}
The constrained DEC is said to satisfy strong regularity relative to a class $\cM'\subseteq\cMall$ if for all
  $\veps>0$,
\begin{align}
  \label{eq:regret_localization_growth_global}
  \sup_{\Mbar\in\cM'}\deccreg[\Creg\cdot\veps](\cM \cup \{ \Mbar \},\Mbar)\leq{}\creg^2\cdot\sup_{\Mbar\in\cM'}\deccreg[\veps](\cM\cup\{\Mbar\},\Mbar).
\end{align}
\end{definition}
This condition is satisfied with $\Creg=2$ and $\creg=2^{\rho/2}$ whenever $\deccreg(\cM,\Mbar)\propto\veps^{\rho}$ for $\rho<2$.
\begin{proposition}[Localization for regret \CompShort]
  \label{prop:constrained_implies_localization_regret}
  Let $\Mbar\in\cMall$ be given, and assume that the strong regularity condition
  \pref{eq:regret_localization_growth} is satisfied relative to $\Mbar$. Then, for all $\veps>0$, letting $\alphaveps\ldef\Creg\cdot\veps+\deccreg[\Creg\cdot\veps](\cM,\Mbar)\leq\Creg^2\cdot\prn*{\veps+\deccreg[\veps](\cM,\Mbar)}$, we have
\begin{align*}
\deccreg(\cM,\Mbar)\leq{}\Cloc\cdot\deccreg[\Creg\cdot\veps](\cMloc[\alphaveps](\Mbar),\Mbar),
\end{align*}
where $\Cloc\ldef{}\prn*{\frac{1}{\creg^2} - \frac{1}{\Creg^2}}^{-1}$. 
\end{proposition}
Note that in the case where  $\deccreg(\cM,\Mbar)\propto\veps^{\rho}$
for a constant $\rho<2$, choosing $\Creg=2$ and $\creg=2^{\rho/2}$
gives $\Cloc=\bigoh(1)$. These results show that the constrained \CompShort---both for PAC and
regret---is equivalent (up to constants) to the constrained \CompShort for
the localized subclass $\cMloc(\Mbar)$, for a radius $\alpha$ that 
depends on the value of the \CompShort itself. In contrast, the offset
\CompShort does not automatically enforce any form of localization,
which explains why it was necessary to explicitly restrict to a
localized subclass in prior work.

\subsubsection{Constrained versus Offset \CompShort: Tighter Equivalence for Localized Classes}
Building on the insights in the prequel, we now show that for localized classes, it is possible to bound the
offset \CompShort for regret by the constrained \CompShort in a tighter fashion
that improves upon \pref{prop:constrained_offset_basic}.

\begin{proposition}
  \label{prop:constrained_offset_localized}
  Let $\alpha,\gamma>0$ and $\Mbar\in\cMall$ be given. For all
  $\veps>0$, we have
  \begin{align}
        \label{eq:constrained_offset_localized1}
    \decoreg(\MM_\alpha(\Mbar)\cup\crl{\Mbar}, \Mbar)
    \leq 
    \deccreg(\cMu,\Mbar) + \max \left\{0,\ \alpha + \frac{1}{2\gamma}  - \frac{\gamma \vep^2}{2} \right\},
  \end{align}
  which in particular yields
  \begin{align}
    \label{eq:constrained_offset_localized2}
         \decoreg(\MM_\alpha(\Mbar)\cup\crl{\Mbar}, \Mbar) 
      \leq  \deccreg[\sqrt{2\alpha/\gamma}](\cMu,\Mbar) + \frac{1}{2\gamma}.
  \end{align}
\end{proposition}
The bound in \eqref{eq:constrained_offset_localized2} replaces the
term $\deccreg[\sqrt{1/\gamma}](\cM)$ in
\pref{prop:constrained_offset_basic} with
$\deccreg[\sqrt{\alpha/\gamma}](\cM)$, leading to improvement when
$\alpha\ll 1$. Notably, the bound is strong enough that, by
combining it with 
\pref{prop:constrained_implies_localization_regret}, it is possible to
upper bound
the constrained \CompShort by the localized offset \CompShort, and then
pass back to the constrained \CompShort in a fashion that loses only
constant factors---at least whenever $\deccreg(\cM)\approxgeq\veps$. The following result uses this approach to derive a near-equivalence for the constrained \CompShort and localized offset
\CompShort; we also use this approach within the proof of \pref{prop:conv_equivalence_regret}.

\begin{proposition}
  \label{prop:equivalence}
  Whenever the strong regularity condition \pref{eq:regret_localization_growth} in
  \pref{def:growth} is satisfied for $\Mbar\in\cMall$, it holds that
  for all $\veps>0$, letting $\alphaequiv\ldef{}\gamma\veps^2$,
  \begin{align}
    \label{eq:improvement_upper}
    c_1\cdot{}\sup_{\gamma>c_3\veps^{-1}}\decoreg(\cM_{c_2\cdot\alphaequiv}(\Mbar),\Mbar)
    \leq{} \deccreg(\cMu,\Mbar)
    \leq
    c_1'\cdot{}\sup_{\gamma>c_3'\vep^{-1}}\decoreg(\cM_{c_2'\cdot\alphaequiv}(\Mbar),\Mbar)
    +c_4'\veps,
  \end{align}
  where $c_1,c_2,c_3>0$ are
  numerical constants and $c_1',c_2',c_3',c_4'>0$ are constants that
  depend only on $\Creg$ and $\creg$.
\end{proposition}
In light of this result, our upper bounds (\pref{thm:regret_upper})
can be thought of as improving prior work by achieving the tightest
possible localization radius (roughly, $\alpha=\bigoh(\gamma\veps^2)$ instead of
$\alpha=\bigoh(\gamma\veps^2 + \decoreg(\cM))$). \pref{sec:related}
gives examples for which this leads to quantitative improvement
in rate.

\subsection{Reference Models: Role of Suboptimality}
\label{sec:reference_regret}

We now turn our attention to understanding the role of the reference model $\Mbar$ with
respect to which the \CompText is defined. Recall that for regret, our upper
and lower bounds scale with
\begin{align}
  \deccreg(\cM)&=\sup_{\Mbar\in\conv(\cM)}\deccreg(\cMu,\Mbar) \label{eq:dec_constrained_full} \\
               &=
                 \sup_{\Mbar\in\conv(\cM)}\inf_{p\in\Delta(\Pi)}\sup_{M\in\cMu}\crl*{\En_{\pi\sim{}p}\brk*{
  \fm(\pim) - \fm(\pi)}
  \mid\En_{\pi\sim{}p}\brk*{\Dhels{M(\pi)}{\Mbar(\pi)}}\leq\veps^2
  }.\notag
\end{align}
By maximizing over $M\in\cMu$, this definition forces the min-player
to choose $p\in\Delta(\Pi)$ such that the suboptimality
$\En_{\pi\sim{}p}\brk*{\fmbar(\pimbar)-\fmbar(\pi)}$ under $\Mbar$ is small.
This is somewhat counterintuitive, since $\Mbar\in\conv(\cM)$ does not
necessarily lie in the class $\cM$, yet our results show that
$\deccreg(\cM)$ characterizes the minimax regret for $\cM$. A-priori,
one might expect that the quantity
$\sup_{\Mbar\in\conv(\cM)}\deccreg(\cM,\Mbar)$, which does not
incorporate suboptimality under $\Mbar$, would be a more natural
complexity measure. In what follows, we show that this quantity has
fundamentally different behavior from \eqref{eq:dec_constrained_full},
and that incorporating suboptimality under $\Mbar$ is essential to
characterize minimax regret.
\begin{proposition}
  \label{prop:regret_union_counterexample}
For any $\veps>0$ sufficiently small, there exists a model class $\cM$
such that
\begin{align}
  \label{eq:union_counterex1}
  \sup_{\Mbar\in\conv(\cM)}\deccreg(\cM,\Mbar)\leq{} c\cdot{}\veps,
\end{align}
yet
\begin{align}
    \label{eq:union_counterex2}
  \deccreg(\cM) = \sup_{\Mbar\in\conv(\cM)}\deccreg(\cMu,\Mbar)\geq{} c'\cdot\veps^{2/3},
\end{align}
where $c,c'>0$ are numerical constants.
\end{proposition}
It is straightforward to show that for the choice
$\veps=\vepslowerT\propto1/\sqrt{T}$, the optimal regret for the class in
\pref{prop:regret_union_counterexample} is
$\En\brk*{\RegDM}=\wt{\Theta}(T^{2/3})$. This result is recovered by
\pref{thm:regret_lower}, which scales with the quantity in \eqref{eq:union_counterex2}. However, the quantity in
\eqref{eq:union_counterex1} incorrectly suggests a $\sqrt{T}$-type rate, which
is not achievable.

For the offset \CompShort, the role of suboptimality under $\Mbar$ is
more subtle. It is possible to show that in general,
$\decoreg(\cMu,\Mbar)\gg\decoreg(\cM,\Mbar)$, analogous to
\pref{prop:regret_union_counterexample}, but
\pref{prop:constrained_offset_union} shows that the latter quantity
suffices to upper bound bound $\deccreg(\cMu,\Mbar)$.%

While the preceding discussion shows that incorporating suboptimality
under $\Mbar\notin\cM$ is necessary to obtain tight guarantees for
regret, the following result shows that this distinction is
largely inconsequential for PAC, and motivates the definition $\deccpac(\cM)=\sup_{\Mbar\in\conv(\cM)}\deccpac(\cM,\Mbar)$.
\begin{proposition}
  \label{prop:pac_union}
  For all $\Mbar\in\cMall$ and $\veps>0$,
  \begin{align}
        \deccpac(\cM\cup\crl{\Mbar},\Mbar)
    \leq{}     \deccpac[\sqrt{3}\veps](\cM,\Mbar) + 4\veps.
    \label{eq:pac_union}
  \end{align}
\end{proposition}

\subsection{Reference Models: Role of Convexity and Randomization}
  \label{sec:convexity}
  We now focus on understanding the role of \emph{improper} reference
  models $\Mbar\notin\cM$.
  Focusing on regret,
  our upper bound (\pref{thm:regret_upper}) scales with
\begin{equation}
  \label{eq:constrained_conv}
\deccreg[\vepsupperT](\cM)=\sup_{\Mbar\in\conv(\cM)}\deccreg[\vepsupperT](\cM\cup\crl{\Mbar},\Mbar),
\end{equation}
which maximizes over all possible reference models in the convex hull
$\conv(\cM)$. On the other hand, our lower
bound (\pref{thm:regret_lower}) scales with
\begin{equation}
  \label{eq:constrained_arbitrary}
\sup_{\Mbar\in\cMall}\deccreg[\vepslowerT](\cMu,\Mbar) \geq \deccreg[\vepslowerT](\cM).
\end{equation}
Both results allow for improper models $\Mbar\notin\cM$, but the
quantity \pref{eq:constrained_arbitrary} allows the reference model to be unconstrained, and could be
larger than the quantity \pref{eq:constrained_conv} a-priori. Why is there no
contradiction here? In what follows, we show that for both the
constrained and offset \CompShort, allowing for arbitrary,
unconstrained reference models as in \eqref{eq:constrained_arbitrary}
can only increase the value beyond that achieved by
$\Mbar\in\conv(\cM)$ by constant factors.

Before stating our results, let us mention a secondary, related goal, which is
to understand the role of \emph{randomized
  reference models}. \cite{foster2021statistical} introduce a variant of
the \CompText tailored to randomized (or, mixture) reference models,
in which $\Mbar$ is drawn from a distribution $\nu\in\Delta(\cM)$. We
define constrained and offset variants of this complexity measure for
$\nu\in\Delta(\cM)$ as follows:
\begin{align}
  \label{eq:dec_regret_randomized}
  &\deccregr(\cM,\nu)=
  \inf_{p\in\Delta(\Pi)}\sup_{M\in\cM}\crl*{\En_{\pi\sim{}p}\brk*{
  \fm(\pim) - \fm(\pi)}
  \mid\En_{\Mbar\sim\nu}\En_{\pi\sim{}p}\brk*{\Dhels{M(\pi)}{\Mbar(\pi)}}\leq\veps^2
    }, \\
  &\decoregr(\cM,\nu)=
  \inf_{p\in\Delta(\Pi)}\sup_{M\in\cM}\En_{\pi\sim{}p}\brk*{
  \fm(\pim) - \fm(\pi)
  - \gamma\cdot\En_{\Mbar\sim\nu}\brk*{\Dhels{M(\pi)}{\Mbar(\pi)}}
  }.
\end{align}
Recent work of \citet{chen2022unified} extends the results of
\citet{foster2021statistical} to provide regret bounds that scale with
$\sup_{\nu\in\Delta(\cM)}\decoreg(\cM,\nu)$, which one might hope to
be smaller than $\sup_{\Mbar\in\conv(\cM)}\decoreg(\cM,\Mbar)$ (it is
never larger due to 
Jensen's inequality). We show that this is not the case: For both
constrained and offset, the
randomized \CompShort is sandwiched between the \CompShort with
$\Mbar\in\cMall$ and the \CompShort with $\Mbar\in\conv(\cM)$.

\begin{proposition}
  \label{prop:conv_equivalence_regret}
  Suppose that \pref{ass:minimax} is satisfied. For all $\gamma>0$, we have
  \begin{align}
    \label{eq:conv_equivalence_regret1}
    \sup_{\Mbar\in\cMall}\decoreg(\cM,\Mbar)
    \leq{} \sup_{\nu\in\Delta(\cM)}\decoregr[\gamma/4](\cM,\nu)
    \leq{} \sup_{\Mbar\in\conv(\cM)}\decoreg[\gamma/4](\cM,\Mbar).
  \end{align}
  In addition, suppose that the strong regularity condition
  (\pref{def:growth}, \eqref{eq:regret_localization_growth_global}) is satisfied relative to $\cMall$.
  Then for all $\veps>0$, we have
    \begin{align}
    \label{eq:conv_equivalence_regret2}
      \sup_{\Mbar\in\cMall}\deccreg(\cMu,\Mbar)
      &\leq
      c_1\sup_{\nu\in\Delta(\cM)}\deccregr[c_2\veps](\cM\cup\crl{\Mbarnu},\nu) + c_3\veps\\
        &\leq{}
      c_1\sup_{\Mbar\in\conv(\cM)}\deccreg[c_2\veps](\cM\cup\crl{\Mbar},\Mbar) + c_3\veps,
    \end{align}
    where $\Mbarnu\ldef{}\En_{M'\sim\nu}\brk*{M'}$ and
    $c_1,c_2,c_3>0$ are constants that depend only on
    $\Creg,\Cloc>0$.
    \end{proposition}
    A similar equivalence holds for PAC; see
    \pref{sec:convexity_pac}. The main consequences of this result are
    as follows.
    \begin{itemize}
    \item Since allowing for arbitrary reference models
      $\Mbar\in\cMall$ never increases the value over reference models
      $\Mbar\in\conv(\cM)$, one can freely work with whichever version
      is more convenient, either for upper or lower bounds.
    \item From a statistical perspective, it is not possible to
      further tighten our results by working with the \CompShort with
      randomized estimators, since this complexity measure is never
      smaller than the variant with $\Mbar\in\conv(\cM)$ by more than constant factors.

    \end{itemize}
    We mention in passing that the proof of the equivalence
    \pref{eq:conv_equivalence_regret1} for
    the offset \CompShort is a simple consequence of the minimax
    theorem and convexity of squared Hellinger distance, but the proof
    of the equivalence \pref{eq:conv_equivalence_regret2} is quite
    involved, and uses the tools developed in \pref{sec:localization}
    to pass back and forth between the constrained and offset
    \CompShort. We are curious as to whether there is a simpler proof.

  \section{Improvement over Prior Work}
\label{sec:related}

In this section, we use the tools developed in \pref{sec:properties} to show
that the regret bounds in \pref{thm:regret_lower,thm:regret_upper}
always improve upon those in prior work
\citep{foster2021statistical,foster2022complexity}. We then highlight
some concrete model classes for which our bounds provide meaningful
improvement, and discuss additional related work.

\paragraph{Regret bounds from prior work}
Recall that for a model class $\cM$ and reference model $\Mbar\in\cM$,
we define the localized subclass around $\Mbar$ via
  \begin{equation}
    \label{eq:localized}
    \cMloc[\alpha](\Mbar) = \crl*{
      M\in\cM: \fmbar(\pimbar) \geq{} \fm(\pim) - \alpha
    }.
  \end{equation}
  where $\alpha>0$ is the radius.
  Focusing on finite classes for simplicity, the best upper
  bounds from prior work are those of \citet{foster2021statistical},
  which take the form
  \begin{equation}
  \label{eq:upper_old}
  \En\brk*{\RegDM}
  \leq{}
\bigoht(1)\cdot{}\min_{\gamma>0}\max\crl[\bigg]{\sup_{\Mbar\in\conv(\cM)}\decoreg(\cMloc[\alphaupper](\Mbar),\Mbar)\cdot{}T,\;
    \gamma\cdot{}\log\abs{\cM}},
\end{equation}
for $\alphaupper=\bigoht\prn[\big]{\decoreg(\cM) + \frac{\gamma}{T}\log\abs{\cM} 
  + \gamma^{-1}}$.  The best lower bounds from prior work are those of
\citet[Theorem D.1]{foster2022complexity}, which apply to all
algorithms with ``sub-Chebychev'' tail
behavior,\footnote{Sub-Chebychev algorithms are those for which the
  root-mean-squared regret is of the same order as the expected
  regret. \citet{foster2021statistical}
    provide lower bounds that do not require the assumption of
    sub-Chebychev tail behavior, but these results depend on the
    \CompShort for a
    smaller subclass of the form $\cMinf[\alpha](\Mbar) = \crl*{M\in\cM:
      \abs{\gm(\pi) - \gmbar(\pi)}\leq\alpha\;\;\forall{}\pi\in\Pi}$,
    and can be loose compared to \eqref{eq:lower_old}.} and scale as
      \begin{equation}
    \label{eq:lower_old}
    \En\brk*{\RegDM}\geq{}
    \bigom(1)\cdot\max_{\gamma>\sqrt{\Ct{}T}}\sup_{\Mbar\in\cM}\decoreg(\cMloc[\alphalower](\Mbar),\Mbar)\cdot{}T,
  \end{equation}
where $\Ct\ldef\bigoh(\log(T\wedge{}\abscont))$ and $\alphalower{}\ldef
  \Ct^{-1}\cdot\frac{\gamma}{T}$.

\paragraph{Our improvement}
The following result, which follows immediately from
\pref{prop:equivalence}, implies that the upper and
lower bounds in \pref{thm:regret_upper} and \pref{thm:regret_lower},
are always tighter than the guarantees in \eqref{eq:upper_old} and
\eqref{eq:lower_old}, respectively, under an appropriate regularity condition. 

\begin{corollary}
  \label{prop:improvement}
  Whenever the strong regularity condition (\pref{def:growth}) is satisfied for $\Mbar\in\cMall$ with $\Cloc,\Creg=\bigoh(1)$, we
  have that for all $\veps>0$ and $\gamma>0$,
  \begin{align}
    \label{eq:improvement_upper}
    \deccreg(\cMu,\Mbar)
    \leq
    \bigoh\prn*{\decoreg(\cMloc[\alphaupperabs](\Mbar),\Mbar)\vee{}0 +
    \gamma\veps^2 + \veps},
  \end{align}
  where $\alphaupperabs=\bigoh\prn*{\decoreg(\cM,\Mbar)\vee{}0 +
    \gamma\veps^2 + \gamma^{-1}}$. In addition, for all $\veps>0$, 
  $\gamma \geq \bigom(\veps^{-1})$, and $\Mbar\in\cMall$,
  \begin{align}
    \label{eq:improvement_lower}
    \deccreg(\cMu,\Mbar) \geq \decoreg(\cMloc[\alphalowerabs](\Mbar),\Mbar),
  \end{align}
  where $\alphalowerabs = \bigom\prn*{\gamma\veps^2 }$.
\end{corollary}

By applying \eqref{eq:improvement_upper} with
$\vepsupperT=\bigoht\prn[\Big]{\sqrt{\frac{\log\abs{\cM}}{T}}}$, we conclude that
the upper bound in \pref{thm:regret_upper} is always bounded above by the quantity
in \eqref{eq:upper_old} up to logarithmic factors in $T$ and $1/\delta$. Similarly, by applying \eqref{eq:improvement_lower}
with $\vepslowerT=\bigomt\prn*{\sqrt{\frac{1}{T}}}$ we see that the
lower bound in \pref{thm:regret_lower} is always bounded below by the
quantity in \eqref{eq:lower_old} up to $\log(T)$ factors and an additive $O(\sqrt{T})$ term. Beyond
simply scaling with a larger complexity measure, \pref{thm:regret_lower} 1) holds for arbitrary
algorithms, removing the sub-Chebychev assumption used by
\citet{foster2022complexity}, and 2) allows for improper reference models $\Mbar\notin\cM$.

We now provide concrete model classes for which our results
lead to quantitative improvements in rates. Our first example
is a model class for which our main upper bound
(\pref{thm:regret_upper}) improves over \citet{foster2021statistical}
by (implicitly) achieving a tighter localization
radius than \eqref{eq:upper_old}.

\begin{example}[Improvement from upper bound]
  \label{ex:upper_improvement}
  Consider a model class $\cMab$ parameterized by
  $\alpha\in(0,1/2]$, $\beta\in(0,1)$, and $A\in\bbN$.
  \begin{enumerate}
  \item $\Pi = \brk{A} \cup \{\pir \}$, where $\picirc$ is a
    ``revealing'' decision.
  \item $\cO=\brk{A}\cup\crl{\perp}$, where $\perp$ is a null symbol.
  \item We have $\cM=\crl*{\Mia}_{i\in\brk{A}}\cup\crl{\Mtila}$. For each $i \in \brk{A}$,
    the model $\Mia \in \cMab$ has rewards and observations defined as follows:
    \begin{enumerate}
    \item For $\pi \in \brk{A}$, $f\sups{\Mia}(\pi) = \frac 12 +
      \alpha \cdot \One{\pi = i}$, and $f\sups{\Mia}(\pir) = 0$. All
      $\pi \in \Pi$ have $r = f\sups{\Mia}(\pi)$ almost surely under $r \sim \Mia(\pi)$.
    \item For $\pi \in \brk{A}$, we receive the observation $o = \perp$.
     Selecting $\pir$ gives the observation $o=i\in\brk{A}$ with
      probability $\beta$ and $o=\perp$ with probability $1-\beta$.

    \end{enumerate}
  \item The model $\Mtila$ is defined as follows:
    \begin{enumerate}\item We have $\fmtila(\pi)=\frac{1}{2}$ for all $\pi\in\brk{A}$
      and $\fmtila(\picirc)=0$, with $r=\fmtila(\pi)$ almost surely
      under $r\sim\Mtila(\pi)$ for all $\pi\in\Pi$.
    \item All $\pi\in\brk{A}$ have $o=\perp$ almost surely. For
      $\picirc$, we observe $o=\perp$ with probability $1-\beta$ and
      $o\sim\unif(\brk{A})$ with probability $\beta$.
    \end{enumerate}

  \end{enumerate}

Let $\cM\ldef{}\cM^{\alpha_1,\beta}\cup\cM^{\alpha_2,\beta}$, with
$\alpha_1=1/2$, $\alpha_2\propto{}T^{-1/4}$, $\beta\propto{}T^{-1/2}$, and $A\propto{}T^{2}$. Then:
\begin{itemize}
\item The \etdp algorithm, via \pref{thm:regret_upper}, achieves
  $\En\brk*{\RegDM}\leq\bigoht(\sqrt{T})$.
\item The regret bound in \eqref{eq:upper_old} scales with $\bigomt(T^{5/8})$.
\end{itemize}
  
\end{example}

The next example is a model class for which our main lower bound
(\pref{thm:regret_lower}) improves over \citet{foster2021statistical},
as a consequence of allowing for improper reference models $\Mbar\notin\cM$.

\begin{example}[Improvement from lower bound]
  Let $A\in\bbN$ and $\Pi=\crl{1,\ldots,A}$. Consider the multi-armed
  bandit model class
  $\cM=\crl*{M_1,\ldots,M_A}$ consisting models of the form
  \[
    M_i(\pi)=\Ber(f_i(\pi)),
  \]
  where
  $f_i(\pi)\ldef{}\frac{1}{2}+\Delta\indic\crl{\pi=i}$. \citet{foster2021statistical}
  show that regardless of how $\Delta>0$ is chosen,
  $\decoreg(\cM,\Mbar)\leq\frac{1}{\gamma}$ for all $\gamma>0$ and
  $\Mbar\in\cM$, so the lower bound \pref{eq:lower_old} can at most give
  $\En\brk{\RegDM}\geq\bigom(\sqrt{T})$. On the other hand, by
  choosing $\Mbar(\pi)=\Ber(\frac{1}{2})$, which has $\Mbar\notin\cM$,
  it is straightforward to
  see that whenever $\Delta\propto\veps\sqrt{A}$, we have
  $\deccreg(\cM)\geq{}\deccreg(\cM,\Mbar)\geq\bigom(\veps\sqrt{A})$. Setting 
  $\Delta\propto\vepslowerT\cdot{}\sqrt{A}$, \pref{thm:regret_lower} gives
  \[
\En\brk*{\RegDM}\geq\bigomt(\sqrt{AT}),
\]
which is optimal. This shows that in general, allowing for improper reference
models $\Mbar\notin\cM$ is necessary to obtain tight lower bounds.
\end{example}

\subsection{Additional Related Work}
Concurrent work of \citet{chen2022unified}
independently discovered the offset variant of the PAC \CompText, and
used it to give upper and lower bounds for PAC sample complexity by
adapting the techniques of \citet{foster2021statistical}. Our guarantees for both regret and PAC are always tighter than these results, analogous to the
improvement we obtain over \citet{foster2021statistical} (see also
\pref{sec:convexity}), but our techniques are otherwise complementary.

\section{Additional Examples}
\label{sec:examples}
We close with some brief examples that showcase the behavior of the
constrained \CompText, as well as our upper and lower bounds, for standard model classes of interest.
For regret, \citet{foster2021statistical} provide lower bounds on the
(localized) offset \CompShort for a number of canonical models in bandits and
reinforcement learning. It is straightforward to derive lower bounds
on the constrained \CompShort by combining these results with
\pref{prop:improvement}. Likewise, \citet{foster2021statistical} give
global upper bounds on the offset \CompShort for the same examples,
which immediately lead to upper bounds on the constrained \CompShort
via \pref{prop:constrained_offset_union}. This approach leads to lower and upper bounds on the
 constrained \CompText for all of the
 examples considered in \citet{foster2021statistical}. We summarize
 these results, as well the implied lower bounds on regret, in
 \pref{table:regret}; upper bounds on regret are similar, but depend
 additionally on $\EstProbHel$. See \citet{foster2021statistical} for further
 background.

 \begin{table*}[htp]
\centering\resizebox{.8\columnwidth}{!}{
\begin{tabular}{ | c | c | c | }
\hline
  Setting & $\deccreg(\cM)$ & Lower Bound (\pref{thm:regret_lower}) \\
\hline 
  Multi-Armed Bandit  & $\veps\sqrt{A}$ & $\sqrt{AT}$ \\
  \hline 
Multi-Armed Bandit w/ gap & $\Delta\indic\crl{\veps > \Delta/\sqrt{A}}$ &  $A/\Delta$\\
  \hline
  Linear Bandit & $\veps\sqrt{d}$ & $\sqrt{dT}$ \\
    \hline
  Lipschitz Bandit  & $\veps^{1-\frac{d}{d+2}}$ & $T^{\frac{d+1}{d+2}}$ \\
  \hline
  ReLU Bandit  & $\indic\crl{\veps>2^{-\bigom(d)}}$ & $2^{\bigom(d)}$ \\
  \hline
  Tabular RL  & $\veps\sqrt{HSA}$ & $\sqrt{HSAT}$ \\
  \hline
  Linear MDP  & $\veps\sqrt{d}$ & $\sqrt{dT}$ \\
  \hline
    RL w/ linear $\Qstar$ & $\indic\crl{\veps\geq{}2^{-\bigom(d)}\vee{}2^{-\bigom(H)}}$ & $2^{\bigom(d)}\wedge{}2^{\bigom(H)}$  \\
  \hline
  Deterministic RL w/ linear $\Qstar$ & $\indic\crl{\veps\geq{} 1/\sqrt{d}}$& $d$ \\
  \hline
\end{tabular}
}
\caption{
Lower bounds for bandits and reinforcement learning recovered by the
constrained \CompText, where $A=\text{\#actions}$, $\Delta=\text{gap}$,
$d=\text{feature dim.}$, $H=\text{episode horizon}$, and
$S=\text{\#states}$. Numerical constants and $\log(T)$ factors are suppressed.
}
\label{table:regret}
\end{table*}

\paragraph{Example: Multi-armed bandit} We now sketch the approach to lower bounds
outlined above in greater detail, focusing on multi-armed bandits for
concreteness. \citet{foster2021statistical} show that for
when $\cM$ is the class of all multi-armed bandit instances with
$\Pi=\crl{1,\ldots,A}$ and Bernoulli rewards, there exists
$\Mbar\in\cM$ such that for all $\gamma\geq{}c_1\cdot{}A$,
\[
\sup_{\Mbar\in\cM}\decoreg(\cMloc[\alphagamma](\Mbar),\Mbar) \geq c_2\cdot\frac{A}{\gamma},
\]
where $\alphagamma\ldef{}c_3\cdot\frac{A}{\gamma}$, and $c_1,c_2,c_3>0$
are numerical constants. \pref{prop:improvement}
implies that for all $\veps>0$ and $\Mbar\in\cMall$,
\[
  \deccreg(\cM) \geq
  \sup_{\gamma>0}\decoreg(\cMloc[\alphalowerabs](\Mbar),\Mbar),
\]
 where $\alphalowerabs = c\cdot\gamma\veps^2$ for a sufficiently small
 numerical constant $c$. For any given $\veps>0$, if we set
 $\gamma=c'\cdot{}A^{1/2}/\veps$ for a sufficiently large constant
 $c'$, we have
 $\cMloc[\alphagamma](\Mbar)\subseteq\cMloc[\alphalowerabs](\Mbar)$,
 and we conclude that
 \[
   \deccreg(\cM) \geq \bigom\prn*{\veps\sqrt{A}}
 \]
 for all $\veps\leq{}c''\cdot{}A^{-1/2}$, where $c''$ is a
 sufficiently small constant. Plugging this lower bound on the
 \CompShort into \pref{thm:regret_lower} yields a lower bound on
 regret of the form $\En\brk*{\RegDM}\geq\bigomt(\sqrt{AT})$.

\subsection*{Acknowledgements}
We thank Jian Qian, Sasha Rakhlin, Rob Schapire, and Andrew Wagenmaker for helpful
comments and discussions.

\clearpage

\bibliography{refs} 

\begin{thebibliography}{38}
\providecommand{\natexlab}[1]{#1}
\providecommand{\url}[1]{\texttt{#1}}
\expandafter\ifx\csname urlstyle\endcsname\relax
  \providecommand{\doi}[1]{doi: #1}\else
  \providecommand{\doi}{doi: \begingroup \urlstyle{rm}\Url}\fi

\bibitem[Abernethy et~al.(2008)Abernethy, Hazan, and
  Rakhlin]{abernethy2008competing}
Jacob Abernethy, Elad Hazan, and Alexander Rakhlin.
\newblock Competing in the dark: An efficient algorithm for bandit linear
  optimization.
\newblock In \emph{Proc. of the 21st Annual Conference on Learning Theory
  (COLT)}, 2008.

\bibitem[Ayoub et~al.(2020)Ayoub, Jia, Szepesvari, Wang, and
  Yang]{ayoub2020model}
Alex Ayoub, Zeyu Jia, Csaba Szepesvari, Mengdi Wang, and Lin Yang.
\newblock Model-based reinforcement learning with value-targeted regression.
\newblock In \emph{International Conference on Machine Learning}, pages
  463--474. PMLR, 2020.

\bibitem[Bubeck et~al.(2011)Bubeck, Munos, Stoltz, and
  Szepesv{\'a}ri]{bubeck2011x}
S{\'e}bastien Bubeck, R{\'e}mi Munos, Gilles Stoltz, and Csaba Szepesv{\'a}ri.
\newblock {X}-armed bandits.
\newblock \emph{Journal of Machine Learning Research}, 12\penalty0 (5), 2011.

\bibitem[Bubeck et~al.(2012)Bubeck, Cesa-Bianchi, and
  Kakade]{bubeck2012towards}
S{\'e}bastien Bubeck, Nicolo Cesa-Bianchi, and Sham~M Kakade.
\newblock Towards minimax policies for online linear optimization with bandit
  feedback.
\newblock In \emph{Conference on Learning Theory}, pages 41--1. JMLR Workshop
  and Conference Proceedings, 2012.

\bibitem[Bubeck et~al.(2017)Bubeck, Lee, and Eldan]{bubeck2017kernel}
S{\'e}bastien Bubeck, Yin~Tat Lee, and Ronen Eldan.
\newblock Kernel-based methods for bandit convex optimization.
\newblock In \emph{Proceedings of the 49th Annual ACM SIGACT Symposium on
  Theory of Computing}, pages 72--85, 2017.

\bibitem[Chen et~al.(2022)Chen, Mei, and Bai]{chen2022unified}
Fan Chen, Song Mei, and Yu~Bai.
\newblock Unified algorithms for rl with decision-estimation coefficients:
  No-regret, pac, and reward-free learning.
\newblock \emph{arXiv preprint arXiv:2209.11745}, 2022.

\bibitem[Dani et~al.(2007)Dani, Hayes, and Kakade]{dani2007price}
Varsha Dani, Thomas~P Hayes, and Sham Kakade.
\newblock The price of bandit information for online optimization.
\newblock 2007.

\bibitem[Dean et~al.(2020)Dean, Mania, Matni, Recht, and Tu]{dean2020sample}
Sarah Dean, Horia Mania, Nikolai Matni, Benjamin Recht, and Stephen Tu.
\newblock On the sample complexity of the linear quadratic regulator.
\newblock \emph{Foundations of Computational Mathematics}, 20\penalty0
  (4):\penalty0 633--679, 2020.

\bibitem[Dong et~al.(2019)Dong, Van~Roy, and Zhou]{dong2019provably}
Shi Dong, Benjamin Van~Roy, and Zhengyuan Zhou.
\newblock Provably efficient reinforcement learning with aggregated states.
\newblock \emph{arXiv preprint arXiv:1912.06366}, 2019.

\bibitem[Donoho and Liu(1987)]{donoho1987geometrizing}
David~L Donoho and Richard~C Liu.
\newblock Geometrizing rates of convergence.
\newblock \emph{Annals of Statistics}, 1987.

\bibitem[Donoho and Liu(1991{\natexlab{a}})]{donoho1991geometrizingii}
David~L Donoho and Richard~C Liu.
\newblock Geometrizing rates of convergence, {II}.
\newblock \emph{The Annals of Statistics}, pages 633--667, 1991{\natexlab{a}}.

\bibitem[Donoho and Liu(1991{\natexlab{b}})]{donoho1991geometrizingiii}
David~L Donoho and Richard~C Liu.
\newblock Geometrizing rates of convergence, {III}.
\newblock \emph{The Annals of Statistics}, pages 668--701, 1991{\natexlab{b}}.

\bibitem[Du et~al.(2019)Du, Krishnamurthy, Jiang, Agarwal, Dudik, and
  Langford]{du2019latent}
Simon Du, Akshay Krishnamurthy, Nan Jiang, Alekh Agarwal, Miroslav Dudik, and
  John Langford.
\newblock Provably efficient {RL} with rich observations via latent state
  decoding.
\newblock In \emph{International Conference on Machine Learning}, pages
  1665--1674. PMLR, 2019.

\bibitem[Du et~al.(2021)Du, Kakade, Lee, Lovett, Mahajan, Sun, and
  Wang]{du2021bilinear}
Simon~S Du, Sham~M Kakade, Jason~D Lee, Shachar Lovett, Gaurav Mahajan, Wen
  Sun, and Ruosong Wang.
\newblock Bilinear classes: A structural framework for provable generalization
  in {RL}.
\newblock \emph{International Conference on Machine Learning}, 2021.

\bibitem[Durrett(2019)]{durrett2019probability}
Richard Durrett.
\newblock \emph{Probability: theory and examples}.
\newblock Duxbury Press, Belmont, CA, fifth edition, 2019.
\newblock ISBN 0-534-24318-5.

\bibitem[Flaxman et~al.(2005)Flaxman, Kalai, and McMahan]{flaxman2005online}
Abraham~D Flaxman, Adam~Tauman Kalai, and H~Brendan McMahan.
\newblock Online convex optimization in the bandit setting: gradient descent
  without a gradient.
\newblock In \emph{Proceedings of the sixteenth annual ACM-SIAM symposium on
  Discrete algorithms}, pages 385--394, 2005.

\bibitem[Foster et~al.(2021)Foster, Kakade, Qian, and
  Rakhlin]{foster2021statistical}
Dylan~J Foster, Sham~M Kakade, Jian Qian, and Alexander Rakhlin.
\newblock The statistical complexity of interactive decision making.
\newblock \emph{arXiv preprint arXiv:2112.13487}, 2021.

\bibitem[Foster et~al.(2022{\natexlab{a}})Foster, Golowich, Qian, Rakhlin, and
  Sekhari]{foster2022note}
Dylan~J Foster, Noah Golowich, Jian Qian, Alexander Rakhlin, and Ayush Sekhari.
\newblock A note on model-free reinforcement learning with the
  decision-estimation coefficient.
\newblock \emph{arXiv preprint arXiv:2211.14250}, 2022{\natexlab{a}}.

\bibitem[Foster et~al.(2022{\natexlab{b}})Foster, Rakhlin, Sekhari, and
  Sridharan]{foster2022complexity}
Dylan~J Foster, Alexander Rakhlin, Ayush Sekhari, and Karthik Sridharan.
\newblock On the complexity of adversarial decision making.
\newblock \emph{arXiv preprint arXiv:2206.13063}, 2022{\natexlab{b}}.

\bibitem[Jiang et~al.(2017)Jiang, Krishnamurthy, Agarwal, Langford, and
  Schapire]{jiang2017contextual}
Nan Jiang, Akshay Krishnamurthy, Alekh Agarwal, John Langford, and Robert~E
  Schapire.
\newblock Contextual decision processes with low {Bellman} rank are
  {PAC}-learnable.
\newblock In \emph{International Conference on Machine Learning}, pages
  1704--1713, 2017.

\bibitem[Jin et~al.(2020)Jin, Yang, Wang, and Jordan]{jin2020provably}
Chi Jin, Zhuoran Yang, Zhaoran Wang, and Michael~I Jordan.
\newblock Provably efficient reinforcement learning with linear function
  approximation.
\newblock In \emph{Conference on Learning Theory}, pages 2137--2143, 2020.

\bibitem[Jin et~al.(2021)Jin, Liu, and Miryoosefi]{jin2021bellman}
Chi Jin, Qinghua Liu, and Sobhan Miryoosefi.
\newblock Bellman eluder dimension: New rich classes of {RL} problems, and
  sample-efficient algorithms.
\newblock \emph{Neural Information Processing Systems}, 2021.

\bibitem[Kleinberg(2004)]{kleinberg2004nearly}
Robert Kleinberg.
\newblock Nearly tight bounds for the continuum-armed bandit problem.
\newblock \emph{Advances in Neural Information Processing Systems},
  17:\penalty0 697--704, 2004.

\bibitem[Krishnamurthy et~al.(2016)Krishnamurthy, Agarwal, and
  Langford]{krishnamurthy2016pac}
Akshay Krishnamurthy, Alekh Agarwal, and John Langford.
\newblock {PAC} reinforcement learning with rich observations.
\newblock In \emph{Advances in Neural Information Processing Systems}, pages
  1840--1848, 2016.

\bibitem[Lattimore(2020)]{lattimore2020improved}
Tor Lattimore.
\newblock Improved regret for zeroth-order adversarial bandit convex
  optimisation.
\newblock \emph{Mathematical Statistics and Learning}, 2\penalty0 (3):\penalty0
  311--334, 2020.

\bibitem[Li(2009)]{li2009unifying}
Lihong Li.
\newblock \emph{A unifying framework for computational reinforcement learning
  theory}.
\newblock Rutgers, The State University of New Jersey---New Brunswick, 2009.

\bibitem[Lillicrap et~al.(2015)Lillicrap, Hunt, Pritzel, Heess, Erez, Tassa,
  Silver, and Wierstra]{lillicrap2015continuous}
Timothy~P Lillicrap, Jonathan~J Hunt, Alexander Pritzel, Nicolas Heess, Tom
  Erez, Yuval Tassa, David Silver, and Daan Wierstra.
\newblock Continuous control with deep reinforcement learning.
\newblock \emph{arXiv preprint arXiv:1509.02971}, 2015.

\bibitem[Magureanu et~al.(2014)Magureanu, Combes, and
  Proutiere]{magureanu2014lipschitz}
Stefan Magureanu, Richard Combes, and Alexandre Proutiere.
\newblock Lipschitz bandits: Regret lower bound and optimal algorithms.
\newblock In \emph{Conference on Learning Theory}, pages 975--999. PMLR, 2014.

\bibitem[Mnih et~al.(2015)Mnih, Kavukcuoglu, Silver, Rusu, Veness, Bellemare,
  Graves, Riedmiller, Fidjeland, Ostrovski, et~al.]{mnih2015human}
Volodymyr Mnih, Koray Kavukcuoglu, David Silver, Andrei~A Rusu, Joel Veness,
  Marc~G Bellemare, Alex Graves, Martin Riedmiller, Andreas~K Fidjeland, Georg
  Ostrovski, et~al.
\newblock Human-level control through deep reinforcement learning.
\newblock \emph{Nature}, 518\penalty0 (7540):\penalty0 529, 2015.

\bibitem[Modi et~al.(2020)Modi, Jiang, Tewari, and Singh]{modi2020sample}
Aditya Modi, Nan Jiang, Ambuj Tewari, and Satinder Singh.
\newblock Sample complexity of reinforcement learning using linearly combined
  model ensembles.
\newblock In \emph{International Conference on Artificial Intelligence and
  Statistics}, pages 2010--2020. PMLR, 2020.

\bibitem[Russo and Van~Roy(2013)]{russo2013eluder}
Daniel Russo and Benjamin Van~Roy.
\newblock Eluder dimension and the sample complexity of optimistic exploration.
\newblock In \emph{Advances in Neural Information Processing Systems}, pages
  2256--2264, 2013.

\bibitem[Silver et~al.(2016)Silver, Huang, Maddison, Guez, Sifre, Van
  Den~Driessche, Schrittwieser, Antonoglou, Panneershelvam, Lanctot,
  et~al.]{silver2016mastering}
David Silver, Aja Huang, Chris~J Maddison, Arthur Guez, Laurent Sifre, George
  Van Den~Driessche, Julian Schrittwieser, Ioannis Antonoglou, Veda
  Panneershelvam, Marc Lanctot, et~al.
\newblock Mastering the game of go with deep neural networks and tree search.
\newblock \emph{nature}, 529\penalty0 (7587):\penalty0 484, 2016.

\bibitem[Sun et~al.(2019)Sun, Jiang, Krishnamurthy, Agarwal, and
  Langford]{sun2019model}
Wen Sun, Nan Jiang, Akshay Krishnamurthy, Alekh Agarwal, and John Langford.
\newblock Model-based {RL} in contextual decision processes: {PAC} bounds and
  exponential improvements over model-free approaches.
\newblock In \emph{Conference on learning theory}, pages 2898--2933. PMLR,
  2019.

\bibitem[Tsybakov(2008)]{tsybakov2008introduction}
Alexandre~B Tsybakov.
\newblock \emph{Introduction to Nonparametric Estimation}.
\newblock Springer Publishing Company, Incorporated, 2008.

\bibitem[Wang et~al.(2020)Wang, Salakhutdinov, and Yang]{wang2020provably}
Ruosong Wang, Russ~R Salakhutdinov, and Lin Yang.
\newblock Reinforcement learning with general value function approximation:
  Provably efficient approach via bounded eluder dimension.
\newblock \emph{Advances in Neural Information Processing Systems}, 33, 2020.

\bibitem[Yang and Wang(2019)]{yang2019sample}
Lin Yang and Mengdi Wang.
\newblock Sample-optimal parametric {Q}-learning using linearly additive
  features.
\newblock In \emph{International Conference on Machine Learning}, pages
  6995--7004. PMLR, 2019.

\bibitem[Yu(1997)]{yu1997assouad}
Bin Yu.
\newblock Assouad, fano, and le cam.
\newblock In \emph{Festschrift for Lucien Le Cam}, pages 423--435. Springer,
  1997.

\bibitem[Zhou et~al.(2021)Zhou, Gu, and Szepesvari]{zhou2021nearly}
Dongruo Zhou, Quanquan Gu, and Csaba Szepesvari.
\newblock Nearly minimax optimal reinforcement learning for linear mixture
  markov decision processes.
\newblock In \emph{Conference on Learning Theory}, pages 4532--4576. PMLR,
  2021.

\end{thebibliography}


\begin{thebibliography}{39}
\providecommand{\natexlab}[1]{#1}
\providecommand{\url}[1]{\texttt{#1}}
\expandafter\ifx\csname urlstyle\endcsname\relax
  \providecommand{\doi}[1]{doi: #1}\else
  \providecommand{\doi}{doi: \begingroup \urlstyle{rm}\Url}\fi

\bibitem[Abernethy et~al.(2008)Abernethy, Hazan, and
  Rakhlin]{abernethy2008competing}
Jacob Abernethy, Elad Hazan, and Alexander Rakhlin.
\newblock Competing in the dark: An efficient algorithm for bandit linear
  optimization.
\newblock In \emph{Proc. of the 21st Annual Conference on Learning Theory
  (COLT)}, 2008.

\bibitem[Audibert and Bubeck(2009)]{audibert2009minimax}
Jean-Yves Audibert and S{\'e}bastien Bubeck.
\newblock Minimax policies for adversarial and stochastic bandits.
\newblock In \emph{COLT}, volume~7, pages 1--122, 2009.

\bibitem[Auer et~al.(2002)Auer, Cesa-Bianchi, Freund, and
  Schapire]{auer2002non}
Peter Auer, Nicolo Cesa-Bianchi, Yoav Freund, and Robert~E. Schapire.
\newblock The nonstochastic multiarmed bandit problem.
\newblock \emph{SIAM Journal on Computing}, 32\penalty0 (1):\penalty0 48--77,
  2002.

\bibitem[Ayoub et~al.(2020)Ayoub, Jia, Szepesvari, Wang, and
  Yang]{ayoub2020model}
Alex Ayoub, Zeyu Jia, Csaba Szepesvari, Mengdi Wang, and Lin Yang.
\newblock Model-based reinforcement learning with value-targeted regression.
\newblock In \emph{International Conference on Machine Learning}, pages
  463--474. PMLR, 2020.

\bibitem[Bubeck et~al.(2011)Bubeck, Munos, Stoltz, and
  Szepesv{\'a}ri]{bubeck2011x}
S{\'e}bastien Bubeck, R{\'e}mi Munos, Gilles Stoltz, and Csaba Szepesv{\'a}ri.
\newblock {X}-armed bandits.
\newblock \emph{Journal of Machine Learning Research}, 12\penalty0 (5), 2011.

\bibitem[Bubeck et~al.(2012)Bubeck, Cesa-Bianchi, and
  Kakade]{bubeck2012towards}
S{\'e}bastien Bubeck, Nicolo Cesa-Bianchi, and Sham~M Kakade.
\newblock Towards minimax policies for online linear optimization with bandit
  feedback.
\newblock In \emph{Conference on Learning Theory}, pages 41--1. JMLR Workshop
  and Conference Proceedings, 2012.

\bibitem[Bubeck et~al.(2017)Bubeck, Lee, and Eldan]{bubeck2017kernel}
S{\'e}bastien Bubeck, Yin~Tat Lee, and Ronen Eldan.
\newblock Kernel-based methods for bandit convex optimization.
\newblock In \emph{Proceedings of the 49th Annual ACM SIGACT Symposium on
  Theory of Computing}, pages 72--85, 2017.

\bibitem[Chen et~al.(2022)Chen, Mei, and Bai]{chen2022unified}
Fan Chen, Song Mei, and Yu~Bai.
\newblock Unified algorithms for rl with decision-estimation coefficients:
  No-regret, pac, and reward-free learning.
\newblock \emph{arXiv preprint arXiv:2209.11745}, 2022.

\bibitem[Dani et~al.(2007)Dani, Hayes, and Kakade]{dani2007price}
Varsha Dani, Thomas~P Hayes, and Sham Kakade.
\newblock The price of bandit information for online optimization.
\newblock 2007.

\bibitem[Dean et~al.(2020)Dean, Mania, Matni, Recht, and Tu]{dean2020sample}
Sarah Dean, Horia Mania, Nikolai Matni, Benjamin Recht, and Stephen Tu.
\newblock On the sample complexity of the linear quadratic regulator.
\newblock \emph{Foundations of Computational Mathematics}, 20\penalty0
  (4):\penalty0 633--679, 2020.

\bibitem[Dong et~al.(2019)Dong, Van~Roy, and Zhou]{dong2019provably}
Shi Dong, Benjamin Van~Roy, and Zhengyuan Zhou.
\newblock Provably efficient reinforcement learning with aggregated states.
\newblock \emph{arXiv preprint arXiv:1912.06366}, 2019.

\bibitem[Donoho and Liu(1987)]{donoho1987geometrizing}
David~L Donoho and Richard~C Liu.
\newblock Geometrizing rates of convergence.
\newblock \emph{Annals of Statistics}, 1987.

\bibitem[Donoho and Liu(1991{\natexlab{a}})]{donoho1991geometrizingii}
David~L Donoho and Richard~C Liu.
\newblock Geometrizing rates of convergence, {II}.
\newblock \emph{The Annals of Statistics}, pages 633--667, 1991{\natexlab{a}}.

\bibitem[Donoho and Liu(1991{\natexlab{b}})]{donoho1991geometrizingiii}
David~L Donoho and Richard~C Liu.
\newblock Geometrizing rates of convergence, {III}.
\newblock \emph{The Annals of Statistics}, pages 668--701, 1991{\natexlab{b}}.

\bibitem[Du et~al.(2019)Du, Krishnamurthy, Jiang, Agarwal, Dudik, and
  Langford]{du2019latent}
Simon Du, Akshay Krishnamurthy, Nan Jiang, Alekh Agarwal, Miroslav Dudik, and
  John Langford.
\newblock Provably efficient {RL} with rich observations via latent state
  decoding.
\newblock In \emph{International Conference on Machine Learning}, pages
  1665--1674. PMLR, 2019.

\bibitem[Du et~al.(2021)Du, Kakade, Lee, Lovett, Mahajan, Sun, and
  Wang]{du2021bilinear}
Simon~S Du, Sham~M Kakade, Jason~D Lee, Shachar Lovett, Gaurav Mahajan, Wen
  Sun, and Ruosong Wang.
\newblock Bilinear classes: A structural framework for provable generalization
  in {RL}.
\newblock \emph{International Conference on Machine Learning}, 2021.

\bibitem[Flaxman et~al.(2005)Flaxman, Kalai, and McMahan]{flaxman2005online}
Abraham~D Flaxman, Adam~Tauman Kalai, and H~Brendan McMahan.
\newblock Online convex optimization in the bandit setting: gradient descent
  without a gradient.
\newblock In \emph{Proceedings of the sixteenth annual ACM-SIAM symposium on
  Discrete algorithms}, pages 385--394, 2005.

\bibitem[Foster et~al.(2021)Foster, Kakade, Qian, and
  Rakhlin]{foster2021statistical}
Dylan~J Foster, Sham~M Kakade, Jian Qian, and Alexander Rakhlin.
\newblock The statistical complexity of interactive decision making.
\newblock \emph{arXiv preprint arXiv:2112.13487}, 2021.

\bibitem[Foster et~al.(2022)Foster, Rakhlin, Sekhari, and
  Sridharan]{foster2022complexity}
Dylan~J Foster, Alexander Rakhlin, Ayush Sekhari, and Karthik Sridharan.
\newblock On the complexity of adversarial decision making.
\newblock \emph{arXiv preprint arXiv:2206.13063}, 2022.

\bibitem[Hazan and Kale(2011)]{hazan2011better}
Elad Hazan and Satyen Kale.
\newblock Better algorithms for benign bandits.
\newblock \emph{Journal of Machine Learning Research}, 12\penalty0 (4), 2011.

\bibitem[Jiang et~al.(2017)Jiang, Krishnamurthy, Agarwal, Langford, and
  Schapire]{jiang2017contextual}
Nan Jiang, Akshay Krishnamurthy, Alekh Agarwal, John Langford, and Robert~E
  Schapire.
\newblock Contextual decision processes with low {Bellman} rank are
  {PAC}-learnable.
\newblock In \emph{International Conference on Machine Learning}, pages
  1704--1713, 2017.

\bibitem[Jin et~al.(2020)Jin, Yang, Wang, and Jordan]{jin2020provably}
Chi Jin, Zhuoran Yang, Zhaoran Wang, and Michael~I Jordan.
\newblock Provably efficient reinforcement learning with linear function
  approximation.
\newblock In \emph{Conference on Learning Theory}, pages 2137--2143, 2020.

\bibitem[Jin et~al.(2021)Jin, Liu, and Miryoosefi]{jin2021bellman}
Chi Jin, Qinghua Liu, and Sobhan Miryoosefi.
\newblock Bellman eluder dimension: New rich classes of {RL} problems, and
  sample-efficient algorithms.
\newblock \emph{Neural Information Processing Systems}, 2021.

\bibitem[Kleinberg(2004)]{kleinberg2004nearly}
Robert Kleinberg.
\newblock Nearly tight bounds for the continuum-armed bandit problem.
\newblock \emph{Advances in Neural Information Processing Systems},
  17:\penalty0 697--704, 2004.

\bibitem[Kober et~al.(2013)Kober, Bagnell, and Peters]{kober2013reinforcement}
Jens Kober, J~Andrew Bagnell, and Jan Peters.
\newblock Reinforcement learning in robotics: A survey.
\newblock \emph{The International Journal of Robotics Research}, 32\penalty0
  (11):\penalty0 1238--1274, 2013.

\bibitem[Krishnamurthy et~al.(2016)Krishnamurthy, Agarwal, and
  Langford]{krishnamurthy2016pac}
Akshay Krishnamurthy, Alekh Agarwal, and John Langford.
\newblock {PAC} reinforcement learning with rich observations.
\newblock In \emph{Advances in Neural Information Processing Systems}, pages
  1840--1848, 2016.

\bibitem[Lattimore(2020)]{lattimore2020improved}
Tor Lattimore.
\newblock Improved regret for zeroth-order adversarial bandit convex
  optimisation.
\newblock \emph{Mathematical Statistics and Learning}, 2\penalty0 (3):\penalty0
  311--334, 2020.

\bibitem[Li et~al.(2016)Li, Monroe, Ritter, Jurafsky, Galley, and
  Gao]{li2016deep}
Jiwei Li, Will Monroe, Alan Ritter, Dan Jurafsky, Michel Galley, and Jianfeng
  Gao.
\newblock Deep reinforcement learning for dialogue generation.
\newblock In \emph{EMNLP}, 2016.

\bibitem[Li(2009)]{li2009unifying}
Lihong Li.
\newblock \emph{A unifying framework for computational reinforcement learning
  theory}.
\newblock Rutgers, The State University of New Jersey---New Brunswick, 2009.

\bibitem[Lillicrap et~al.(2015)Lillicrap, Hunt, Pritzel, Heess, Erez, Tassa,
  Silver, and Wierstra]{lillicrap2015continuous}
Timothy~P Lillicrap, Jonathan~J Hunt, Alexander Pritzel, Nicolas Heess, Tom
  Erez, Yuval Tassa, David Silver, and Daan Wierstra.
\newblock Continuous control with deep reinforcement learning.
\newblock \emph{arXiv preprint arXiv:1509.02971}, 2015.

\bibitem[Magureanu et~al.(2014)Magureanu, Combes, and
  Proutiere]{magureanu2014lipschitz}
Stefan Magureanu, Richard Combes, and Alexandre Proutiere.
\newblock Lipschitz bandits: Regret lower bound and optimal algorithms.
\newblock In \emph{Conference on Learning Theory}, pages 975--999. PMLR, 2014.

\bibitem[Mnih et~al.(2015)Mnih, Kavukcuoglu, Silver, Rusu, Veness, Bellemare,
  Graves, Riedmiller, Fidjeland, Ostrovski, et~al.]{mnih2015human}
Volodymyr Mnih, Koray Kavukcuoglu, David Silver, Andrei~A Rusu, Joel Veness,
  Marc~G Bellemare, Alex Graves, Martin Riedmiller, Andreas~K Fidjeland, Georg
  Ostrovski, et~al.
\newblock Human-level control through deep reinforcement learning.
\newblock \emph{Nature}, 518\penalty0 (7540):\penalty0 529, 2015.

\bibitem[Modi et~al.(2020)Modi, Jiang, Tewari, and Singh]{modi2020sample}
Aditya Modi, Nan Jiang, Ambuj Tewari, and Satinder Singh.
\newblock Sample complexity of reinforcement learning using linearly combined
  model ensembles.
\newblock In \emph{International Conference on Artificial Intelligence and
  Statistics}, pages 2010--2020. PMLR, 2020.

\bibitem[Russo and Van~Roy(2013)]{russo2013eluder}
Daniel Russo and Benjamin Van~Roy.
\newblock Eluder dimension and the sample complexity of optimistic exploration.
\newblock In \emph{Advances in Neural Information Processing Systems}, pages
  2256--2264, 2013.

\bibitem[Silver et~al.(2016)Silver, Huang, Maddison, Guez, Sifre, Van
  Den~Driessche, Schrittwieser, Antonoglou, Panneershelvam, Lanctot,
  et~al.]{silver2016mastering}
David Silver, Aja Huang, Chris~J Maddison, Arthur Guez, Laurent Sifre, George
  Van Den~Driessche, Julian Schrittwieser, Ioannis Antonoglou, Veda
  Panneershelvam, Marc Lanctot, et~al.
\newblock Mastering the game of go with deep neural networks and tree search.
\newblock \emph{nature}, 529\penalty0 (7587):\penalty0 484, 2016.

\bibitem[Sun et~al.(2019)Sun, Jiang, Krishnamurthy, Agarwal, and
  Langford]{sun2019model}
Wen Sun, Nan Jiang, Akshay Krishnamurthy, Alekh Agarwal, and John Langford.
\newblock Model-based {RL} in contextual decision processes: {PAC} bounds and
  exponential improvements over model-free approaches.
\newblock In \emph{Conference on learning theory}, pages 2898--2933. PMLR,
  2019.

\bibitem[Wang et~al.(2020)Wang, Salakhutdinov, and Yang]{wang2020provably}
Ruosong Wang, Russ~R Salakhutdinov, and Lin Yang.
\newblock Reinforcement learning with general value function approximation:
  Provably efficient approach via bounded eluder dimension.
\newblock \emph{Advances in Neural Information Processing Systems}, 33, 2020.

\bibitem[Yang and Wang(2019)]{yang2019sample}
Lin Yang and Mengdi Wang.
\newblock Sample-optimal parametric {Q}-learning using linearly additive
  features.
\newblock In \emph{International Conference on Machine Learning}, pages
  6995--7004. PMLR, 2019.

\bibitem[Zhou et~al.(2021)Zhou, Gu, and Szepesvari]{zhou2021nearly}
Dongruo Zhou, Quanquan Gu, and Csaba Szepesvari.
\newblock Nearly minimax optimal reinforcement learning for linear mixture
  markov decision processes.
\newblock In \emph{Conference on Learning Theory}, pages 4532--4576. PMLR,
  2021.

\end{thebibliography}

\clearpage

\appendix

\section{Preliminaries}
\label{app:prelims}
\subsection{Minimax Theorem}

For certain structural results, we require that the offset \CompText
(either the regret or PAC variant) is equal to its Bayesian
counterpart. This is a consequence of the minimax theorem whenever
mild topological conditions are satisfied; note that our objective can
always be made convex-concave by writing
\[
      \decoreg(\cM,\Mbar)
=\inf_{p\in\Delta(\Pi)}\sup_{\mu\in\Delta(\cM)}\En_{\pi\sim{}p,M\sim\mu}\brk*{\fm(\pim)-\fm(\pi)-\gamma\cdot\Dhels{M(\pi)}{\Mbar(\pi)}},
  \]
so all that is required to invoke the minimax
theorem is compactness. We state this as an
assumption to avoid committing to a particular set of technical conditions.
\begin{assumption}[Minimax swap]
  \label{ass:minimax}
  For the regret \CompShort, we have
  \begin{align}
    \label{eq:minimax_regret}
    \decoreg(\cM,\Mbar)
    = \decoregbayes(\cM,\Mbar)
\ldef\sup_{\mu\in\Delta(\cM)}\inf_{p\in\Delta(\Pi)}\En_{\pi\sim{}p,M\sim\mu}\brk*{\fm(\pim)-\fm(\pi)-\gamma\cdot\Dhels{M(\pi)}{\Mbar(\pi)}}.
  \end{align}
  For the PAC \CompShort, we have
    \begin{align}
    \label{eq:minimax_pac}
    \decopac(\cM,\Mbar)
    = \decopacbayes(\cM,\Mbar)
      \ldef
\sup_{\mu\in\Delta(\cM)}\inf_{p,q\in\Delta(\Pi)}\En_{M\sim\mu}\brk*{\En_{\pi\sim{}p}\brk*{\fm(\pim)-\fm(\pi)}-\gamma\cdot\En_{\pi\sim{}q}\brk*{\Dhels{M(\pi)}{\Mbar(\pi)}}}.
  \end{align}
\end{assumption}
As the simplest possible example, \pref{ass:minimax} is satisfied whenever $\cR$ is bounded and $\Pi$ is
finite (cf. Proposition 4.2 in
    \citet{foster2021statistical}), but assumption can be shown to hold under substantially more general conditions.

\section{Omitted Proofs from \creftitle{sec:lower}}
\label{app:lower}

\subsection{Proof of Regret Lower Bound (\preft{thm:regret_lower})}
In this section, we prove \Cref{thm:regret_lower}. The proof proceeds in two parts:
\begin{itemize}
\item In \Cref{sec:mbar-in-cm}, we state and prove
  \cref{lem:mbar-in-class}, a lower bound which is similar to
  \Cref{thm:regret_lower}, but restricts to proper reference models
  (specifically, the lower bound scales with $\sup_{\Mbar \in \cM} \deccreg[\vepslowerT](\MM, \Mbar)$).
  is $\Mbar$.
\item In \Cref{sec:all-mbar}, we prove the following algorithmic
  result (\pref{lem:add-mbar}): For any class $\MM$ and any
  $\Mbar \in \cMall$ (not necessarily in $\MM$), if there is an
  algorithm that achieves regret of at most $R$ with respect to the
  model class $\cM$, then there is an algorithm that achieves regret
  at most $O(R \cdot \log T)$ with respect to the model class $\cM
  \cup \{ \Mbar \}$. We then prove \pref{thm:regret_lower} by combining this result with
  \pref{lem:mbar-in-class}.%

\end{itemize}

We mention in passing that the two-part approach in this section can
also be applied to derive lower bounds for the PAC framework, but we adopt the alternative approach in
\pref{sec:pac_lower} because it leads to a result with fewer
logarithmic factors.

\subsubsection{Lower Bound for Proper Reference Models ($\protect\widebar{M}\in\cM$)}
\label{sec:mbar-in-cm}
\newcommand{\vepstil}{\wt{\veps}(T)}

In this section we prove \pref{lem:mbar-in-class}, a weaker lower bound analogous to the one stated in \Cref{thm:regret_lower}, but with the DEC replaced by a smaller quantity constrained to have $\Mbar \in \cM$. This weaker version is shown below. 
\begin{lemma}
  \label{lem:mbar-in-class}
  Let $\til \vep(T) \ldef c_1\cdot\frac{1}{\sqrt{T\Ct}}$, where $c_1>0$ is
  a sufficiently small numerical constant. For all $T\in\bbN$ such
  that the condition %
  \begin{align}
                       \label{eq:regret_lower_condition_mbarin}
    \sup_{\Mbar \in \MM} \deccreg[\til \vep(T)](\MM, \Mbar)\geq{}8\cdot\til \vep(T)
  \end{align}
  is satisfied, we have that for any regret minimization algorithm, there exists a model
  in $\cM$ such that, under this model,
  \begin{align}
    \label{eq:regret_lower_mbarin}
      \En\brk*{\RegDM} \geq{} \bigom(T)\cdot \sup_{\Mbar \in \MM} \deccreg[\til \vep(T)](\cM, \Mbar).
  \end{align}
\end{lemma}
We remark that the condition \Cref{eq:regret_lower_condition_mbarin} can be relaxed by replacing the constant 8 on the right-hand side with any constant strictly greater than 1. 
\begin{proof}[Proof of \Cref{lem:mbar-in-class}]
Let the algorithm under consideration be fixed, and let $\bbP\sups{M}$
denote the induced law of $\hist\ind{T}$ when $M$ is the underlying
model. Let $\En\sups{M}$ denote the corresponding expectation, and let
$\pm\ldef{}\Enm\brk*{\frac{1}{T}\sum_{t=1}^{T}p\ind{t}}$. Define $\vep \ldef \til \vep(T) = \frac{c_1}{\sqrt{TC(T)}}$, where the
constant $c_1 > 0$ will be specified below. Let $\Mbar \in \MM$
be chosen to maximize $\deccreg[\vep](\MM, \Mbar)$, and define $\delta := \deccreg[\vep](\MM, \Mbar)$.

\paragraph{Restricting to models performing poorly on $\protect\Mbar$}
If $\E_{\pi \sim\pmbar}[g\sups{\Mbar}(\pi)] \geq \delta / 10$, then,
by the definition of $\pmbar$, we have $\E\sups{\Mbar}[\RegDM] \geq T
\cdot \delta/10$, completing the proof of the lemma. Hence, we may
assume going forward that $\E_{\pi \sim \pmbar}[g\sups{\Mbar}(\pi)] < \delta/10$. 

\paragraph{Choosing an alternative model}
Define
    \begin{equation}
      M = \argmax_{M\in\cM}\crl*{
      \En_{\pi\sim{}\pmbar}\brk*{\fm(\pim) - \fm(\pi)} \mid{} \En_{\pi\sim{}\pmbar}\brk*{\Dhels{M(\pi)}{\Mbar(\pi)}}\leq\veps^2
    },\label{eq:mone}
  \end{equation}
  so that
  \begin{equation}
    \label{eq:mone_lb}
    \En_{\pi\sim{}\pmbar}\brk*{\gm(\pi)} \geq{} \deccreg(\cM,\Mbar) = \delta.
  \end{equation}

  Let $c_2\in(0,1)$ be fixed and define $\cE = \crl*{\pi : \gm(\pi) \geq{}c_2\cdot{}\delta}$. %
Recall that by Lemma A.13 of \citet{foster2021statistical}, we have
  \[
    \Dhels{\bbP\sups{M}}{\bbPmbar}
    \leq \Ct\cdot{}T\cdot\En_{\pi\sim\pmbar}\brk*{\Dhels{M(\pi)}{\Mbar(\pi)}} \leq C(T) \cdot T \cdot \vep^2,
  \]
  where we remind the reader that $\Ct = \bigoh(\log(T\wedge{}\abscont))$. %
  We choose the constant $c_1 > 0$ in the definition of $\vep=\vepstil$ to be sufficiently small so that 
  $
  \Dhels{\bbP\sups{M}}{\bbPmbar}\leq1/100$ and thus  $\Dtv{\bbP\sups{M}}{\bbPmbar}\leq\frac{1}{10}$.

  Observe that from the definition of $\cE$, we have
\begin{align}
  \En_{\pi\sim\pm}\brk*{\fm(\pim) - \fm(\pi)} \geq
  c_2\delta\cdot{}\pm(\cE)
  &\geq{} c_2\delta\cdot{}(\pmbar(\cE)-\Dtv{\bbPm}{\bbPmbar}) \notag\\
  &\geq{} c_2\delta\cdot{}(\pmbar(\cE)-1/10).
    \label{eq:tv_lower}
\end{align}
Therefore, it  suffices to lower bound $\pmbar(\ME)$ by $1/2$.

\paragraph{Lower bounding the gap}
We now compute
\begin{align}
  \fm(\pim) - \fmbar(\pimbar) & \geq \E_{\pi \sim \pmbar}[\gm(\pi) - \gmbar(\pi)] - \E_{\pi \sim \pmbar}[|\fm(\pi) - \fmbar(\pi)|]\nonumber\\
  & \geq \E_{\pi \sim \pmbar}[\gm(\pi)] - \frac{\delta}{10} - \vep\nonumber\\
  & \geq \frac{9}{10} \delta - \vep,\label{eq:gap_lb}
\end{align}
where the second inequality uses the assumption that
$\E_{\pmbar}[\gmbar] < \delta /10$ and \Cref{lem:hellinger_to_value},
and the final inequality uses \eqref{eq:mone_lb}.

\paragraph{Finishing up}
We conclude by noting that
\begin{align}
  \pmbar(\ME^\c) \cdot \left( \frac{9}{10} \delta - \vep \right) &\leq \E_{\pi \sim \pmbar} [ \One{\ME^\c} \cdot (\fm(\pim) - \fmbar(\pimbar))] \nonumber\\
  &\leq \E_{\pi \sim \pmbar} [ \One{\ME^\c} \cdot (\gm(\pi) - \gmbar(\pi))] + \vep\nonumber\\
  &\leq c_2 \delta + \vep\nonumber,
\end{align}
where the first inequality uses \eqref{eq:gap_lb}, the second
inequality uses \Cref{lem:hellinger_to_value}, and the
third inequality uses that $\gmbar(\pi) \geq 0$ for all $\pi \in \Pi$,
as well as the fact that $\gm(\pi) < c_2 \delta$ for $\pi \in
\ME^\c$. Rearranging, we conclude that $\pmbar(\ME_1^\c) \leq 1/2$ as
long as $c_2 \leq 1/8$ and $\vep \leq \delta/8$. Our choice of $\Mbar$,
together with the growth condition
\Cref{eq:regret_lower_condition_mbarin}, ensures that we indeed have $\vep \leq \delta/8$, thus establishing via \eqref{eq:tv_lower}, that $\E_{\pi \sim \pm}[\gm(\pi)] \geq \Omega(\delta)$ as desired.
\end{proof}

\subsubsection{Reducing from Improper ($\protect\Mbar\in\cMall$) to Proper $(\protect\Mbar \in \MM)$}
\label{sec:all-mbar}
\newcommand{\Enmp}[2]{\En^{\sss{#1},#2}}
\newcommand{\Pmp}[2]{\bbP^{\sss{#1},#2}}
\newcommand{\EnmpT}[2]{\En^{\sss{#1},#2}}
In this section, we work with several choices for the model class and
regret minimization algorithm. To avoid ambiguity, let us introduce some 
additional notation. Recall that an algorithm for
the $T$-timestep interactive decision making problem (in the regret
framework) is specified by a sequence $p=(p\^1, \ldots, p\^T)$, where
for each
$t \in [T]$, $p\^t$ is a probability kernel from $(\Omega\^{t-1},
\mathscr{F}\^{t-1})$ to $(\Pi, \mathscr{P})$.
Given an algorithm $p$, we let $\bbP^{\sss{M},p}\brk*{\cdot}$ denote
the law it induces on $\hist\ind{T}$ when $M\in\cMall$ is the
underlying model, and let $\Enmp{M}{p}\brk*{\cdot}$ denote the
corresponding expectation. With this notation, the algorithm's
expected regret when the underlying model is $M\in\cMall$ is $\Enmp{M}{p}[\RegDM]$.

The following lemma is the main technical result of this section. It shows that any model $\Mbar \in \cM$ with bounded optimal value can be added to a model class $\MM$ without substantially increasing the minimax regret.
\begin{lemma}
  \label{lem:add-mbar}
  Let the time $T \in \BN$ and model class $\MM$ be fixed. Let
  $\Mbar \in \cMall$ be any model such that for all $M\in\cM$, $\fmbar(\pimbar) \leq \fm(\pim) + \delta$
  for some $\delta>0$. %
  The minimax regret for the model class $\MM \cup \{ \Mbar \}$ is bounded above as follows:
  \begin{align}
\hspace{-0.3cm} \inf_{(p')\^1, \ldots, (p')\^T} \sup_{\Mstar \in \MM \cup \{ \Mbar \}} \Enmp{\Mstar}{p'}[\RegDM] \leq C \log T \cdot \inf_{p\^1, \ldots, p\^T} \sup_{\Mstar \in \MM} \Enmp{\Mstar}{p} [\RegDM] + C \cdot (\sqrt{T} + \delta T)\label{eq:pprime-reg-ub}.
  \end{align}
  where $C>0$ denotes a universal constant.
\end{lemma}

Before proving \Cref{lem:add-mbar}, we show how it implies \Cref{thm:regret_lower}.
\begin{proof}[Proof of \Cref{thm:regret_lower}]
  Fix $T \in \BN$, and write $\vep = \vepslowerT = \frac{c_1}{\sqrt{2
      T C(T)}}$, where the constant $c_1 > 0$ is chosen as in
  \cref{lem:mbar-in-class} (in particular, note that $\sqrt{2} \vep =
  \til \vep(T)$, where $\til \vep(T)$ is as defined in
  \cref{lem:mbar-in-class}). Let $\Mbar \in \cMall$ be chosen to
  maximize $\deccreg[\vep](\MM \cup \{ \Mbar \}, \Mbar)$, so that
  $\deccreg[\vep](\MM) = \deccreg[\vep](\MM \cup \{ \Mbar \},
  \Mbar)$. Define
\begin{align}
\til \MM := \{ M \in \MM \cup \{ \Mbar \} \ | \ \fm(\pi\subs{\Mbar}) \geq f\sups{\Mbar}(\pi\subs{\Mbar}) - \sqrt{2} \vep\}\nonumber,
\end{align}
so that $\Mbar \in \til \MM$. 
Then by \Cref{lem:localization_one_sided}, we have
\begin{align}
\deccreg(\MM) =  \deccreg(\MM\cup\crl{\Mbar}, \Mbar) \leq \deccreg[\sqrt{2} \vep](\til \MM, \Mbar) + \sqrt{2} \vep\label{eq:use-onesided-tilm},
\end{align}
Note that for all $M \in \til \MM$, we have that $\fmbar(\pimbar) \leq \fm(\pimbar) + \sqrt{2} \vep \leq \fm(\pim) + \sqrt{2} \vep$. 
Thus, by applying \Cref{lem:add-mbar} with $\delta = \sqrt{2} \vep$ to the class $\til \MM \backslash \{ \Mbar \}$, we see that
\begin{align}
  \inf_{(p')\^1, \ldots, (p')\^T} \sup_{\Mstar \in \til \MM} \Enmp{\Mstar}{p'}[\RegDM] \leq &  C \log T \cdot \inf_{p\^1, \ldots, p\^T} \sup_{\Mstar \in \til \MM \backslash \{ \Mbar \}} \Enmp{\Mstar}{p} [\RegDM] + C \cdot (\sqrt{T} + \sqrt{2} \vep T)\nonumber\\
  \leq &  C \log T \cdot \inf_{p\^1, \ldots,p\^T} \sup_{\Mstar \in \MM} \Enmp{\Mstar}{p} [\RegDM] + C \cdot (\sqrt{T} + \sqrt{2} \vep T)\label{eq:relate-algs},
\end{align}
where the second inequality follows since $\til \MM \backslash \{
\Mbar \} \subset \MM$.

To proceed, we will apply \cref{lem:mbar-in-class} to the class $\til \MM$. %
To verify that the condition \eqref{eq:regret_lower_condition_mbarin}
is satisfied for this class, we note that, as remarked above, $\til \vep(T) = \sqrt{2} \vep$, so that 
\begin{align}
\sup_{\Mbar_0 \in \til \MM} \deccreg[\til \vep(T)](\til \MM, \Mbar_0) = \sup_{\Mbar_0 \in \til \MM} \deccreg[\sqrt{2} \vep](\til \MM, \Mbar_0) \geq \deccreg[\sqrt{2} \vep](\til \MM, \Mbar) \geq \deccreg[\vep](\MM) - \sqrt{2} \vep \geq 8 \vep\nonumber,
\end{align}
where the second-to-last inequality uses \eqref{eq:use-onesided-tilm} and the final inequality uses the assumption from \eqref{eq:regret_lower_condition} that $\deccreg[\vep](\MM) \geq 10 \vep$. \cref{lem:mbar-in-class} gives that, for some universal constant $c_2 > 0$, 
\begin{align}
  \inf_{(p')\^1, \ldots, (p')\^T} \sup_{\Mstar \in \til \MM} \Enmp{\Mstar}{p'}[\RegDM] &\geq  c_2 \cdot T \cdot \sup_{\Mbar_0 \in \til \MM} \deccreg[\sqrt{2}\vep](\til\MM, \Mbar_0) \nonumber\\
  & \geq  c_2 \cdot T \cdot \deccreg[\sqrt{2} \vep](\til \MM, \Mbar) \geq c_2 \cdot T \cdot (\deccreg[\vep](\MM) - \sqrt{2} \vep) \label{eq:use-lemma-mbarinm},
\end{align}
where the final inequality uses \eqref{eq:use-onesided-tilm}. Combining \eqref{eq:use-lemma-mbarinm} and \eqref{eq:relate-algs} gives that
\begin{align}
\inf_{p\^1,\ldots, p\^T} \sup_{\Mstar \in \MM} \Enmp{\Mstar}{p}[\RegDM] \geq \frac{1}{C \log T} \cdot \left( c_2 \cdot T \cdot (\deccreg(\MM) - \sqrt{2} \vep) - C \cdot (\sqrt{T} + \sqrt{2} \vep T)\right)\nonumber,
\end{align}
which implies that for some constants $C',C'', c_2' > 0$, we have
\begin{align}
\inf_{p\^1, \ldots, p\^T} \sup_{\Mstar \in \MM} \Enmp{\Mstar}{p}[\RegDM] \geq  \frac{1}{\log T} \cdot \left( c_2' \cdot T \cdot (\deccreg[\vep](\MM)  - C' \cdot \vep) - C'' \cdot \sqrt{T} \right)\nonumber.
\end{align}
As long as $\frac{1}{2} \cdot \deccreg[\vep](\MM) > C' \cdot \vep$, it follows that
\begin{align}
\inf_{p\^1, \ldots, p\^T} \sup_{\Mstar \in \MM} \Enmp{\Mstar}{p}[\RegDM] \geq \frac{c_2'\cdot  T}{2 \log T} \cdot \deccreg(\MM) - \frac{C''\cdot  \sqrt{T}}{\log T}\nonumber, %
\end{align}
as desired.

\end{proof}

Finally, we prove \cref{lem:add-mbar}.

\begin{proof}[\pfref{lem:add-mbar}]
  Fix any algorithm $p = (p\^1, \ldots, p\^T)$. We define a modified
  algorithm $p' = ((p')\^1, \ldots, (p')\^T)$ in
  \cref{alg:add-mbar}. Roughly speaking, $p'$ runs $p$ multiple times,
  re-initializing $p$ whenever the average reward for the current run
  falls too far below $\fmbar(\pimbar)$. If the algorithm $p'$
  finds that it has re-initialized $p$ more than $\log(T)$ times, it will
  switch to playing $\pimbar$ for all remaining rounds. The crux of
  the proof will be to show that the worst-case regret of $p'$ for
  models in $\cM\cup\crl{\Mbar}$ is not much larger than the
  worst-case regret of $p$ for models in $\cM$.
       \begin{algorithm}[ht]
    \setstretch{1.3}
     \begin{algorithmic}[1]
       \State \textbf{parameters}:
       \Statex[1] Number of rounds $T\in\bbN$.
       \Statex[1] Algorithm $p = (p\^1, \ldots, p\^T)$.
       \State Initialize $I = 1$, $T_1 = 1$, and $R = 4 \cdot \sup_{M^\st \in \MM} \EnmpT{\Mstar}{p}[\RegDM] + \delta T + 8\sqrt{T}$.
       \For{$1 \leq t \leq T$}
       \State  \label{line:define-pprime} Define $(p')\^t(\cdot)$ to
       be the distribution $p\^{t-T_I+1}(\cdot \mid \{ (\pi\^s, r\^s,
       o\^s)\}_{s=T_I}^{t-1})$.
       \State Draw $\pi\^t \sim (p')\^t$, and observe $(\pi\^t, r\^t, o\^t)$. %
       \If{$ \sum_{s=T_I}^{t} ( \fmbar(\pimbar) - r\^s) \geq R$}\label{line:add-mbar-test}
       \State Set $T_{I+1} := t+1$ and then  increment $I$. \hfill\algcommentlight{This has the effect of re-initializing $p$.}
       \EndIf
       \If{$I > \lceil \log T\rceil $}
       \State \textbf{break} out of loop.
       \EndIf
       \EndFor
       \State For remaining time steps $t$ (if any): play $\pi\^t := \pimbar$ (i.e., set $(p')\^t = \indic_{\pimbar}$). \label{line:play-pimbar}
     \end{algorithmic}
     \caption{Algorithm $p'$ used in proof of \cref{lem:add-mbar}}
     \label{alg:add-mbar}
     \end{algorithm}

As per our convention, in the context of \cref{alg:add-mbar}, we let $\mathscr{F}\^t$ denote the sigma-algebra generated by $\{ (\pi\^s, r\^s, o\^s) \}_{1 \leq s \leq t}$. 
We bound the regret of the algorithm $p'$ by considering the following cases for $M^\st \in \MM\cup\crl{\Mbar}$.

\paragraph{Case 1: $M^\st = \protect\Mbar$} Let $T_0 \in [T+1]$ be
defined to be the smallest value of $t$ for which the decision
$\pi\ind{t}$ is chosen using the rule at  \cref{line:play-pimbar}, or
$T+1$ if there is no such step. By construction, we have that
\begin{align}
\sum_{t=1}^{T_0-1} (\fmbar(\pimbar) - r\^t) = \sum_{t=1}^T \One{t < T_0} \cdot (\fmbar(\pimbar) - r\^t)  \leq (R+1) \cdot \lceil \log T\rceil.\label{eq:rlogt-alg}
\end{align}
Note that, for each $t \in [T]$, the variable $\One{t < T_0}$ is
measurable with respect to $\mathscr{F}\^{t-1}$. As a result, we have
\begin{align}
  \EnmpT{\Mbar}{p'}[\RegDM] &= \EnmpT{\Mbar}{p'} \left[\sum_{t=1}^T \left( \fmbar(\pimbar) - \E_{\pi\^t \sim (p')\^t}[\fmbar(\pi\^t)]\right)\right]\nonumber\\
  &= \EnmpT{\Mbar}{p'} \left[ \sum_{t=1}^T \One{t < T_0} \cdot \prn*{\fmbar(\pimbar) - \E_{\pi\^t \sim (p')\^t} [\fmbar(\pi\^t)]}\right]\nonumber\\
  &= \EnmpT{\Mbar}{p'} \left[ \sum_{t=1}^T \One{t < T_0} \cdot (\fmbar(\pimbar) - r\^t) \right] + \EnmpT{\Mbar}{p'} \left[ \sum_{t=1}^T \One{t < T_0} \cdot \prn*{r\^t - \E_{\pi\^t \sim (p')\^t}[\fmbar(\pi\^t)]} \right]\nonumber\\
  &\leq (R+1)\cdot  \log T\nonumber,
\end{align}
where the final inequality uses \eqref{eq:rlogt-alg} and the fact that for each $t$, we have 
\begin{align}
  \E \left[ \One{t < T_0} \cdot \prn*{r\^t - \E_{\pi\^t \sim (p')\^t}[\fmbar(\pi\^t)]} \ | \ \mathscr{F}\^{t-1}\right] &= \One{t < T_0} \cdot \E[r\^t - \E_{\pi\^t \sim (p')\^t}\brk*{\fmbar(\pi\^t)] \ | \ \mathscr{F}\^{t-1}} =0\nonumber.
\end{align}
Thus, in the case $\Mstar = \Mbar$, we have verified that the claimed upper bound in \eqref{eq:pprime-reg-ub} on the regret of $p'$ holds.

\paragraph{Case 2: $M^\st \in \MM$} We first state and prove two
technical lemmas.
\begin{lemma}
  \label{lem:stopping}
  For the algorithm $p$, any model $\Mstar \in \cMall$, and random
  variable $\tau$ {(potentially dependent on
  $\hist\ind{T}$)} taking values in $[T]$, it holds that
  \begin{align}
    \EnmpT{\Mstar}{p} \left[ \sum_{t=1}^\tau \E_{\pi\^t \sim p\^t} [\fmstar(\pimstar) - \fmstar(\pi\^t) ]\right] \leq \EnmpT{\Mstar}{p}[\RegDM]\nonumber.
  \end{align}
\end{lemma}
\begin{proof}[\pfref{lem:stopping}]
The result follows by noting that
  \begin{align}
&     \EnmpT{\Mstar}{p}[\RegDM] - \EnmpT{\Mstar}{p} \left[ \sum_{t=1}^\tau \E_{\pi\^t \sim p\^t} [\fmstar(\pimstar) - \fmstar(\pi\^t) ]\right]\nonumber\\
    &= \EnmpT{\Mstar}{p} \left[ \sum_{t=1}^T \One{\tau < t} \cdot \E_{\pi\^t \sim p\^t} [\fmstar(\pimstar) - \fmstar(\pi\^t)]\right] \geq 0\nonumber,
  \end{align}
  where we have used that the random variable $ \E_{\pi\^t \sim p\^t} [\fmstar(\pimstar) - \fmstar(\pi\^t)]$ is non-negative a.s.
\end{proof}

{The next lemma concerns the probability that a \emph{single run} of the
  algorithm $p$ violates the condition in \pref{line:add-mbar-test} of \pref{alg:add-mbar}.
  }
\begin{lemma}
  \label{lem:succ-prob12}
  For any algorithm $p = (p\^1, \ldots, p\^T)$ and model $\Mstar \in  \MM$, it holds that
  \begin{align}
\Pmp{\Mstar}{p} \left( \exists t \leq T \ : \ \sum_{s=1}^t ( \fmbar(\pimbar) - r\^s ) > R \right) \leq \frac 12, \nonumber
  \end{align}
  where $R = 4 \cdot \sup_{M^\st \in \MM} \E\sups{M^\st, p}[\RegDM] + \delta T + 8\sqrt{T}$.
\end{lemma}
\begin{proof}[\pfref{lem:succ-prob12}]
Let $\mathscr{F}\^t$ be the sigma-algebra generated by $\{ (\pi\^s, r\^s, o\^s) \}_{s=1}^t$. 
Fix $\Mstar\in\cM$ and define $R_0 := \EnmpT{\Mstar}{p}[\RegDM]$. By Markov's inequality and the fact that the random variables $\E_{\pi\^t \sim p\^t} [\fmstar(\pimstar) - \fmstar(\pi\^t)]$ are all non-negative, it holds that
  \begin{align}
    & \Pmp{\Mstar}{p} \left( \sup_{t \leq T} \sum_{s=1}^{t} \E_{\pi\^s \sim p\^s} [\fmstar(\pimstar) - \fmstar(\pi\^s)] > 4R_0 \right)\nonumber\\
&= \Pmp{\Mstar}{p} \left( \sum_{t=1}^T \E_{\pi\^t \sim p\^t} [\fmstar(\pimstar) - \fmstar(\pi\^t)] > 4R_0 \right) \leq  \frac 14 \label{eq:r0-bound-1}.
  \end{align}
  Now, define $X_0 = 0$ and $X_t = \sum_{s=1}^t \left( \E_{\pi\^s \sim
      p\^s}[\fmstar(\pi\^s)] - r\^s\right)$ for $t\in\brk{T}$. Note
  that $(X_t)_{t \geq 0}$ is a martingale with respect to the
  filtration $\mathscr{F}\^t$. Therefore, by Theorem 4.5.1 of \citet{durrett2019probability}, it holds that
  \begin{align}
    \EnmpT{\Mstar}{p} \left[ \sup_{t \leq T} |X_t|^2 \right] &\leq 4 \cdot \EnmpT{\Mstar}{p} \left[ \sum_{t=1}^T \E \left[ (\E_{\pi\^t \sim p\^t}[\fmstar(\pi\^t)] - r\^t)^2 \ | \ \scrF\^{t-1}\right] \right]\leq 4T\nonumber,
  \end{align}
  where the final inequality uses that $|\E_{\pi\^t \sim p\^t}[\fmstar(\pi\^t)] - r\^t| \leq 1$ for all $t$. By Jensen's inequality and Markov's inequality, it follows that for any $\lambda > 0$,
  \begin{align}
\Pmp{\Mstar}{p} \left( \sup_{t \leq T} |X_t| > 2\lambda \sqrt{T} \right) \leq \frac{1}{\lambda}\nonumber,
  \end{align}
  and by choosing $\lambda = 4$, we see that
  \begin{align}
\Pmp{\Mstar}{p} \left( \sup_{t \leq T} \sum_{s=1}^t \left(\E_{\pi\^s \sim p\^s}[\fmstar(\pi\^s)] - r\^s\right) > 8 \sqrt{T} \right) \leq \frac{1}{4}\label{eq:root-t-bound}.
  \end{align}
  
Combining \eqref{eq:r0-bound-1} and \eqref{eq:root-t-bound}, we have
  \begin{align}
\Pmp{\Mstar}{p} \left( \exists t \leq T \ : \ \sum_{s=1}^t \E_{\pi\^s \sim p\^s} [\fmstar(\pimstar) - r\^s] > 4R_0 + 8\sqrt{T}\right) \leq \frac 12\nonumber.
  \end{align}
  Since $\Mstar \in \MM$, and so $\fmstar(\pimstar)  \geq \fmbar(\pimbar) - \delta $, it follows that
    \begin{align}
\Pmp{\Mstar}{p} \left( \exists t \leq T \ : \ \sum_{s=1}^t \E_{\pi\^s \sim p\^s} [\fmbar(\pimbar) - r\^s] > 4R_0 + 8 \sqrt{T} + \delta T \right) \leq \frac 12\nonumber,
    \end{align}
    as desired.
  \end{proof}
We now continue with the proof of \cref{lem:add-mbar}.  Write $L =
\lceil \log T \rceil$, and denote the final value of $I$ in \cref{alg:add-mbar} by
$I' \leq L + 1$. If $I' \leq L$, then set $T_{I'+1} = \cdots = T_{L+1}
= T+1$. Note that for each $\ell \in [L+1]$, $T_\ell-1$ is a stopping
time (since the event $\{T_\ell-1=t\}$ is measurable with respect to
$\mathscr{F}\^{t}$ for each $t \in [T]$), and thus $T_\ell$ is a
stopping time as well.

  \Cref{lem:succ-prob12} together with the definition of $(p')\^t$ in \cref{line:define-pprime} establishes that for each $\ell \in [L]$,
  \begin{align}
\Pmp{\Mstar}{p'} \left( T_{\ell+1} \leq T \ | \ T_\ell \leq T \right) &= \Pmp{\Mstar}{p'} \left( \exists t \mbox{ s.t. } \ T-1 \geq  t \geq T_\ell \mbox{ and } \sum_{s=T_\ell}^t (\fmbar(\pimbar) - r\^s) > R \ | \ T_\ell \leq T \right) \leq \frac 12 \nonumber.
  \end{align}
  Therefore, since the event $\{ T_\ell \leq T \}$ is equal to the event $\{ T_{\ell'} \leq T \ \forall \ell' \leq \ell\}$,
  \begin{align}
    \Pmp{\Mstar}{p'} \left( T_{L+1} \leq T\right)&= \Pmp{\Mstar}{p'} \left( \forall \ell \leq L+1, \ T_\ell \leq T \right)\nonumber\\
    &= \prod_{\ell=1}^L \Pmp{\Mstar}{p'} \left( T_{\ell+1} \leq T \ | \ T_{\ell'} \leq T \ \forall \ell' \leq \ell \right)\nonumber\\
    &= \prod_{\ell=1}^L \Pmp{\Mstar}{p'} \left( T_{\ell+1} \leq T \ | \ T_\ell \leq T \right)\nonumber\\
    &\leq (1/2)^L \leq 1/T\nonumber.
  \end{align}
  Furthermore, for each $\ell \in [L]$, we have
  \begin{align}
    \Enmp{\Mstar}{p'} \left[   \sum_{t=T_\ell}^{T_{\ell+1}-1} \E_{\pi\^t \sim (p')\^t}[\fmstar(\pimstar) - \fmstar(\pi\^t)]\right] &= \Enmp{\Mstar}{p'} \left[ \sum_{t=1}^T \One{T_\ell \leq t < T_{\ell+1}} \cdot \E_{\pi\^t \sim (p')\^t} [\fmstar(\pimstar) - \fmstar(\pi\^t)] \right]\nonumber\\
    &\leq \EnmpT{\Mstar}{p}[\RegDM]   \nonumber,
  \end{align}
  where the inequality uses \Cref{lem:stopping} and the definition of $p'$ in \cref{line:define-pprime} for steps $T_\ell \leq t < T_{\ell+1}$.  %
  It follows that
  \begin{align}
    & \Enmp{\Mstar}{p'} \left[ \sum_{t=1}^T \E_{\pi\^t \sim (p')\^t} [\fmstar(\pimstar) - \fmstar(\pi\^t)] \right]\nonumber\\
    &\leq T \cdot \Pmp{\Mstar}{p'} (T_{L+1} \leq T) + \sum_{\ell=1}^L  \Enmp{\Mstar}{p'} \left[   \sum_{t=T_\ell}^{T_{\ell+1}-1} \E_{\pi\^t \sim (p')\^t}[\fmstar(\pimstar) - \fmstar(\pi\^t)]\right] \nonumber\\
    &\leq1 + L \cdot  \EnmpT{\Mstar}{p}[\RegDM]\nonumber,
  \end{align}
which verifies the claimed upper bound on regret in \eqref{eq:pprime-reg-ub}.
\end{proof}

\section{Omitted Proofs from \creftitle{sec:upper}}
\label{app:upper}
In this section we prove \Cref{thm:regret_upper}, which shows that \Cref{alg:regret} attains a regret bound based on the constrained \CompShort

\subsection{Proof of Regret Upper Bound (\preft{thm:regret_upper})}
Toward proving \pref{thm:regret_upper}, we introduce a few \emph{success events} that will be used throughout the analysis:
\begin{enumerate}
\item For each $i \in [N]$, $\SA_{i}$ denotes the event that all of the following inequalities hold: %
  \begin{align}
 \forall s \in \MS_i, \qquad    &\sum_{j=1}^{J_{i,s}} \E_{\pi \sim p\^s} \left[ \hell{M^\st(\pi)}{\til M_s\^j(\pi)}\right] \leq  {\Est(J_i, \delta)},\label{eq:mtil-mstar-close-ai}\\
    &\E_{\pi \sim p\^{s_i}} \left[ \hell{\wh M_i(\pi)}{M^\st(\pi)} \right] \leq   \frac{\Est(J_i, \delta)}{J_i},\label{eq:mhati-mstar-close}\\
    &\E_{\pi \sim p\^{s_i}} \left[ \hell{\wh M\^{s_i}(\pi)}{M^\st(\pi)} \right] \leq \vep_i^2,\label{eq:mhatsi-mstar-close}
   \end{align}
   where we recall that $s_i$ is the index defined on \cref{line:set-si-perm}.

\item For each $i \in [N]$, $\SB_i$ denotes the event that
  \begin{align}
\sum_{t\in \ME_i} \E_{\pi \sim p\^t} \left[ \hell{M^\st(\pi)}{\wh M\^t(\pi)}\right] \leq {\Est(|\ME_i|, \delta)}\label{eq:def-bi-event}.
  \end{align}
\item For each $i \in [N]$, $\SC_i$ denotes the event that $M^\st \in \MM_i$.
\item For each $i \in [N]$, $\SD_i$ denotes the event that there is some $s \in \MS_i$ so that %
  \begin{align}
\E_{\pi \sim p\^s} \left[ \hell{M^\st(\pi)}{\wh M\^s(\pi)} \right] \leq \frac{\vep_i^2}{16}\nonumber.
  \end{align}
\end{enumerate}
In addition, we define
\[\SA = \bigcap_{i \in [N]} \SA_i,\ \SB = \bigcap_{i \in [N]} \SB_i, \ \SC = \bigcap_{i \in [N]} \SC_i, \ \text{and} \ \SD = \bigcap_{i \in [N]} \SD_i. \]

We also recall the following notation, which will be used throughout the proof.
\begin{itemize}
\item For $i \in [N]$ we set \begin{equation}
    \alpha_i \ldef C_0 \cdot \deccreg[\vep_i](\MM) + 64 \vep_i,\label{eq:cnot}
  \end{equation} with the convention that $\alpha_0 = 1$.
The constant $C_0>0$ in \eqref{eq:cnot}, as well as the constant $C_1>0$ specified in \cref{alg:regret}, will need to be taken sufficiently large; in what follows, we show that $C_0 \geq 20$ and $C_1 \geq 128$ will suffice.
\end{itemize}

\subsubsection{Technical Lemmas}

Before proving \pref{thm:regret_upper}, we state and prove several technical lemmas concerning the performance of \pref{alg:regret}. The following lemma shows that the event $\SA \cap \SB \cap \SC$ occurs with high probability.
\begin{lemma}
  \label{lem:good-event-lb}
Suppose that $C_1 \geq 128$. The event $\SA \cap \SB \cap \SC \cap \SD$ occurs with probability at least $1 - 3L \delta N$. 
\end{lemma}
\begin{proof}[\pfref{lem:good-event-lb}]
  We show that  $\BP \left(\bigcap_{i' \leq i} \SA_{i'} \cap \SB_{i'} \cap \SC_{i'}\cap \SD_{i'} \right) \geq 1 - 3L\delta i$ for each $i \in [N]$ using induction on $i$. Fix $i \in [N]$ and let us condition on $\bigcap_{i' < i} (\SA_{i'} \cap \SB_{i'} \cap \SC_{i'} \cap \SD_{i'})$.

  \paragraph{Establishing that $\SC_i$ holds} The fact that $\SA_{i-1}$ holds implies that $\E_{\pi \sim \wh p_{i-1}} \left[ \hell{M^\st(\pi)}{\wh M_{i-1}(\pi)} \right] \leq \frac{\Est(J_{i-1}, \delta)}{J_{i-1}}$, which implies (by definition) that $M^\st \in \MM_i$, i.e., $\SC_i$ holds.

  \paragraph{Establishing  that $\SB_i$ holds} 
  Conditioned on $\SC_i$, it follows from \cref{ass:hellinger_oracle_constrained} that $\SB_i$ holds with probability at least $1-\delta$.

  \paragraph{Establishing that $\SD_i$ holds}
   Note that as a consequence of our parameter settings,
  \begin{align}
    \vep_i^2 \geq 2^{-i} \cdot T\cdot \vep_N^2 = 2^{-i} \cdot C_1 \cdot \Est(T, \delta) \cdot L \geq \max \left\{ 32 \cdot \frac{\Est(J_i, \delta)}{J_i} , %
    32 \cdot \frac{\Est(|\ME_i|,\delta)}{|\ME_i|}\right\}\label{eq:bound-ei-ji},
  \end{align}
  where we have used that $C_1 \geq 128$ and that $\Est(T, \delta) \geq \max\{\Est(|\ME_i|, \delta), \Est(J_i, \delta) \}$. Then since $\SB_i$ (i.e., (\ref{eq:def-bi-event})) holds, at least $|\ME_i|/2$ rounds $t \in \ME_i$ satisfy $M^\st \in \MH_{\vep_i/4, p\^t}(\wh M\^t)$. Since $|\MS_i| = L \geq \log 1/\delta$, it follows that with probability at least $1-\delta$, there is some $s \in \MS_i$, which we denote by $s_i^\st$, for which $M^\st \in \MH_{\vep_i/4, p\^{s_i^\st}}(\wh M\^{s_i^\st})$. In particular, conditioned on $\SB_i$, the event $\SD_i$ holds with probability at least $1-\delta$.

\paragraph{Establishing that $\SA_i$ holds}  Next, from \cref{ass:hellinger_oracle_constrained} and the fact that Hellinger distance is always non-negative, we have that with probability at least $1-L\delta$, for all $s \in \MS_i$,
  \begin{align}
\sum_{j=1}^{J_{i,s}} \E_{\pi \sim p\^s} \left[ \hell{M^\st(\pi)}{\til M_s\^j(\pi)}\right] \leq {\Est(J_i, \delta)}\label{eq:jsi-stopping},
  \end{align}
  which verifies that \cref{eq:mtil-mstar-close-ai} holds.
Next, let us condition on the event that $\SD_i$ holds. Then applying \cref{eq:jsi-stopping} to $s=s_i^\st$ and using the definition of $s_i^\st$ (recall that $s^{\star}_i$ is defined in the prequel so that $M^\st \in \MH_{\vep_i/4, p\^{s_i^\st}}(\wh M\^{s_i^\st})$), we see that
  \begin{align}
    \sum_{j=1}^{J_{i,s_i^\st}} \E_{\pi \sim p\^{s_i^\st}} \left[ \hell{\wh M\^{s_i\^\st}(\pi)}{\til M_{s_i^\st}\^j(\pi)}\right] \leq &  \sum_{j=1}^{J_{i,s_i^\st}}  2\E_{\pi \sim p\^{s_i^\st}} \left[ \hell{\wh M\^{s_i\^\st}(\pi)}{ M^\st(\pi)}\right] + 2 \E_{\pi \sim p\^{s_i^\st}} \left[ \hell{\wh M^\st(\pi)}{\til M_{s_i^\st}\^j(\pi)}\right] \nonumber\\
    \leq & 2 J_{i, s_i^\st} \cdot \frac{\vep_i^2}{16} + 2 \sum_{j=1}^{J_{i,s_i^\st}} \E_{\pi \sim p\^{s_i^\st}} \left[ \hell{M^\st(\pi)}{\til M_{s_i\^\st}\^j(\pi)} \right] \nonumber\\
    \leq &  \frac{J_i \vep_i^2}{8} + 2 \cdot \Est(J_i, \delta) \leq \frac{3J_i\vep_i^2}{16}\nonumber,
  \end{align}
  where the final inequality uses \eqref{eq:bound-ei-ji}. By the definition of $J_{i,s_i^\st}$ on \cref{line:choose-jis,line:assign-si}, it must be the case that $J_{i,s_i^\st} = J_i$, and thus $s_i^{\mathrm{tmp}}$ is assigned at least once on \cref{line:assign-si}. Therefore, the value of $s_i$ set on \cref{line:set-si-perm} satisfies
  \begin{align}
\E_{\pi \sim p\^{s_i}} \left[ \hell{\wh M_i(\pi)}{M^\st(\pi)} \right] \leq \frac{1}{J_i} \sum_{j=1}^{J_i} \E_{\pi \sim p\^{s_i}} \left[ \hell{\til M_{s_i}\^j(\pi)}{M^\st(\pi)} \right] \leq \frac{\Est(J_i, \delta)}{J_i}\nonumber,
  \end{align}
  where the first inequality uses convexity of the squared Hellinger distance and the second inequality uses \cref{eq:jsi-stopping} together with the fact that $J_{i,s_i} = J_i$. The above display verifies \cref{eq:mhati-mstar-close}; to verify \cref{eq:mhatsi-mstar-close}, we note that, since $s_i^{\mathrm{tmp}}$ is assigned at least once,
  \begin{align}
\sum_{j=1}^{J_i} \E_{\pi \sim p\^{s_i}} \left[ \hell{\wh M\^{s_i}(\pi)}{\til M_{s_i}\^j(\pi)} \right] \leq \frac{J_i\vep_i^2}{4}\label{eq:use-si-defn}.
  \end{align}
  Thus, we may compute
  \begin{align}
    \E_{\pi \sim p\^{s_i}} \left[ \hell{\wh M\^{s_i}(\pi)}{M^\st(\pi)} \right] \leq & 2 \cdot\E_{\pi \sim p\^{s_i}} \left[ \hell{\wh M_i(\pi)}{\wh M\^{s_i}(\pi)} \right] + 2\cdot\E_{\pi \sim p\^{s_i}} \left[ \hell{\wh M_{i}(\pi)}{M^\st(\pi)} \right]\nonumber\\
    \leq & \frac{2}{J_i} \sum_{j=1}^{J_i} \E_{\pi \sim p\^{s_i}} \left[ \hell{\wh M\^{s_i}(\pi)}{\til M_{s_i}\^j(\pi)}\right] + \frac{2}{J_i}\sum_{j=1}^{J_i} \E_{\pi \sim p\^{s_i}} \left[ \hell{M^\st(\pi)}{\til M_{s_i}\^j(\pi)} \right] \nonumber\\
    \leq & \frac{\vep_i^2}{2} + \frac{2 \cdot \Est(J_i, \delta)}{J_i}\leq  \vep_i^2 \nonumber,
  \end{align}
  where the second inequality uses the convexity of squared Hellinger distance, the third inequality uses \cref{eq:jsi-stopping} for $s = s_i$ and \cref{eq:use-si-defn}, and the final inequality uses \cref{eq:bound-ei-ji}. As the above display verifies \cref{eq:mhatsi-mstar-close}, we conclude that conditioned on $\SD_i$ holding, $\SA_i$ holds with probability at least $1-L\delta$.

  \paragraph{Wrapping up} Summarizing, conditioned on $ \bigcap_{i' < i} \SA_{i'} \cap \SB_{i'} \cap \SC_{i'}\cap \SD_{i'} $, we have shown that $\SA_i \cap \SB_i \cap \SC_i \cap \SD_i$ holds with probability $1-2\delta - L\delta \geq 1-3L\delta$. 
  Thus, the inductive hypothesis that $\BP \left( \bigcap_{i' < i} \SA_{i'} \cap \SB_{i'} \cap \SC_{i'} \right) \geq 1 - 3L\delta (i-1)$ implies that  $\BP\left( \bigcap_{i' \leq i} \SA_{i'} \cap \SB_{i'} \cap \SC_{i'} \cap \SD_{i'} \right)\geq  (1-3L\delta(i-1)) \cdot (1-3L\delta) \geq 1 - 3L\delta i$.

  Summarizing, we get that $\BP(\SA \cap \SB \cap \SC \cap \SD) \geq 1 - 3L \delta N$. 
\end{proof}

\Cref{lem:refined-acc} shows that the distributions $\wh p_i$ computed in \Cref{alg:regret} enjoy low suboptimality with respect to $\Mstar$. %
\begin{lemma}[Accuracy of refined policies]
  \label{lem:refined-acc}
 Suppose that $C_0 \geq 4$.  Then for each epoch $i \in [N]$, under the event $\SA_i$, the distribution $\wh p_i$ satisfies
  \begin{align}
\E_{\pi \sim \wh p_i} \left[ \fmstar(\pimstar) - \fmstar(\pi) \right] \leq \frac{\alpha_i}{4}\label{eq:refined-acc}.
  \end{align}
\end{lemma}
\begin{proof}[\pfref{lem:refined-acc}]
Conditioning on the event $\SA_i$ gives that \cref{eq:mhatsi-mstar-close} holds, which can in particular be written as $M^\st \in \MH_{\vep_i, p\^{s_i}}(\wh M\^{s_i})$.  
 Therefore, by the choice of $\wh p_i = p\^{s_i}$ in \cref{line:set-si-perm} and the definition in
 \Cref{line:regret_algo_dec}, under $\SA_i$,
  \begin{align}
    \E_{\pi \sim p\^{s_i}} \left[ \fmstar(\pimstar) - \fmstar(\pi) \right] & \leq \sup_{M \in \MH_{\vep_i, p\^{s_i}}(\wh M\^{s_i}) \cup \{ \wh M\^{s_i} \}} \E_{\pi \sim p\^{s_i}} \left[ \fm(\pim) - \fm(\pi) \right] \nonumber\\
    =& \deccreg[\vep_i](\MM \cup \{ \wh M\^{s_i} \}, \wh M\^{s_i}) \leq \sup_{\Mbar \in \co(\MM)} \deccreg[\vep_i](\MM \cup \{ \Mbar \}, \Mbar) \leq \frac{\alpha_i}{4}\nonumber,
  \end{align}
  where the final inequality follows as long as $C_0 \geq 4$.
  \end{proof}

 \Cref{lem:compare-arb-distr} relates the suboptimality under $M^\st$ for any distribution $p \in \Delta(\Pi)$ to that of any model $M \in \co(\MM_{i+1})$, in terms of the distance between $M$ and $M^\st$. We ultimately apply the lemma with $M=\wh M\^t$ for each $t \in \ME_{i+1}$ to derive the following lemma, \Cref{lem:epoch-optimal-pol}.
  \begin{lemma}[Comparison with models in refined class]
    \label{lem:compare-arb-distr}
Fix $i \in [N]$. Then  for all $M \in \co(\MM_{i+1})$ and all $p \in \Delta(\Pi)$, under the event $\SA_i$,
    \begin{align}
\E_{\pi \sim p} \left[ \fmstar(\pimstar) - \fmstar(\pi) \right] \leq \E_{\pi \sim p} \left[ \fm(\pim) - \fm(\pi) \right] + \frac{\alpha_i}{2} + \sqrt{\E_{\pi \sim p} \left[ \hell{M(\pi)}{M^\st(\pi)}\right]}.\label{eq:mstar-com-bound}
    \end{align}
  \end{lemma}
  \begin{proof}[\pfref{lem:compare-arb-distr}]
    We first upper bound the optimal value under $M^\st$ by the optimal value under any $M \in \co(\MM_{i+1})$. To do so, first note that, by \cref{lem:refined-acc} (in particular, the fact that \eqref{eq:refined-acc} holds at epoch $i$), we have that $\E_{\pi \sim \wh p_i} \left[ \fmstar(\pimstar) - \fmstar(\pi) \right] \leq \frac{\alpha_i}{4}$ under $\SA_i$.  Then, for any $M \in \MM_{i+1}$, we have that, under the event $\SA_i$,
    \begin{align}
                                           \E_{\pi \sim \wh p_i} \left[ \fmstar(\pimstar) - \fm(\pi) \right]
      & \leq \E_{\pi \sim \wh p_i} \left[ \fmstar(\pimstar) - \fmstar(\pi) \right] + \sqrt{ \E_{\pi \sim \wh p_i} \left[ \hell{M(\pi)}{M^\st(\pi)}\right]} \nonumber\\
      & \leq \frac{\alpha_i}{4} + \sqrt{ 2 \E_{\pi \sim \wh p_i} \left[ \hell{M(\pi)}{\wh M_i(\pi)}\right] + 2 \E_{\pi \sim \wh p_i} \left[ \hell{\wh M_i(\pi)}{M^\st(\pi)}\right]}\nonumber\\
      & \leq \frac{\alpha_i}{4} + 2 \sqrt{\frac{\Est(J_i,\delta)}{J_i}} \leq \frac{\alpha_i}{2}\label{eq:bound-mstar-ai2},
    \end{align}
    where the second-to-last inequality holds since $M \in \MM_{i+1}$ and by assumption of the event $\SA_{i}$ (in particular, using \cref{eq:mhati-mstar-close}), and the final inequality holds since $2 \sqrt{\frac{\Est(J_i, \delta)}{J_i}} \leq \frac{\alpha_i}{4}$ by our choice of $\alpha_i \geq 64 \vep_i$ and \eqref{eq:bound-ei-ji}.   
    
    Now fix any $M \in \co(\MM_{i+1})$, and note that we can write $M = \E_{M' \sim \nu_M}[M']$ for some $\nu_M \in \Delta(\MM_{i+1})$. Then for all $\pi \in \Pi$, $\fm(\pi) = \E_{M' \sim \nu_M}[\fmp(\pi)]$, and it follows from \eqref{eq:bound-mstar-ai2} that (again under $\SA_i$)
    \begin{align}
      \fmstar(\pimstar) - \fm(\pim) \leq \E_{\pi \sim \wh p_i} \left[ \fmstar(\pimstar) - \fm(\pi) \right] = \E_{M' \sim \nu_M} \E_{\pi \sim \wh p_i} \left[ \fmstar(\pimstar) - \fmp(\pi) \right] \leq \frac{\alpha_i}{2}\nonumber.
    \end{align}

    Given any $M \in \co(\MM_{i+1})$, we have now that under $\SA_i$,
    \begin{align}
      \E_{\pi \sim p} \left[ \fmstar(\pimstar) - \fmstar(\pi) \right] & \leq \frac{\alpha_i}{2} + \E_{\pi \sim p} \left[ \fm(\pim) - \fm(\pi)\right] + \E_{\pi \sim p} \left[ | \fm(\pi) - \fmstar(\pi)| \right]\nonumber\\
      & \leq\frac{\alpha_i}{2}+ \E_{\pi \sim p} \left[ \fm(\pim) - \fm(\pi) \right] + \sqrt{\E_{\pi \sim p} \left[ \hell{M(\pi)}{M^\st(\pi)}\right]}\nonumber,
    \end{align}
    as desired.
  \end{proof}

  Our final technical lemma, \Cref{lem:epoch-optimal-pol}, bounds the sub-optimality for all policies played in each epoch $\ME_i$. The need to establish a result of this type is a crucial difference between the regret and PAC frameworks, and motivates many of the algorithm design choices behind \pref{alg:regret}.
\begin{lemma}[``Backup'' regret guarantee]
  \label{lem:epoch-optimal-pol}
  Fix any $i \in [N]$. 
  Then for all $t \in \ME_i$, we have that under the event $\SA_{i-1}$, 
  \begin{align}
\E_{\pi \sim p\^t} \left[ \fmstar(\pimstar) - \fmstar(\pi) \right] & \leq \deccreg[\vep_i](\MM \cup \{ \wh M\^t \} ,\wh M\^t) + \frac{\alpha_{i-1}}{2} + \sqrt{\E_{\pi \sim p\^t} \left[\hell{\wh M\^t(\pi)}{M^\st(\pi)}\right]}\nonumber.
  \end{align}
\end{lemma}
\begin{proof}[\pfref{lem:epoch-optimal-pol}]
  Fix any $t \in \ME_i$. %
  The choice of $p\^t$ in \cref{line:regret_algo_dec} ensures that %
  \begin{align}
    \E_{\pi \sim p\^t} [f\sups{\wh M\^t}(\pi\subs{\wh M\^t}) - f\sups{\wh M\^t}(\pi)] \leq \deccreg[\vep_i](\MM \cup \{ \wh M\^t \}, \wh M\^t) \label{eq:h-cap-m-ub}. 
  \end{align}

  Next, under the event $\SA_{i-1}$, we have
  \begin{align}
    & \E_{\pi \sim p\^t} \left[ \fmstar(\pimstar) - \fmstar(\pi) \right] \nonumber\\
    & \leq \E_{\pi \sim p\^t} \left[ f\sups{\wh M\^t}(\pi\subs{\Mhat\ind{t}}) - f\sups{\wh M\^t}(\pi) \right] + \frac{\alpha_{i-1}}{2} + \sqrt{\E_{\pi \sim p\^t} \left[\hell{\wh M\^t(\pi)}{M^\st(\pi)}\right]} \nonumber\\
    & \leq \deccreg[\vep_i](\MM \cup \{ \wh M\^t\} ,\wh M\^t) + \frac{\alpha_{i-1}}{2} + \sqrt{\E_{\pi \sim p\^t} \left[\hell{\wh M\^t(\pi)}{M^\st(\pi)}\right]}\nonumber,
  \end{align}
  where the first inequality uses \cref{lem:compare-arb-distr} at epoch $i-1$ with $p = p\^t$ and $M = \wh M\^t$, together with the fact that $\wh M\^t \in \co(\MM_i)$ by construction, and the second inequality uses \eqref{eq:h-cap-m-ub}. %
\end{proof}

\subsubsection{Proof of \creftitle{thm:regret_upper}}

\begin{proof}[\pfref{thm:regret_upper}]
  Let us condition on the event $\SA \cap \SB \cap \SC \cap \SD$, which, by \cref{lem:good-event-lb}, holds with probability $1-3L\delta = 1 - 3 \lceil \log 1/\delta \rceil \cdot \delta $. 
  Fix $i\in\brk{N}$. We analyze the regret in each epoch $i$ as follows.
  \begin{itemize}
  \item We first analyze the rounds in $t\in\ME_i$. By \cref{lem:epoch-optimal-pol}, under the event $\SA_{i-1}$, we have
    \begin{align}
      & \sum_{t \in \ME_i} \E_{\pi \sim p\^t} \left[ \fmstar(\pimstar) - \fmstar(\pi) \right]\nonumber\\
      & \leq |\ME_i| \cdot \left(\sup_{\Mbar \in \co(\MM)} \deccreg[\vep_i](\MM \cup \{ \Mbar \}, \Mbar) + \alpha_{i-1} \right) + \sum_{t \in \ME_i} \sqrt{\E_{\pi \sim p\^t} \brk*{\hell{\wh M\^t(\pi)}{M^\st(\pi)}}}\nonumber\\
      & \leq |\ME_i| \cdot \left( \alpha_i + \alpha_{i-1} \right) + \sqrt{|\ME_i| \cdot \sum_{t \in \ME_i} \E_{\pi \sim p\^t} \brk*{ \hell{\wh M\^t(\pi)}{M^\st(\pi)}}} \nonumber\\
      & \leq 2 \cdot |\ME_i| \cdot \alpha_{i-1} + \sqrt{|\ME_i| \cdot \Est(|\ME_i|, \delta)}\label{eq:regret-ei},
    \end{align}
    where the second-to-last inequality follows by our choice of $\alpha_i$ and the final inequality follows from the fact that $\SB_i$ holds and $\alpha_i \leq \alpha_{i-1}$. 
  \item We next analyze the rounds in $\MR_i$. %
    We first analyze those rounds in which a decision $\pi_s\^j \sim p\^s$ was sampled on \cref{line:sample-pisj-ps}. To do so, fix any $s \in \MS_i$. We first note that, by definition of $J_{i,s}$,
    \begin{align}
      \sum_{j=1}^{J_{i,s}} \sqrt{ \E_{\pi \sim p\^s} \left[ \hell{\til M_s\^j(\pi)}{\wh M\^s(\pi)} \right]} \leq & \sqrt{J_{i,s} \cdot \sum_{j=1}^{J_{i,s}} \E_{\pi \sim p\^s} \left[ \hell{\til M_s\^j(\pi)}{\wh M\^s(\pi)} \right]}\nonumber\\
      \leq & \sqrt{J_{i,s} \cdot \left( \frac{J_i \vep_i^2}{4} + 2 \right)} \leq \sqrt{2J_{i}} + J_i \vep_i/2\nonumber.
    \end{align}
    Furthermore, since the event $\SA_i$ holds (in particular, using \cref{eq:mtil-mstar-close-ai}), we have
    \begin{align}
      \sum_{j=1}^{J_{i,s}} \sqrt{\E_{\pi \sim p\^s} \left[ \hell{\til M_s\^j(\pi)}{\Mstar(\pi)} \right]} \leq & \sqrt{J_{i,s} \cdot \sum_{j=1}^{J_{i,s}} \E_{\pi \sim p\^s} \left[ \hell{\til M_s\^j(\pi)}{\Mstar(\pi)} \right]} \nonumber\\
      \leq & \sqrt{J_{i,s} \cdot \Est(J_i, \delta)} \leq \sqrt{J_i^2 \vep_i^2/32} \leq J_i \vep_i\nonumber,
    \end{align}
    where the second-to-last inequality uses \cref{eq:bound-ei-ji}. 
Using the above displays, we have 
    \begin{align}
      & \sum_{j=1}^{J_{i,s}} \E_{\pi_s\^j \sim p\^s} \left[ \fmstar(\pimstar) - \fmstar(\pi_s\^j)\right]\nonumber\\
      \leq & J_{i,s} \cdot \left( \deccreg[\vep_i](\MM) + \alpha_{i-1} + \sqrt{\E_{\pi \sim p\^s} \left[ \hell{\wh M\^s(\pi)}{\Mstar(\pi)} \right] }\right)\nonumber\\
      \leq & 2J_{i,s} \cdot \alpha_{i-1} + \sum_{j=1}^{J_{i,s}} \sqrt{ 2\E_{\pi \sim p\^s} \left[ \hell{\wh M\^s(\pi)}{\til M_s\^j(\pi)} \right]} + \sqrt{2\E_{\pi \sim p\^s} \left[ \hell{\Mstar(\pi)}{\til M_s\^j(\pi)}\right]}\nonumber\\
      \leq & 2J_i \cdot \alpha_{i-1} + \frac{3J_i\vep_i}{2} + \sqrt{2J_i}\nonumber.
    \end{align}

    Next we analyze the rounds $t \in \MR_i$ where $\pi\^t \sim p\^{s_i} = \wh p_i$ on \cref{line:play-psi}. Since $\SA_i$ holds, we have from \cref{lem:refined-acc} that $\E_{\pi\^t \sim \wh p_i} [\fmstar(\pimstar) - \fmstar(\pi)] \leq \alpha_i/4$, meaning that the total contribution to the regret from such rounds $t \in \MR_i$ is at most $|\MR_i| \cdot \alpha_i/4$. Thus, the overall contribution to regret from rounds in $\MR_i$ is bounded above as follows:
    \begin{align}
      \sum_{t \in \MR_i} \E_{\pi \sim p\^t} [\fmstar(\pimstar) - \fmstar(\pi)] \leq & \frac{|\MR_i| \alpha_i}{4} + L \cdot \left( 2J_i \alpha_i + \frac{3J_i\vep_i}{2} + \sqrt{2J_i}\right)\nonumber\\
      \leq & 4 |\MR_i| \alpha_i +  \sqrt{2L|\MR_i|}\nonumber,
    \end{align}
    where in the second inequality we have used that $|\MR_i| = J_i \cdot L$ and $3\vep_i/2 \leq \alpha_i$. 
  \end{itemize}
  Summarizing, under the event $\SA \cap \SB \cap \SC \cap \SD$, the total regret is bounded above by
  \begin{align}
    \sum_{t=1}^T \E_{\pi \sim p\^t} \left[ \fmstar(\pimstar) - \fmstar(\pi) \right] \leq & \sum_{i=1}^N \left( 4 \alpha_{i-1} \cdot (|\MR_i| + |\ME_i|) + \sqrt{|\ME_i| \cdot \Est(|\ME_i|, \delta)} + \sqrt{2L |\MR_i|} \right)\label{eq:reg-bound-ugly}.
  \end{align}
  
  We now simplify the expression in \eqref{eq:reg-bound-ugly}. Recall that \cref{asm:dec} gives that for all $\vep > 0$,
  \begin{align}
 \deccreg[\vep](\MM) \leq C_{\reg}^2 \cdot  \deccreg[\vep/C_\reg](\MM)\nonumber.
  \end{align}
  Applying this inequality a total of $\left\lceil \frac{\log(\vep_i / \vep_N)}{\log(C_{\reg})} \right\rceil$ times for each $i \in [N]$ gives that
  \begin{align}
 \deccreg[\vep_i](\MM) & \leq C_{\reg}^2 \cdot \left(\frac{\vep_i}{\vep_N}\right)^2 \cdot  \deccreg[\vep_N](\MM)\nonumber\\
    &= C_{\reg}^2 \cdot 2^{N-i} \cdot \deccreg[\vep_N](\MM)\nonumber.
  \end{align}
  
  Then by the choice $\alpha_i = C_0 \cdot \deccreg[\vep_i](\MM) + 64 \vep_i$ for each $i \in [N]$, we have
  \begin{align}
    \sum_{i=1}^N \alpha_{i-1} \cdot 2^i & \leq 64 \sum_{i=1}^N \vep_{i-1} \cdot 2^i + C_0 \sum_{i=1}^N 2^i \cdot  \deccreg[\vep_{i-1}](\MM)\nonumber\\
    & \leq 64 \sum_{i=1}^N \vep_N \cdot \sqrt{2^{N+1+i}} + O\left( \sum_{i=1}^N 2^N \cdot  \deccreg[\vep_N](\MM)\right)\nonumber\\
    & \leq 128 \cdot \sqrt{C_1 \cdot \Est(T, \delta) \cdot L} \cdot  \sum_{i=1}^N \sqrt{2^{i+1}} + O \left(NT \cdot  \deccreg[\vep_N](\MM) \right)\nonumber.
  \end{align}
    Therefore, we may upper bound the expression in \eqref{eq:reg-bound-ugly} as follows (using that $\Est(|\ME_i|, \delta) \leq \Est(T, \delta)$ for each $i$): %
  \begin{align}
    & \sum_{i=1}^N \left( 4 \alpha_{i-1} \cdot (|\MR_i| + |\ME_i|) + \sqrt{|\ME_i| \cdot \Est(|\ME_i|, \delta)} + \sqrt{2L |\MR_i|} \right)\nonumber\\
    \leq & 4 \sum_{i=1}^N \alpha_{i-1} \cdot 2^i + \sqrt{\Est(T, \delta)} \sum_{i=1}^N \sqrt{2^i} + \sqrt{2L} \sum_{i=1}^N \sqrt{2^i}\nonumber\\
    \leq & O \left(\sqrt{\Est(T, \delta) \cdot L} \sum_{i=1}^N \sqrt{2^i} + \sqrt{L} \sum_{i=1}^N \sqrt{2^i} + NT \cdot \deccreg[\vep_N](\MM) \right)\nonumber\\
    \leq & O \left( \sqrt{T \log(1/\delta) \cdot \Est(T, \delta)} + T \log(T) \cdot \deccreg[\vep_N](\MM) \right)\nonumber,
  \end{align}
  where we have used that $L = O(\log 1/\delta)$ in the final inequality.
    The proof is completed by rescaling from $\delta$ to $\delta^2$ and noting that, by construction, we have $\vep_N \leq  C\cdot \sqrt{\frac{\Est(T, 1/\delta) \cdot \log1/\delta}{T}}$ for a universal constant $C>0$. 
\end{proof}

\section{Proofs and Additional Results from \creftitle{sec:properties}}
\label{app:properties}

\subsection{Technical Lemmas}

\begin{lemma}
  \label{lem:hellinger_to_value}
  Let $M$ and $\Mbar$ have $\cR\subseteq\brk*{0,1}$. Then for all
  $\veps\geq{}0$ and $p\in\Delta(\Pi)$, if $M\in\cH_{p,\veps}(\Mbar)$, then
  \begin{align}
    \Enp\brk*{\abs{\fm(\pi)-\fmbar(\pi)}}\leq\veps.
  \end{align}
\end{lemma}
\begin{proof}[\pfref{lem:hellinger_to_value}]
  Since rewards are in $\brk*{0,1}$, we have
  \begin{align*}
    \Enp\brk*{\abs{\fm(\pi)-\fmbar(\pi)}}
    \leq{}\Enp\brk*{\Dtv{M(\pi)}{\Mbar(\pi)}}
    \leq{}    \sqrt{\Enp\brk*{\Dhels{M(\pi)}{\Mbar(\pi)}}}\leq\veps.
  \end{align*}
\end{proof}

\begin{lemma}
  \label{lem:localization_one_sided}
  Fix a model class $\cM$. Let $\Mbar\in\cMall$ and $\veps>0$ be given, and set
  \[
    \cM'=\crl{M\in\cM\mid{}\fmbar(\pimbar) \leq \fm(\pimbar) + \veps}.
  \]
  Then we have
  \begin{align*}
    \deccreg[\veps/\sqrt{2}](\cM,\Mbar)
    \leq \deccreg(\cM',\Mbar) + \veps.
  \end{align*}
\end{lemma}
\begin{proof}[\pfref{lem:localization_one_sided}]
  Let $p\in\Delta(\Pi)$ achieve the value of $\deccreg(\cM',\Mbar)$,
  and set $p'=\frac{1}{2}p+\frac{1}{2}\indic_{\pimbar}$. Let
  $M\in\cH_{p',\veps/\sqrt{2}}(\Mbar)\subseteq\cH_{p,\veps}(\Mbar)\cap\cH_{\indic_{\pimbar},\veps}(\Mbar)$. We claim that
  $M\in\cM'$. Indeed,
  \begin{align*}
    \fmbar(\pimbar) -\fm(\pimbar) \leq \Dhel{\Mbar(\pimbar)}{M(\pimbar)}\leq\veps
  \end{align*}
  by \pref{lem:hellinger_to_value}. It follows that
    \begin{align}
      \sup_{M\in\cH_{p',\veps/\sqrt{2}}(\Mbar)}\En_{\pi\sim{}p}\brk*{\fm(\pim)-\fm(\pi)}
    \leq{}
      \sup_{M\in\cH_{p,\veps}(\Mbar)\cap\cM'}\En_{\pi\sim{}p}\brk*{\fm(\pim)-\fm(\pi)}
      \leq \deccreg(\cM',\Mbar),
      \label{eq:one_sided1}
    \end{align}
    so that
  \begin{align*}
    \sup_{M\in\cH_{p',\veps/\sqrt{2}}(\Mbar)}\En_{\pi\sim{}p'}\brk*{\fm(\pim)-\fm(\pi)}
    \leq\frac{1}{2}\deccreg(\cM',\Mbar)
    +  \frac{1}{2}\sup_{M\in\cH_{p',\veps/\sqrt{2}}(\Mbar)}\brk*{\fm(\pim)-\fm(\pimbar)}.
  \end{align*}
  To bound the final term above, we have
  \begin{align*}
    \sup_{M\in\cH_{p',\veps/\sqrt{2}}(\Mbar)}\brk*{\fm(\pim)-\fm(\pimbar)}
    &\leq{}
      \sup_{M\in\cH_{p',\veps/\sqrt{2}}(\Mbar)}\brk*{\fm(\pim)-\fmbar(\pimbar)}
      + \veps \\
    &\leq{}
      \sup_{M\in\cH_{p',\veps/\sqrt{2}}(\Mbar)}\Enp\brk*{\fm(\pim)-\fmbar(\pi)}
      + \veps\\
    &\leq{}
      \sup_{M\in\cH_{p',\veps/\sqrt{2}}(\Mbar)}\Enp\brk*{\fm(\pim)-\fm(\pi)}
      + 2\veps\\
    &\leq{} \deccreg(\cM',\Mbar) + 2\veps,
  \end{align*}
  where the first and third inequalities use
  \pref{lem:hellinger_to_value}, and the last inequality applies \eqref{eq:one_sided1}.
  
\end{proof}

\begin{lemma}
  \label{lem:pac_constrained_alt}
For a model class $\cM$ and reference model $\Mbar\in\cMall$, define
\begin{align}
  \label{eq:define-deccpacalt}
  \deccpacalt(\cM,\Mbar)
  = \inf_{p,q\in\Delta(\Pi)}\sup_{M\in\cH_{p,\veps}(\Mbar)\cap\cH_{q,\veps}(\Mbar)}\Enp\brk*{\fm(\pim)-\fm(\pi)},
\end{align}
{with the convention that the value above is zero when $\cH_{p,\veps}(\Mbar)\cap\cH_{q,\veps}(\Mbar)=\emptyset$.}
For all $\Mbar\in\cMall$ and $\veps>0$, we have
  \begin{align}
    \deccpacalt(\cM,\Mbar)
    \leq{}    \deccpac(\cM,\Mbar)
        \leq{}    \deccpacalt[\sqrt{2}\veps](\cM,\Mbar).
  \end{align}
\end{lemma}
\begin{proof}[\pfref{lem:pac_constrained_alt}]
  The first inequality is immediate. For the second, we have
  \begin{align*}
      \deccpac(\cM,\Mbar)
  \leq{} \inf_{p,q\in\Delta(\Pi)}\sup_{M\in\cH_{\frac{1}{2}p+\frac{1}{2}q,\veps}(\Mbar)}\Enp\brk*{\fm(\pim)-\fm(\pi)},
  \end{align*}
  by observing that for any minimizer $q$ for $\deccpac$, we can
  arrive at an upper bound by substituting
  $q'=\frac{1}{2}p+\frac{1}{2}q$. The result now follows because $\cH_{\frac{1}{2}p+\frac{1}{2}q,\veps}(\Mbar)\subseteq\cH_{p,\sqrt{2}\veps}(\Mbar)\cap\cH_{q,\sqrt{2}\veps}(\Mbar)$.
\end{proof}

\subsection{Additional Properties of the \CompText}
\label{app:additional_properties}

\subsubsection{Localization}
\label{sec:localization_additional}

The following result is an extension of
\pref{prop:constrained_offset_localized} which accommodates randomized estimators.
\begin{proposition}
  \label{prop:constrained_offset_localized_rand}
  Let $\alpha,\gamma>0$ and $\nu\in\Delta(\cM)$ be given. Let
  $\Mbarnu=\En_{M'\sim\nu}\brk*{M'}$. For all
  $\veps>0$, we have
  \begin{align}
        \label{eq:constrained_offset_localized_rand1}
    \decoregr(\MM_\alpha(\Mbarnu)\cup\crl{\Mbarnu}, \nu)
    \leq 
    \deccregr(\cM\cup\crl{\Mbarnu},\nu) + \max \left\{0,\ \alpha + \frac{1}{2\gamma}  - \frac{\gamma \vep^2}{2} \right\}.
  \end{align}
  which in particular yields
  \begin{align}
    \label{eq:constrained_offset_localized_rand2}
         \decoregr(\MM_\alpha(\Mbarnu)\cup\crl{\Mbarnu}, \nu) 
      \leq  \deccregr[\sqrt{2\alpha/\gamma}](\cM\cup\crl{\Mbarnu},\nu) + \frac{1}{2\gamma}.
  \end{align}
\end{proposition}
\begin{proof}[\pfref{prop:constrained_offset_localized_rand}]%
  \newcommand{\cMnu}{\cM\cup\crl{\Mbarnu}}%
      Fix $\nu \in \Delta(\cM)$ and $\veps>0$, and let $p \in \Delta(\Pi)$
  be a minimizer for $\deccregr[\veps](\cM\cup\crl{\Mbarnu},\nu)$. Fix any $M \in \MM_\alpha(\Mbarnu)\cup\crl{\Mbarnu}$. We bound the regret under
  $p$ by considering two cases.

  \noindent\emph{Case 1.} If $\En_{\Mbar\sim\nu}\E_{\pi \sim p} \left[
    \hell{M(\pi)}{\Mbar(\pi)} \right] \leq \vep^2$, it follows from
  the definition $\deccregr[\veps](\cM\cup\crl{\Mbarnu},\nu)$ of that $\E_{\pi \sim p} \left[ \fm(\pim) - \fm(\pi) \right] \leq \deccregr[\vep](\cM\cup\crl{\Mbarnu}, \nu)$.

\noindent\emph{Case 2.} For the second case, suppose that $\En_{\Mbar\sim\nu}\E_{\pi \sim p} \left[ \hell{M(\pi)}{\Mbar(\pi)} \right] > \vep^2$.  We now compute
  \begin{align}
    \E_{\pi \sim p} \left[ \fm(\pim) - \fm(\pi) \right] & \leq \alpha + \E_{\pi \sim p} \left[ \fmbarnu(\pimbarnu) - \fm(\pi) \right]\nonumber\\
    & \leq \alpha + \E_{\pi \sim p} \left[ \fmbarnu(\pimbarnu) - \fmbarnu(\pi) \right] + \frac{1}{2\gamma} + \frac{\gamma}{2} \cdot \E_{\pi \sim p} \left[ (\fm(\pi) - \fmbarnu(\pi))^2 \right]\nonumber\\
    & \leq \alpha + \deccregr[\vep](\cM\cup\crl{\Mbarnu}, \nu) +
      \frac{1}{2\gamma} + \frac{\gamma}{2} \cdot \E_{\pi \sim p}
      \left[ \hell{M(\pi)}{\Mbarnu(\pi)} \right]\nonumber, \\
          & \leq \alpha + \deccregr[\vep](\cM\cup\crl{\Mbarnu}, \nu) + \frac{1}{2\gamma} + \frac{\gamma}{2} \cdot \En_{\Mbar\sim\nu}\E_{\pi \sim p} \left[ \hell{M(\pi)}{\Mbar(\pi)} \right]\nonumber,
  \end{align}
  where the second inequality uses Young's inequality and the final inequality uses convexity of the squared Hellinger distance. 
  Rearranging, we obtain
  \begin{align}
    &\E_{\pi \sim p} \left[ \fm(\pim) - \fm(\pi) - \gamma \cdot
      \En_{\Mbar\sim\nu}\brk*{\hell{M(\pi)}{\Mbar(\pi)}} \right] \nonumber\\
    & \leq \alpha + \deccregr[\vep](\cM\cup\crl{\Mbar}, \nu) + \frac{1}{2\gamma} - \frac{\gamma}{2} \cdot \En_{\Mbar\sim\nu}\E_{\pi \sim p} \left[ \hell{M(\pi)}{\Mbar(\pi)}\right]\nonumber\\
    & \leq \alpha +  \deccregr[\vep](\cMnu, \Mbarnu) +
      \frac{1}{2\gamma} - \frac{\gamma \vep^2}{2}.\nonumber
  \end{align}
  Recalling that $M$ can be any model in $\MM_\alpha(\Mbarnu)\cup\crl{\Mbarnu}$, we obtain
  \begin{align}
    \decoregr(\cMloc(\Mbarnu)\cup\crl{\Mbarnu}, \nu) & \leq \max \left\{
                                     \deccregr[\vep](\cMnu, \nu),\
                                     \alpha + \frac{1}{2\gamma} +
                                                       \deccregr[\vep](\cMnu, \nu) -
                                     \frac{\gamma \vep^2}{2}
                                     \right\}\nonumber \\
    & = \deccregr[\vep](\cMnu, \nu) + \max \left\{0,\ \alpha + \frac{1}{2\gamma}  - \frac{\gamma \vep^2}{2} \right\}.\nonumber
  \end{align}
  
\end{proof}

\subsubsection{Role of Convexity for PAC DEC}
\label{sec:convexity_pac}

For $\nu\in\Delta(\cM)$, we define ``randomized''  variants of the PAC \CompShort,
analogous to those introduced in \pref{sec:properties}, as follows:
\begin{align}
  \label{eq:dec_pac_randomized}
  &\deccpacr(\cM,\nu)=
  \inf_{p,q\in\Delta(\Pi)}\sup_{M\in\cM}\crl*{\En_{\pi\sim{}p}\brk*{
  \fm(\pim) - \fm(\pi)}
  \mid\En_{\Mbar\sim\nu}\En_{\pi\sim{}q}\brk*{\Dhels{M(\pi)}{\Mbar(\pi)}}\leq\veps^2
    }, \\
  &\decopacr(\cM,\nu)=
  \inf_{p,q\in\Delta(\Pi)}\sup_{M\in\cM}\crl*{\En_{\pi\sim{}p}\brk*{
  \fm(\pim) - \fm(\pi)}
  - \gamma\cdot\En_{\Mbar\sim\nu}\En_{\pi\sim{}q}\brk*{\Dhels{M(\pi)}{\Mbar(\pi)}}}.
\end{align}
The following result provides a PAC counterpart to
\eqref{eq:conv_equivalence_regret1} of \pref{prop:conv_equivalence_regret}.
\begin{proposition}
  \label{prop:conv_equivalence_pac}
  Suppose that \pref{ass:minimax} holds. For all $\gamma>0$, we have
  \begin{align}
    \label{eq:conv_equivalence_pac1}
    \sup_{\Mbar\in\cMall}\decopac(\cM,\Mbar)
    \leq{} \sup_{\nu\in\Delta(\cM)}\decopacr[\gamma/4](\cM,\nu)
    \leq{} \sup_{\Mbar\in\conv(\cM)}\decopac[\gamma/4](\cM,\Mbar).
  \end{align}

\end{proposition}
A PAC analogue of \eqref{eq:conv_equivalence_regret2} can be proven by
adapting the proof of \pref{prop:conv_equivalence_regret}; we do not
include this result.

\begin{proof}[\pfref{prop:conv_equivalence_pac}]
  Let $\Mbar\in\cMall$ and $\gamma>0$ be given. We first prove the inequality
    \pref{eq:conv_equivalence_pac1}. By \pref{ass:minimax}, we have
  \begin{align*}
    \decopac(\cM,\Mbar)
  = \sup_{\mu\in\Delta(\cM)}\inf_{p,q\in\Delta(\Pi)}\En_{M\sim\mu}\brk*{\En_{\pi\sim{}p}\brk*{\fm(\pim)-\fm(\pi)}-\gamma\cdot\En_{\pi\sim{}q}\brk*{\Dhels{M(\pi)}{\Mbar(\pi)}}}.
  \end{align*}
  Since Hellinger distance satisfies the triangle
  inequality, we have that for all $\pi\in\Pi$,
  \begin{align*}
    \En_{M,M'\sim\mu}\brk*{\Dhels{M(\pi)}{M'(\pi)}}
      &\leq{}     2\En_{M\sim\mu}\brk*{\Dhels{M(\pi)}{\Mbar(\pi)}}
    +    2\En_{M'\sim\mu}\brk*{\Dhels{M'(\pi)}{\Mbar(\pi)}} \\
    &= 4\En_{M\sim\mu}\brk*{\Dhels{M(\pi)}{\Mbar(\pi)}}.
  \end{align*}
  It follows that
  \begin{align*}
    \decopac(\cM,\Mbar)
    &\leq
\sup_{\mu\in\Delta(\cM)}\inf_{p,q\in\Delta(\Pi)}\En_{M\sim\mu}\brk*{\En_{\pi\sim{}p}\brk*{\fm(\pim)-\fm(\pi)}-\frac{\gamma}{4}\cdot\En_{M'\sim\mu}\En_{\pi\sim{}q}\brk*{\Dhels{M(\pi)}{M'(\pi)}}}
    \\
        &\leq
\sup_{\nu\in\Delta(\cM)}\sup_{\mu\in\Delta(\cM)}\inf_{p,q\in\Delta(\Pi)}\En_{M\sim\mu}\brk*{\En_{\pi\sim{}p}\brk*{\fm(\pim)-\fm(\pi)}-\frac{\gamma}{4}\cdot\En_{M'\sim\nu}\En_{\pi\sim{}q}\brk*{\Dhels{M(\pi)}{M'(\pi)}}}
    \\
    &\leq
\sup_{\nu\in\Delta(\cM)}\inf_{p,q\in\Delta(\Pi)}\sup_{M\in\cM}\crl*{\En_{\pi\sim{}p}\brk*{\fm(\pim)-\fm(\pi)}-\frac{\gamma}{4}\cdot\En_{M'\sim\nu}\En_{\pi\sim{}q}\brk*{\Dhels{M(\pi)}{M'(\pi)}}}
    \\
    &= \sup_{\nu\in\Delta(\cM)}\decopacr[\gamma/4](\cM,\nu).
  \end{align*}
Jensen's inequality further implies that $\sup_{\nu\in\Delta(\cM)}\decopacr[\gamma/4](\cM,\nu)\leq \sup_{\Mbar\in\conv(\cM)}\decopac[\gamma/4](\cM,\Mbar)$.

\end{proof}

\subsubsection{PAC DEC with Greedy Decisions}

For a model class $\cM$ and reference model $\Mbar\in\cMall$, define
\[
  \deccpacg(\cM,\Mbar)=
  \inf_{q\in\Delta(\Pi)}\sup_{M\in\cM}\crl*{
  \fm(\pim) - \fm(\pimbar)
  \mid\En_{\pi\sim{}q}\brk*{\Dhels{M(\pi)}{\Mbar(\pi)}}\leq\veps^2
  },
\]
{with the convention that the value above is zero when $\cH_{q,\veps}(\Mbar)=\emptyset$.}

\begin{proposition}
  \label{prop:greedy_equivalence}
  For all $\veps>0$ and $\Mbar\in\cMall$, we have
  \begin{align}
    \deccpac(\cM,\Mbar)
    \leq     \deccpacg(\cM,\Mbar)
    \leq{}     \deccpac[\sqrt{3}\veps](\cM,\Mbar) + 4\veps.
  \end{align}
\end{proposition}

\begin{proof}[\pfref{prop:greedy_equivalence}]
  It is immediate that $\deccpacg(\cM,\Mbar)\geq\deccpac(\cM,\Mbar)$,
  so let us prove the second inequality. Let $\Mbar\in\cM$ and $\veps>0$
  be given, and let $(p,q)$ be minimizers for
  $\deccpacalt(\cM,\Mbar)$. Define $q'=\frac{1}{3}q +\frac{1}{3}p + 
  \frac{1}{3}\indic_{\pimbar}$. Note that
  $\cH_{q',\veps/\sqrt{3}}(\Mbar)\subseteq
  \cH_{q,\veps}(\Mbar)\cap\cH_{p,\veps}(\Mbar)\cap\cH_{\indic_{\pimbar},\veps}(\Mbar)$. As a
  result, for all $M\in\cH_{q',\veps/\sqrt{3}}(\Mbar)$, we have
  \begin{align*}
    \fm(\pim) -\fm(\pimbar)
    & =     \fm(\pim) -\fmbar(\pimbar) + (\fmbar(\pimbar) -
      \fm(\pimbar))\\
    & \leq     \fm(\pim) -\fmbar(\pimbar) + \veps \\
    & \leq     \Enp\brk*{\fm(\pim) -\fmbar(\pi)} + \veps\\
    & \leq     \Enp\brk*{\fm(\pim) -\fm(\pi)} + 2\veps\\
    & \leq \deccpacalt(\cM, \Mbar) + 2\vep\\
    & \leq     \deccpac(\cM,\Mbar) + 2\veps,
  \end{align*}
  where the first and third inequalities use \pref{lem:hellinger_to_value} and the final inequality uses \Cref{lem:pac_constrained_alt}.  \end{proof}

\subsection{Omitted Proofs}
\begin{proof}[\pfref{prop:constrained_offset_union}]%
  Let $\Mbar\in\cMall$ be given. We first prove a more
  general result under the assumption that for some $\delta>0$, 
  $\fmbar(\pimbar)\leq{}\fm(\pim)+\delta$ for all $M\in\cM$:
\begin{align}\label{eq:constrained_offset_union_intermediate}
  \deccreg(\cMu,\Mbar)
  \leq{} 
   \delta +
  \inf_{\gamma>0}\crl*{\decoreg(\cM,\Mbar)\vz +4\gamma\veps^2 +
  (4\gamma)^{-1}} .
\end{align}
  Then, at the end of the proof, we show that it is possible to take $\delta=\bigoh(\veps)$ without loss of generality.

  Let $\gamma > 0$ be given and let $p_0$ be the
  minimizer for $\decoreg(\cM,\Mbar)$, so that
  \[
    \sup_{M\in\cM}\En_{\pi\sim\pnot}\brk*{
      \fm(\pim)-\fm(\pi)
      -\gamma\cdot\Dhels{M(\pi)}{\Mbar(\pi)}
      } \leq \decoreg(\cM,\Mbar).
    \]
Let
$\Mstar\ldef{}\argmin_{M\in\cM}\En_{\pi\sim\pnot}\brk*{\Dhels{M(\pi)}{\Mbar(\pi)}}$,
and let
$\Delta^{2}\ldef{}\En_{\pi\sim\pnot}\brk*{\Dhels{\Mstar(\pi)}{\Mbar(\pi)}}$.

We
will bound the constrained \CompShort, $\deccreg(\cMu,\Mbar)$, by
playing the distribution
\[
p\ldef{} (1-q)\cdot{}\indic_{\pimbar} + q\cdot{}\pnot,
\]
where
\[
q\ldef \frac{2\veps^2}{\Delta^2}\wedge{}1.
\]
Before proceeding, we state a basic technical lemma.
\begin{lemma}
  \label{lem:constrained_offset_union_prelim}
  The distribution $\pnot$ satisfies
  \begin{align}
  \En_{\pi\sim{}\pnot}\brk*{\fmbar(\pimbar)-\fmbar(\pi)}
    &\leq{}   \delta +
      \decoreg(\cM,\Mbar) + (4\gamma)^{-1}+2\gamma\Delta^2.\label{eq:constrained_offset_union_prelim}
\end{align}  
\end{lemma}
We bound the value of the constrained DEC for $p$ by considering two cases.
\paragraphi{Case 1: $q=1$}
If $q=1$, then $\Delta^2\leq2\veps^2$, and $p=\pnot$. For models
$M\in\cH_{p,\veps}(\Mbar)$, the definition of $\pnot$ implies that
\begin{align*}
\En_{\pi\sim{}p}\brk*{
  \fm(\pim)-\fm(\pi)}
  &\leq{} \decoreg(\cM,\Mbar) + 
    \gamma\cdot\En_{\pi\sim{}p}\brk*{\Dhels{M(\pi)}{\Mbar(\pi)}}\\
    &\leq{} \decoreg(\cM,\Mbar) + 
      \gamma\veps^2.
\end{align*}
For the model $\Mbar$, \pref{lem:constrained_offset_union_prelim} implies that
\begin{align*}
  \En_{\pi\sim{}p}\brk*{\fmbar(\pimbar)-\fmbar(\pi)}
    &\leq{}   \delta +
      \decoreg(\cM,\Mbar) + (4\gamma)^{-1}+4\gamma\veps^2.
\end{align*}

\paragraphi{Case 2: $q<1$}
If $q<1$, then for all $M\in\cM$, we have
\begin{align*}
  \En_{\pi\sim{}p}\brk*{\Dhels{M(\pi)}{\Mbar(\pi)}}
  &\geq{}
    q\cdot{}\En_{\pi\sim{}\pnot}\brk*{\Dhels{M(\pi)}{\Mbar(\pi)}}  \\
  &\geq{}
    q\cdot{}\En_{\pi\sim{}\pnot}\brk*{\Dhels{\Mstar(\pi)}{\Mbar(\pi)}}
  =  2\veps^2 > \veps^2,
\end{align*}
where the second inequality uses that $\Mstar$ minimizes
$\En_{\pi\sim{}\pnot}\brk*{\Dhels{M(\pi)}{\Mbar(\pi)}}$, and the last
inequality uses that $q=\frac{2\veps^2}{\Delta^2}$ whenever
$q<1$. It
follows that $\cH_{p,\veps}(\Mbar)\cup\crl{\Mbar}=\crl{\Mbar}$, so
we only need to bound the regret of the distribution $p$ under
$\Mbar$. To do so, we observe that
\begin{align*}
  \En_{\pi\sim{}p}\brk*{\gmbar(\pi)}
  = q\cdot{}\En_{\pi\sim{}\pnot}\brk*{\gmbar(\pi)}
  &\leq{} q\cdot\prn*{
    \delta +
    \decoreg(\cM,\Mbar)+(4\gamma)^{-1} + 2\gamma\Delta^2 }\\
  &\leq{}     \delta +
    \decoreg(\cM,\Mbar)\vz + (4\gamma)^{-1}+q\cdot{}2\gamma\Delta^2 \\
    &=     \delta +
      \decoreg(\cM,\Mbar)\vz + (4\gamma)^{-1}+4\gamma\veps^2,
\end{align*}
where the first inequality uses
\pref{lem:constrained_offset_union_prelim}, and the final equality
uses that $q=\frac{2\veps^2}{\Delta^2}$.

\paragraphi{Finishing up}
We have established that
\[
  \deccreg(\cMu,\Mbar)
  \leq{} 
   \delta +
  \inf_{\gamma>0}\crl*{\decoreg(\cM,\Mbar)\vz +4\gamma\veps^2 +
    (4\gamma)^{-1}}
\]
whenever $\fmbar(\pimbar)\leq\fm(\pim)+\delta$ for all $M\in\cM$. To
conclude, we appeal to \pref{lem:localization_one_sided} applied to the class $\MM \cup \{ \Mbar \}$, which implies
that
\begin{align}
  \deccreg(\cMu,\Mbar)
  \leq{}   \deccreg[\sqrt{2}\veps](\cM'\cup\crl{\Mbar},\Mbar) + \sqrt{2}\veps,\label{eq:relate-cm-cmprime}
\end{align}
where
$\cM'=\crl*{M\in\cM\mid{}\fmbar(\pimbar)\leq{}\fm(\pim)+\sqrt{2}\veps}$. Applying
\pref{eq:constrained_offset_union_intermediate} to the quantity
$\deccreg[\sqrt{2}\veps](\cM'\cup\crl{\Mbar},\Mbar)$ and combining with \pref{eq:relate-cm-cmprime} yields
\begin{align*}
  \deccreg(\cMu,\Mbar)
  \leq{} 2\sqrt{2}\veps +   \inf_{\gamma>0}\crl*{\decoreg(\cM,\Mbar)\vz +8\gamma\veps^2 +
  (4\gamma)^{-1}}.
\end{align*}
To simplify this result slightly, we consider two cases. If
$\decoreg(\cM,\Mbar)\leq{}(4\gamma)^{-1}$, then choosing
$\gamma=(4\veps)^{-1}$ gives
$\inf_{\gamma>0}\crl*{\decoreg(\cM,\Mbar) +4\gamma\veps^2 +
  (4\gamma)^{-1}}\leq{}4\veps$. Otherwise, we have $\inf_{\gamma>0}\crl*{\decoreg(\cM,\Mbar) +4\gamma\veps^2 +
  (4\gamma)^{-1}}\leq{} \inf_{\gamma>0}\crl*{2\cdot{}\decoreg(\cM,\Mbar) +8\gamma\veps^2}$.
  
\end{proof}
\begin{proof}[\pfref{lem:constrained_offset_union_prelim}]
  Observe
that
\begin{align*}
  \En_{\pi\sim{}\pnot}\brk*{\fmbar(\pimbar)-\fmbar(\pi)}
  &\leq{}   \delta +
  \En_{\pi\sim{}\pnot}\brk*{\fmstar(\pimstar)-\fmbar(\pi)} \\
  &\leq{}  \delta +
    \En_{\pi\sim{}\pnot}\brk*{\fmstar(\pimstar)-\fmstar(\pi)} + \Delta,
\end{align*}
where the first inequality uses the fact that $\fmbar(\pimbar) \leq \fm(\pim) + \delta$ for all $M \in \MM$, and the second inequality uses \pref{lem:hellinger_to_value}. In
addition, the definition of $\pnot$ implies that
$\En_{\pi\sim{}\pnot}\brk*{\fmstar(\pimstar)-\fmstar(\pi)}\leq{}\decoreg(\cM,\Mbar)+\gamma\Delta^2$,
so we have
\begin{align*}
  \En_{\pi\sim{}\pnot}\brk*{\fmbar(\pimbar)-\fmbar(\pi)}
  &\leq{}   \delta +
    \decoreg(\cM,\Mbar)+\gamma\Delta^2 + \Delta\notag\\
    &\leq{}   \delta +
      \decoreg(\cM,\Mbar)+2\gamma\Delta^2 + (4\gamma)^{-1}.
\end{align*}  
\end{proof}

\begin{proof}[\pfref{prop:constrained_offset_pac}]
We first prove the inequality \pref{eq:constrained_offset_pac1}. Let
$\veps>0$ and $\Mbar\in\cMall$ be fixed. Using the method of Lagrange
multipliers, we have
\begin{align*}
  \deccpac(\cM,\Mbar) &= \inf_{p,q\in\Delta(\Pi)}\sup_{M\in\cM}\crl*{
  \En_{\pi\sim{}p}\brk*{\gm(\pi)}\mid{}\Enq\brk*{\Dhels{M(\pi)}{\Mbar(\pi)}}\leq\veps^2
                        } \\
  &= \inf_{p,q\in\Delta(\Pi)}\sup_{M\in\cM}\max\crl*{\inf_{\gamma\geq{}0}\crl*{
  \En_{\pi\sim{}p}\brk*{\gm(\pi)} - \gamma\prn*{\Enq\brk*{\Dhels{M(\pi)}{\Mbar(\pi)}}-\veps^2}},0
    } \\
      &\leq \inf_{\gamma\geq{}0}\inf_{p,q\in\Delta(\Pi)}\sup_{M\in\cM}\max\crl*{
  \En_{\pi\sim{}p}\brk*{\gm(\pi)} - \gamma\prn*{\Enq\brk*{\Dhels{M(\pi)}{\Mbar(\pi)}}-\veps^2},0
        }\\
        &\leq \inf_{\gamma\geq{}0}\crl*{
          \decopac(\cM,\Mbar)\vz + \gamma\veps^2
  }.
\end{align*}

We now prove the inequality \pref{eq:constrained_offset_pac2}. Let
$\gamma\geq{}1$ and $\Mbar\in\cMall$ be fixed. For integers $i \geq 0$, define $\vep_i = 2^{-i/2}$. For each $i \geq 0$, let $(p_i, q_i)$ denote a minimizer to the following expression:
  \begin{align}
\inf_{p_i, q_i\in \Delta(\Pi)} \sup_{M \in \MH_{q_i, \vep_i}(\Mbar)} \E_{\pi \sim p_i}[g^M(\pi)] = \deccpac[\vep_i](\MM, \Mbar)\nonumber.
  \end{align}
  Recalling that $L = 2\lceil  \log 2\gamma \rceil$, set \[
    q = \frac 12 \cdot \indic_{\pimbar} + \frac{q_0 + \cdots +
      q_{L-1}}{4L} + \frac{p_0 + \cdots + p_{L-1}}{4L},\mathand p =
    \indic_{\pimbar}.\]
  Consider any $M \in \MM$. We will upper bound the value
  \begin{align*}
    \Enp\brk*{\fm(\pim) - \fm(\pi)} - 4\gamma{}L\cdot\Enq\brk*{\Dhels{M(\pi)}{\Mbar(\pi)}}.
  \end{align*}
  Choose $j\in\crl{0,\ldots,L-1}$ as large as possible so that
  \begin{align}
\E_{\pi \sim q} \left[ \hell{M(\pi)}{\Mbar(\pi)} \right] \leq \frac{ \vep_j^2}{4L}\label{eq:constrained_offset_j}.
  \end{align}
  If such an index $j$ does not exist, we must have $\E_{\pi \sim q} \left[
    \hell{M(\pi)}{\Mbar(\pi)} \right] > 1/(4L)$. In this case, we have
  \begin{align}
    \E_{\pi \sim p} [\fm(\pim) - \fm(\pi)] - 4\gamma{}L \cdot \E_{\pi \sim q} \left[ \hell{M(\pi)}{\Mbar(\pi)} \right] \leq 1 - \gamma \leq 0\leq\deccpac[0](\cM,\Mbar).\nonumber
  \end{align}
  Suppose going forward that $0 \leq j \leq L-1$ satisfying
  \pref{eq:constrained_offset_j} exists. If $j < L-1$, since
   we chose the largest possible value of $j$, we have $\E_{\pi \sim q} \left[
    \hell{M(\pi)}{\Mbar(\pi)}\right] \geq \frac{\vep_j^2}{8L}$. In addition, regardless of the value of $j \in \{0, 1, \ldots, L-1\}$, by the definition of $q$, we have
  \begin{align}
    \E_{\pi \sim q_j} \left[ \hell{M(\pi)}{\Mbar(\pi)}\right] & \leq  4L \cdot \E_{\pi \sim q} \left[ \hell{M(\pi)}{\Mbar(\pi)}\right] \leq  \vep_j^2,\nonumber\intertext{and}
    \E_{\pi \sim p_j} \left[ \hell{M(\pi)}{\Mbar(\pi)}\right] & \leq  4L \cdot \E_{\pi \sim q} \left[ \hell{M(\pi)}{\Mbar(\pi)}\right] \leq  \vep_j^2\nonumber,
  \end{align}
 that is, $M \in \MH_{p_j, \vep_j}(\Mbar) \cap \MH_{q_j,
   \vep_j}(\Mbar)$. It follows that
 \begin{align}
    \fm(\pim) - \fmbar(\pimbar) & \leq \E_{\pi \sim p_j} \left[ \fm(\pim) - \fmbar(\pi) \right] \nonumber\\
    & \leq \E_{\pi \sim p_j} \left[ \fm(\pim) - \fm(\pi) \right] + \vep_j\nonumber\\
    & \leq \deccpac[\vep_j](\MM, \Mbar) + \vep_j\nonumber,
  \end{align}
  where the second inequality uses that $M \in \MH_{p_j,
    \vep_j}(\Mbar)$ and the final inequality uses that $M \in
  \MH_{q_j, \vep_j}(\Mbar)$. As a result, we can compute
  \begin{align}
&    \E_{\pi \sim p} \left[ \fm(\pim) - \fm(\pi) \right] - 4\gamma L \cdot \E_{\pi \sim q} \left[ \hell{M(\pi)}{\Mbar(\pi)} \right] \nonumber\\
    & \leq  \fm(\pim) - \fm(\pimbar) - \frac{4\gamma L}{8L} \cdot \vep_j^2 \cdot \One{j < L-1} \nonumber\\
    & \leq \fmbar(\pimbar) - \fm(\pimbar) + \deccpac[\vep_j](\MM, \Mbar) + \vep_j - \frac{\gamma}{2} \vep_j^2 \cdot \One{j < L-1} \nonumber\\
    & \leq \frac{1}{2\gamma} + \frac{\gamma}{2} \cdot \left( \fmbar(\pimbar) - \fm(\pimbar) \right)^2 +  \deccpac[\vep_j](\MM, \Mbar) + \vep_j -\frac{\gamma}{2} \vep_j^2 \cdot \One{j < L-1}\nonumber\\
    & \leq \frac{1}{2\gamma} + \gamma \cdot \E_{\pi \sim q} \left[ \hell{M(\pi)}{\Mbar(\pi)} \right] + \deccpac[\vep_j](\MM, \Mbar) + \vep_j -\frac{\gamma}{2}  \vep_j^2 \cdot \One{j < L-1}\nonumber\\
    & \leq \frac{1}{\gamma} +  \gamma \cdot \E_{\pi \sim q} \left[ \hell{M(\pi)}{\Mbar(\pi)} \right] + \deccpac[\vep_j](\MM, \Mbar) + \vep_j -\frac{\gamma}{2}  \vep_j^2 \nonumber,
  \end{align}
  where the final inequality uses that $\vep_{L-1}^2 \leq 1/\gamma^2$ since $L \geq 2\log(2\gamma)$. Rearranging, we obtain
  \begin{align}
    &\E_{\pi \sim p}[\fm(\pim) - \fm(\pi)] - \gamma \cdot (4L+1) \cdot
    \E_{\pi \sim q} [\hell{M(\pi)}{\Mbar(\pi)}] \\
    & \leq \frac{1}{\gamma} + \deccpac[\vep_j](\MM, \Mbar) + \vep_j -
      \frac{\gamma}{2} \cdot \vep_j^2\nonumber\\
      & \leq \frac{2}{\gamma} + \deccpac[\vep_j](\MM, \Mbar) - \frac{\gamma}{4} \cdot \vep_j^2\nonumber,
  \end{align}
  as desired.
\end{proof}

\begin{proof}[\pfref{prop:constrained_offset_basic}]
  Let
$\gamma>0$ and $\Mbar\in\cMall$ be given. Fix $\veps>0$ to be chosen later, and let $p$ be the minimizer for
$\deccreg(\cM,\Mbar)$. Consider the value of the offset DEC for $M\in\cM$:
\begin{align*}
\En_{\pi\sim{}p}\brk*{\fm(\pim)-\fm(\pi)-\gamma\cdot\Dhels{M(\pi)}{\Mbar(\pi)}}.
\end{align*}
We consider two cases. First, if $M\in\cH_{p,\veps}(\Mbar)$, it is
immediate that
\begin{align*}
  \En_{\pi\sim{}p}\brk*{\fm(\pim)-\fm(\pi)-\gamma\cdot\Dhels{M(\pi)}{\Mbar(\pi)}}
  \leq{} \En_{\pi\sim{}p}\brk*{\fm(\pim)-\fm(\pi)} \leq \deccreg(\cM,\Mbar).
\end{align*}
On the other hand if $M\notin\cH_{p,\veps}(\Mbar)$, we have
$\En_{\pi\sim{}p}\brk*{\Dhels{M(\pi)}{\Mbar(\pi)}}\geq\veps^2$, and since
$\gm\leq{}1$, we have
\begin{align*}
  \En_{\pi\sim{}p}\brk*{\fm(\pim)-\fm(\pi)-\gamma\cdot\Dhels{M(\pi)}{\Mbar(\pi)}}
  \leq{} 1-\gamma\veps^2 \leq{} 0
\end{align*}
by choosing $\veps=\gamma^{-1/2}$. We conclude that
\[
\decoreg(\cM,\Mbar) \leq{} \deccreg[\gamma^{-1/2}](\cM,\Mbar).
\]
\end{proof}

\begin{proof}[\pfref{prop:constrained_offset_counterexample1}]%
  Recall the model classes $\MM^{\alpha,\beta}$ defined in \cref{ex:upper_improvement}, parametrized by $\alpha \in (0,1/2], \beta \in (0,1), A \in \BN$. %
  Consider any choice for $\alpha$, $\beta$, and $A$; we will specify
  concrete values below.
  \cref{lem:ex_ub} gives that for all $\vep > 0$,
  \begin{align}
\deccreg[\vep](\MM^{\alpha, \beta}) = \sup_{\Mbar \in \co(\MM^{\alpha,\beta})} \deccreg[\vep](\MM^{\alpha, \beta} \cup \{ \Mbar \}, \Mbar) \leq O \left( \frac{\vep^2}{\beta} \right).\nonumber
  \end{align}
  On the other hand, \cref{lem:ex_lb} gives that for the choice of $\Mbar = \til M \in \MM^{\alpha, \beta}$, we have, for all $\gamma > 0$,
  \begin{align}
\decoreg[\gamma](\MM^{\alpha, \beta}, \til M) \geq \frac{\alpha}{2 + 8 \gamma \beta} - 4\gamma / A\nonumber.
  \end{align}
Given $\gamma > 0$,  let us choose $\alpha = 1/2,\ \beta = 1/\sqrt{\gamma}$, and $A = 256 \gamma^2 / \beta$. Then the resulting model class $\MM = \MM^{\alpha, \beta}$ satisfies $\deccreg(\MM) \leq O(\vep^2 \gamma^{1/2})$ for all $\vep > 0$, yet %
  \begin{align}
\sup_{\Mbar \in \MM} \decoreg[\gamma](\MM, \Mbar) \geq \frac{1}{4 + 16 \gamma \beta} - \frac{4\gamma}{A} \geq \frac{1}{32} \cdot \left( \frac{1}{\gamma^{1/2}} \wedge 1 \right) - \frac{4\gamma}{A} \geq \frac{1}{64} \cdot \left( \frac{1}{\gamma^{1/2}} \wedge 1 \right) \geq \Omega(\gamma^{-1/2})\nonumber.
  \end{align}
  \end{proof}

\begin{proof}[\pfref{prop:constrained_implies_localization_pac}]
  Recall the definition of $\deccpacalt(\MM, \Mbar)$ in \cref{eq:define-deccpacalt}. 
Let $\veps>0$ and $\Mbar\in\cMall$
be given, and let $(p',q')$ be minimizers for $\deccpacalt(\cM,\Mbar)$. Then for
all $M\in\cH_{p',\veps}(\Mbar)\cap\cH_{q',\veps}(\Mbar)$, we have
  \begin{align*}
  \deccpacalt(\cM,\Mbar) &\geq \En_{\pi\sim{}p'}\brk*{\fm(\pim)-\fm(\pi)} \\
                      &\geq \En_{\pi\sim{}p'}\brk*{\fm(\pim)-\fmbar(\pi)} -\veps\\
                      &\geq \fm(\pim)-\fmbar(\pimbar) -\veps,
  \end{align*}
  where the second inequality uses \pref{lem:hellinger_to_value}. That is, if we define $\alpha=\veps+\deccpacalt(\cM,\Mbar)$, we have
  $\cH_{p',\veps}(\Mbar)\cap\cH_{q',\veps}(\Mbar)\subseteq\cMloc(\Mbar)$. Now, let $(p,q)$ be
  minimizers for $\deccpac(\cMloc(\Mbar),\Mbar)$, and set
  $\qbar=\frac{1}{3}q+\frac{1}{3}q' + \frac{1}{3}p'$. We have
  \begin{align*}
    \deccpac[\veps/\sqrt{3}](\cM,\Mbar)
    &\leq{}
      \sup_{M\in\cH_{\qbar,\veps/\sqrt{3}}(\Mbar)}\En_{\pi\sim{}p}\brk*{\fm(\pim)-\fm(\pi)}
    \\
    &\leq{}
      \sup_{M\in\cH_{q,\veps}(\Mbar)\cap\cH_{p',\veps}(\Mbar) \cap\cH_{q',\veps}(\Mbar)}\En_{\pi\sim{}p}\brk*{\fm(\pim)-\fm(\pi)} \\
    &\leq{}
      \sup_{M\in\cH_{q,\veps}(\Mbar)\cap\cMloc(\Mbar)}\En_{\pi\sim{}p}\brk*{\fm(\pim)-\fm(\pi)}
    \\
    &=\deccpac(\cMloc(\Mbar),\Mbar).
  \end{align*}
  Finally, using \pref{lem:pac_constrained_alt}, we have
  $\alpha\leq{}\veps + \deccpac[\sqrt{2}\veps](\cM,\Mbar)$.
  
\end{proof}

\begin{proof}[\pfref{prop:constrained_implies_localization_regret}] We first prove the following result, which does not requite any regularity condition.

\begin{lemma}
  \label{lem:constrained-localization-core}
  Fix any $\Mbar \in \cMall$ and $\vep > 0$, and let $\alpha' := \vep + \deccreg[\vep](\MM, \Mbar)$. Then for any constant $\Creg \geq \sqrt 2$, it holds that 
  \begin{align*}
 \deccreg[\veps/\Creg](\cM,\Mbar)
\leq  \frac{1}{\Creg^2}\deccreg(\cM,\Mbar)
    + \deccreg(\cMloc[\alpha'](\Mbar),\Mbar).
    \end{align*}
\end{lemma}
  \begin{proof}[\pfref{lem:constrained-localization-core}]
  Given $\veps>0$ and
$\Mbar$, let $p$ be the distribution that achieves the
value for $\deccreg(\cM,\Mbar)$. Then for all
$M\in\cH_{p,\veps}(\Mbar)$, we have
\begin{align*}
  \deccreg(\cM,\Mbar) &\geq \Enp\brk*{\fm(\pim)-\fm(\pi)} \\
                      &\geq \Enp\brk*{\fm(\pim)-\fmbar(\pi)} -\veps\\
                      &\geq \fm(\pim)-\fmbar(\pimbar) -\veps,
\end{align*}
where the second inequality uses \pref{lem:hellinger_to_value}. Hence,
for $\alpha'\ldef{}\veps+\deccreg(\cM,\Mbar)$, we have
$\cH_{p,\veps}(\Mbar)\subseteq\cMloc[\alpha'](\Mbar)$.

Now, let $p'$ be the minimizer for
$\deccreg(\cMloc[\alpha'](\Mbar),\Mbar)$. Set
$\pbar=\frac{1}{\Creg^2}p+\prn*{1-\frac{1}{\Creg^2}}p'$. Using that
\[\frac{1}{\Creg^2}\cdot\prn*{1-\frac{1}{\Creg^2}}^{-1}\leq{}1\] whenever
$\Creg\geq\sqrt{2}$, we have
\begin{align*}
  \deccreg[\veps/\Creg](\cM,\Mbar)
  &\leq{}\sup_{M\in\cH_{\pbar,\veps/\Creg}(\Mbar)}\En_{\pi\sim{}\pbar}\brk*{\fm(\pim)-\fm(\pi)}
  \\
  &\leq{}\sup_{M\in\cH_{p,\veps}(\Mbar)\cap\cH_{p',\veps}(\Mbar)}\En_{\pi\sim{}\pbar}\brk*{\fm(\pim)-\fm(\pi)} \\
  &\leq{}\frac{1}{\Creg^2}\sup_{M\in\cH_{p,\veps}(\Mbar)}\En_{\pi\sim{}p}\brk*{\fm(\pim)-\fm(\pi)}
    +
\sup_{M\in\cH_{p,\veps}(\Mbar)\cap\cH_{p',\veps}(\Mbar)}\En_{\pi\sim{}p'}\brk*{\fm(\pim)-\fm(\pi)}\\
  &=\frac{1}{\Creg^2}\deccreg(\cM,\Mbar)
    +
    \sup_{M\in\cH_{p,\veps}(\Mbar)\cap\cH_{p',\veps}(\Mbar)}\En_{\pi\sim{}p'}\brk*{\fm(\pim)-\fm(\pi)}\\
  &\leq{}\frac{1}{\Creg^2}\deccreg(\cM,\Mbar)
    +
    \sup_{M\in\cMloc[\alpha'](\Mbar)\cap\cH_{p',\veps}(\Mbar)}\En_{\pi\sim{}p'}\brk*{\fm(\pim)-\fm(\pi)}\\
    &=\frac{1}{\Creg^2}\deccreg(\cM,\Mbar)
      + \deccreg(\cMloc[\alpha'](\Mbar),\Mbar).
\end{align*}
\end{proof}

We now complete the proof of \cref{prop:constrained_implies_localization_regret}. Under the assumed growth condition, we have $\deccreg[\vep/\Creg](\MM, \Mbar) \geq \frac{1}{\creg^2} \deccreg[\vep](\MM, \Mbar)$, so rearranging the result of \cref{lem:constrained-localization-core} (with $\alpha' = \alpha(\vep)$) yields
\[
\deccreg[\veps/\Creg](\cM,\Mbar)\leq{}\prn*{\frac{1}{\creg^2}-\frac{1}{\Creg^2}}^{-1}\cdot \deccreg(\cMloc[\alpha(\vep)](\Mbar),\Mbar).
\]
The result in the proposition follows by rescaling $\vep$ to $\vep \cdot \Creg$. 
\end{proof}

For use later on, we also prove the following variant of
                                \cref{prop:constrained_implies_localization_regret},
                                which concerns the \CompShort for the
                                class $\cM\cup\crl{\Mbar}$.
\begin{proposition}[Localization for regret \CompShort; global version]
  \label{prop:constrained-localization-allm}
  Consider any set $\MM' \subseteq \cMall$, 
and assume that the strong regularity condition
  \cref{eq:regret_localization_growth_global} is satisfied relative to $\MM'$. Then, for all $\veps>0$, letting $\alphaveps\ldef\Creg\cdot\veps+\sup_{\Mbar \in \MM'}\deccreg[\Creg\cdot\veps](\cM\cup\{\Mbar\},\Mbar)\leq\Creg^2\cdot\prn*{\veps+\sup_{\Mbar\in\cM'}\deccreg[\veps](\cM\cup\{\Mbar\},\Mbar)}$, we have
\begin{align*}
\sup_{\Mbar \in \MM'} \deccreg(\cM\cup \{ \Mbar \},\Mbar)\leq{}\Cloc\cdot \sup_{\Mbar \in \MM'} \deccreg[\Creg\cdot\veps](\cMloc[\alphaveps](\Mbar) \cup \{ \Mbar \},\Mbar),
\end{align*}
where $\Cloc\ldef{}\prn*{\frac{1}{\creg^2} - \frac{1}{\Creg^2}}^{-1}$.
\end{proposition}

\begin{proof}[\pfref{prop:constrained-localization-allm}]
  Define $\alpha := \vep + \sup_{\Mbar \in \MM'} \deccreg[\vep](\MM \cup \{ \Mbar \} , \Mbar)$. 
  Applying \cref{lem:constrained-localization-core} to the class $\MM \cup \{ \Mbar \}$ for each choice of $\Mbar \in \MM'$, we obtain that
  \begin{equation}
    \label{eq:mprime-selfbounding}
\sup_{\Mbar\in\cM'}  \deccreg[\veps/\Creg](\cM\cup\{\Mbar\},\Mbar)
\leq  \frac{1}{\Creg^2}\sup_{\Mbar\in\cM'}\deccreg(\cM\cup\{\Mbar\},\Mbar)
    + \sup_{\Mbar\in\cM'}\deccreg(\cMloc(\Mbar)\cup\{\Mbar\},\Mbar).
  \end{equation}

The growth condition \cref{eq:regret_localization_growth_global} gives that
\[
  \sup_{\Mbar\in\cM'}\deccreg[\veps/\Creg](\cM\cup \{ \Mbar \},\Mbar)\geq\frac{1}{\creg^2}\sup_{\Mbar\in\cM'}\deccreg(\cM \cup \{ \Mbar \},\Mbar).
\]
Then rearranging \cref{eq:mprime-selfbounding} yields
\begin{equation}
  \label{eq:constrained_implies_localization_general}
\sup_{\Mbar\in\cM'}\deccreg[\veps/\Creg](\cM\cup\{\Mbar\},\Mbar)\leq{}\prn*{\frac{1}{\creg^2}-\frac{1}{\Creg^2}}^{-1}\cdot\sup_{\Mbar\in\cM'}\deccreg(\cMloc(\Mbar) \cup \{ \Mbar \},\Mbar).
\end{equation}
The result in the proposition statement follows by %
replacing $\vep$ with $\vep \cdot \Creg$. 
\end{proof}

\begin{proof}[\pfref{prop:constrained_offset_localized}]
  Fix $\Mbar \in \cMall$ and $\veps>0$, and let $p \in \Delta(\Pi)$
  be a minimizer for $\deccreg[\veps](\cMu,\Mbar)$.

  Fix any $M \in \MM_\alpha(\Mbar)\cup\crl{\Mbar}$. We bound the regret under
  $p$ by considering two cases.

\noindent\emph{Case 1.} If $\E_{\pi \sim p} \left[ \hell{M(\pi)}{\Mbar(\pi)} \right] \leq \vep^2$, then $M \in \MH_{p, \vep}(\Mbar)$, and it follows that $\E_{\pi \sim p} \left[ \fm(\pim) - \fm(\pi) \right] \leq \deccreg[\vep](\cMu, \Mbar)$.

\noindent\emph{Case 2.} For the second case, suppose that $\E_{\pi \sim p} \left[ \hell{M(\pi)}{\Mbar(\pi)} \right] > \vep^2$.  We now compute
  \begin{align}
    \E_{\pi \sim p} \left[ \fm(\pim) - \fm(\pi) \right] & \leq \alpha + \E_{\pi \sim p} \left[ \fmbar(\pimbar) - \fm(\pi) \right]\nonumber\\
    & \leq \alpha + \E_{\pi \sim p} \left[ \fmbar(\pimbar) - \fmbar(\pi) \right] + \frac{1}{2\gamma} + \frac{\gamma}{2} \cdot \E_{\pi \sim p} \left[ (\fm(\pi) - \fmbar(\pi))^2 \right]\nonumber\\
    & \leq \alpha + \deccreg[\vep](\cMu, \Mbar) + \frac{1}{2\gamma} + \frac{\gamma}{2} \cdot \E_{\pi \sim p} \left[ \hell{M(\pi)}{\Mbar(\pi)} \right]\nonumber,
  \end{align}
  where the second inequality uses Young's inequality.
  Rearranging, we obtain
  \begin{align}
    \E_{\pi \sim p} \left[ \fm(\pim) - \fm(\pi) - \gamma \cdot \hell{M(\pi)}{\Mbar(\pi)} \right] & \leq \alpha + \deccreg[\vep](\cMu, \Mbar) + \frac{1}{2\gamma} - \frac{\gamma}{2} \cdot \E_{\pi \sim p} \left[ \hell{M(\pi)}{\Mbar(\pi)}\right]\nonumber\\
    & \leq \alpha +  \deccreg[\vep](\cMu, \Mbar) +
      \frac{1}{2\gamma} - \frac{\gamma \vep^2}{2}.
  \end{align}
  Recalling that $M$ can be any model in $\MM_\alpha(\Mbar)\cup\crl{\Mbar}$, we obtain
  \begin{align}
    \decoreg(\cMloc(\Mbar)\cup\crl{\Mbar}, \Mbar) & \leq \max \left\{
                                     \deccreg[\vep](\cMu, \Mbar),\
                                     \alpha + \frac{1}{2\gamma} +
                                     \deccreg[\vep](\cMu, \Mbar) -
                                     \frac{\gamma \vep^2}{2}
                                     \right\}\nonumber \\
    & = \deccreg[\vep](\cMu, \Mbar) + \max \left\{0,\ \alpha + \frac{1}{2\gamma}  - \frac{\gamma \vep^2}{2} \right\}.\nonumber
  \end{align}
\end{proof}

\begin{proof}[\pfref{prop:equivalence}]
  We begin with the upper bound on the constrained \CompShort. Let $\veps>0$ be fixed. Using
  \pref{prop:constrained_implies_localization_regret}, we have
  \begin{align*}
    \deccreg(\cMu,\Mbar)
    \leq     \Cloc\cdot\deccreg[\Creg\veps](\cMloc(\Mbar)\cup\crl{\Mbar},\Mbar),
  \end{align*}
  where $\alpha=\Creg\veps +
  \deccreg[\Creg\veps](\cM\cup\crl{\Mbar},\Mbar)\leq{}\Creg\veps +
  \creg^2\deccreg[\veps](\cM\cup\crl{\Mbar},\Mbar)
  \leq{}\Creg^2\prn*{
    \veps + \deccreg[\veps](\cM,\Mbar)}$ (see \pref{def:growth}). Next, for all $\gamma>0$, using \pref{prop:constrained_offset_union}, we have
  \[
    \deccreg[\Creg\veps](\cMloc(\Mbar)\cup\crl{\Mbar},\Mbar)
    \leq 8\cdot\prn*{\decoreg(\cMloc(\Mbar),\Mbar)\vz +
      \Creg^2\gamma\veps^2} + 7\Creg\veps,
  \]
  so that
  \[
    \deccreg[\veps](\cMu,\Mbar)
    \leq 8\Cloc\cdot\prn*{\decoreg(\cMloc(\Mbar),\Mbar)\vz +
      \Creg^2\gamma\veps^2} + 7\Cloc\Creg\veps.
  \]
  We now set
  \begin{align*}
    \gammastar = (16\Creg^2\Cloc)^{-1}\cdot{}\frac{\veps + \deccreg(\cMu,\Mbar)}{\veps^2},
  \end{align*}
  which satisfies $\gammastar \geq \frac{1}{16\Creg^2\Cloc \cdot \vep}$ and gives
  \begin{align*}
    \deccreg[\veps](\cMu,\Mbar)
    \leq 8\Cloc\cdot\decoreg[\gammastar](\cMloc(\Mbar),\Mbar)\vz +
    \frac{1}{2}\cdot\deccreg(\cM\cup\Mbar,\Mbar)+ (7\Cloc\Creg+1/2)\veps,
  \end{align*}
  or after rearranging,
    \begin{align*}
    \deccreg[\veps](\cMu,\Mbar)
    \leq 16\Cloc\cdot\decoreg[\gammastar](\cMloc(\Mbar),\Mbar)\vz 
      + 2(7\Cloc\Creg+1/2)\veps.
    \end{align*}
    In addition, we have
    \begin{align*}
      \alpha \leq \Creg^2 (16\Creg^2\Cloc)\cdot\gammastar\veps^2=\Creg^2 (16\Creg^2\Cloc)\cdot\alpha(\veps,\gammastar).
    \end{align*}
    To conclude, we over-bound by maximizing over $\gammastar\geq \frac{1}{16\Creg^2 \Cloc \cdot \vep}$.
    
  For the lower bound on the constrained \CompShort, it is an immediate consequence of
  \pref{prop:constrained_offset_localized} that for all $\veps>0$ and
  $\gamma>\sqrt{2}\cdot\veps^{-1}$, letting
  $\alpha=\frac{\gamma\veps^2}{4}$, 
\begin{align*}
  \decoreg(\cMloc(\Mbar),\Mbar) \leq \deccreg(\cMu,\Mbar) +
  \max\crl*{\alpha + \frac{1}{2\gamma}-\frac{\gamma\veps^2}{2},0} = \deccreg(\cMu,\Mbar),
\end{align*}
for all $\Mbar\in\cMall$. Since we are free to maximize over $\gamma\geq\sqrt{2}\veps^{-1}$, this establishes the result.
  
\end{proof}

\noah{for a later version: would be cleaner just to slightly modify
  below example so it coincides with the example of Example 5.1 here,
  i.e., should be able to make the observing arm revealing with
  probability $\alpha$ (though we actually have to do new work here
  since we're claiming something different about it)}
\dfcomment{agreed -- happy to punt for now}
\begin{proof}[\pfref{prop:regret_union_counterexample}]
  Consider the following model class $\MM$, parametrized by $\alpha \in (0,1/2)$:
  \begin{enumerate}
  \item $\Pi = \bbN \cup \{\pir \}$.
  \item We have $\MM=\crl*{M_a}_{a\in\bbN}$. For each $a \in \bbN$,
    the model $M_a \in \MM$ has rewards and observations defined as follows:
    \begin{enumerate}
    \item For $\pi \in \bbN$, $f\sups{M_a}(\pi) = \frac 12 + \alpha \cdot \One{\pi = a}$, while $f\sups{M_a}(\pir) = 0$.
    \item For all $\pi \in \Pi$, we have $r = f\sups{M_a}(\pi)$ almost surely under $r \sim M_a(\pi)$.
    \item For $\pi \in \bbN$, we receive the observation $o = \perp$.
    \item Selecting $\pir$ gives the observation $o \in \{0,1\}^\bbN$, where for each $i \in \bbN$, $o_i \sim \Ber(1/2 + \alpha \cdot \One{a=i})$ is drawn independently (thus, we have $\MO = \{0,1\}^\bbN \cup \{ \perp \}$).
    \end{enumerate}
  \end{enumerate}

  \paragraphi{Upper bound}
  We will show that there are constants $c,C > 0$ so that, for $\vep > 0$,
  \begin{align}
\sup_{\Mbar \in \co(\MM)} \deccreg[\vep](\MM, \Mbar) \leq \alpha \cdot \One{\vep \geq \sqrt{c}\cdot \alpha} \leq C \cdot \vep\nonumber.
  \end{align}
Since $\sup_{\Mbar \in \co(\MM)} \deccreg(\MM, \Mbar) \leq \alpha$ for all $\vep \geq 0$ (as the choice of $p = \indic_{1}$ satisfies $\E_{\pi \sim p}[\fm(\pim) - \fm(\pi)] \leq \alpha$ for all $M \in \MM$), it suffices to show that for $\vep < \sqrt{c} \cdot \alpha$ and for any $\Mbar \in \co(\MM)$, we have $\deccreg(\MM, \Mbar) = 0$.
  
  Given $\Mbar \in \co(\MM)$ and $\vep \leq 1/2$, we can write $\Mbar(\pi) = \E_{M' \sim \nu}[M'(\pi)]$ for some $\nu \in \Delta(\MM)$. We define a distribution $p \in \Delta(\Pi)$ according to the following cases:
  \begin{itemize}
  \item If $\nu$ puts mass at most $2/3$ on each model $M \in \MM$, we define $p = \indic_{\pir}$.
  \item Otherwise, there is a unique choice for $a^\st \in [A]$ so that $\nu(M_{a^\st}) \geq 2/3$, and in this case, we define $p = \indic_{a^\st}$. 
  \end{itemize}
  Now consider any model $M_a \in \MM$. Consider the first case above,
  and write $o \sim M_a(\pir)$ and $\wb{o} \sim \Mbar(\pir)$. Note
  that $o_a \sim \Ber(1/2 + \alpha)$, while $\wb{o}_a \sim \Ber(1/2 +
  \beta)$ for some $\beta \leq \frac{2}{3} \cdot \alpha$. It follows that
  \begin{align}
\E_{\pi \sim p} \left[ \hell{M_a(\pi)}{\Mbar(\pi)} \right] = \hell{M_a(\pir)}{\Mbar(\pir)} \geq \hell{\Ber(1/2 + \alpha)}{\Ber(1/2 + \beta)} \geq c \cdot \alpha^2\nonumber,
  \end{align}
  for a numerical constant $c > 0$. Since $c\alpha^2 > \vep^2$, it follows that $M_a \not \in \MH_{p,\vep}(\Mbar)$; since the choice of $M_a$ is arbitrary, we conclude that $\MH_{p,\vep}(\Mbar)=\emptyset$ in this case.

  Now, consider the second case above. For $a = a^\st$, we have that
$ 
\E_{\pi \sim p} \left[ f\sups{M_a}(\pi\subs{M_a}) - f\sups{M_a}(\pi) \right] = 0.
$
For $a\neq\astar$, we have that
$\bbP_{r\sim{}M_a(\astar)}(r\neq{}1/2)=0$, while
$\bbP_{r\sim{}\Mbar(\astar)}(r\neq{}1/2)\geq{}2/3$. As a result,
\begin{align}
\E_{\pi \sim p} \left[ \hell{M_a(\pi)}{\Mbar(\pi)} \right] = \hell{M_a(a^\st)}{\Mbar(a^\st)} \geq \hell{\Ber(0)}{\Ber(2/3)} \geq 4/9 > \vep^2,
\end{align}
meaning that $M_a \not \in \MH_{p,\vep}(\Mbar)$.

\paragraphi{Lower bound}
Pick $A\geq{}2$, and let $\Mbar=\Unif(\crl*{M_a}_{a\in\brk{A}})$. Given $p\in\Delta(\Pi)$, let $a=\argmin_{a\in\brk{A}}p(a)$, so that
$p(a)\leq{}1/A$. We observe that
\begin{align*}
  \En_{\pi\sim{}p}\brk*{f\sups{M_a}(\pi\subs{M_a}) - f\sups{M_a}(\pi)}
  \geq{} \alpha(1-1/A) \geq{} \alpha/2
\end{align*}
and
\begin{align*}
  \En_{\pi\sim{}p}\brk*{\Dhels{M_a(\pi)}{\Mbar(\pi)}}
  \leq{} p(\picirc)\cdot\Dhels{M_a(\picirc)}{\Mbar(\picirc)} + 2p(a) + \sum_{i\in\brk{A},i\neq{}a}p(i) \Dhels{M_a(i)}{\Mbar(i)}.
\end{align*}
For all $i\neq{}a$, we have
$\Dhels{M_a(i)}{\Mbar(i)}\leq{}\Dhels{\Ber(0)}{\Ber(1/A)}\leq{}2/A$,
so that $\sum_{i\in\brk{A},i\neq{}a}p(i)
\Dhels{M_a(i)}{\Mbar(i)}\leq{}2/A$. As long as $\alpha$ is a
sufficiently small numerical constant, we also have, using the tensorization property of the squared Hellinger distance,
\begin{align*}
  \Dhels{M_a(\picirc)}{\Mbar(\picirc)}
  &\leq{} \Dhels{\Ber(1/2+\alpha)}{\Ber(1/2+\alpha/A)}
  + (A-1)\cdot \Dhels{\Ber(1/2)}{\Ber(1/2+\alpha/A)}\\
  &\leq{} c\cdot\prn*{ \alpha^2 + (A-1)\cdot\frac{\alpha^2}{A^2}}\\
  &\leq C\cdot\alpha^2,
\end{align*}
where $C,c>0$ are numerical constants. Altogether, this gives
\begin{align*}
  \En_{\pi\sim{}p}\brk*{\Dhels{M_a(\pi)}{\Mbar(\pi)}}
  \leq{} p(\picirc)\cdot{}C\alpha^2 + 4/A.
\end{align*}
We choose $A$ large enough such that $4/A\leq{}\veps^2/2$. There are
now two cases to consider.
\begin{itemize}
\item If $p(\picirc)\leq \frac{\veps^2}{2C\alpha^2}$, then
  $M_a\in\cH_{p,\veps}(\Mbar)$, and
  \[
    \En_{\pi\sim{}p}\brk*{f\sups{M_a}(\pi\subs{M_a}) -
      f\sups{M_a}(\pi)} \geq \frac{\alpha}{2}.
  \]
\item If this is not the case, we have
  \begin{align*}
    \En_{\pi\sim{}p}\brk*{\fmbar(\pimbar) - \fmbar(\pi)}
    \geq{} \frac{1}{2}p(\picirc) \geq{} 
    \frac{\veps^2}{4C\alpha^2} \wedge{} 1.
  \end{align*}
\end{itemize}
By combining these cases, we conclude that there are numerical
constants $C,c>0$ such that
\begin{align*}
  \deccreg(\cMu,\Mbar)
  \geq{} C\cdot{}\alpha\indic\crl[\big]{\veps>c\cdot{}\alpha^{3/2}}.
\end{align*}
In particular, choosing $\alpha\propto\veps^{2/3}$ gives
$\deccreg(\cMu,\Mbar)\geq\bigom(\veps^{2/3})$, while $\deccreg(\cM,\Mbar)\leq\bigoh(\veps)$.

\end{proof}

\begin{proof}[\pfref{prop:pac_union}]
By \pref{prop:greedy_equivalence}, we have that for all $\veps>0$ and $\Mbar\in\cMall$,
  \begin{align*}
    \deccpac(\cM\cup\crl{\Mbar},\Mbar)
    \leq     \deccpacg(\cM\cup\crl{\Mbar},\Mbar)
    =  \deccpacg(\cM,\Mbar)
    \leq{}     \deccpac[\sqrt{3}\veps](\cM,\Mbar) + 4\veps.
  \end{align*}
\end{proof}

\begin{proof}[\pfref{prop:conv_equivalence_regret}]
By \pref{ass:minimax}, we have
  \begin{align*}
      \decoreg(\cM,\Mbar)
  = \sup_{\mu\in\Delta(\cM)}\inf_{p\in\Delta(\Pi)}\En_{\pi\sim{}p,M\sim\mu}\brk*{\fm(\pim)-\fm(\pi)-\gamma\cdot\Dhels{M(\pi)}{\Mbar(\pi)}}.
  \end{align*}
  Observe that since Hellinger distance satisfies the triangle
  inequality, we have that for all $\pi\in\Pi$,
  \begin{align*}
    \En_{M,M'\sim\mu}\brk*{\Dhels{M(\pi)}{M'(\pi)}}
      &\leq{}     2\En_{M\sim\mu}\brk*{\Dhels{M(\pi)}{\Mbar(\pi)}}
    +    2\En_{M'\sim\mu}\brk*{\Dhels{M'(\pi)}{\Mbar(\pi)}} \\
    &= 4\En_{M\sim\mu}\brk*{\Dhels{M(\pi)}{\Mbar(\pi)}}.
  \end{align*}
  As a result, we have
  \begin{align*}
    \decoreg(\cM,\Mbar)
    &\leq
\sup_{\mu\in\Delta(\cM)}\inf_{p\in\Delta(\Pi)}\En_{\pi\sim{}p,M\sim\mu}\brk*{\fm(\pim)-\fm(\pi)-\frac{\gamma}{4}\cdot\En_{M'\sim\mu}\Dhels{M(\pi)}{M'(\pi)}}
    \\
    &\leq
      \sup_{\nu\in\Delta(\cM)}\sup_{\mu\in\Delta(\cM)}\inf_{p\in\Delta(\Pi)}\En_{\pi\sim{}p,M\sim\mu}\brk*{\fm(\pim)-\fm(\pi)-\frac{\gamma}{4}\cdot\En_{M'\sim\nu}\Dhels{M(\pi)}{M'(\pi)}} \\
          &\leq
\sup_{\nu\in\Delta(\cM)}\inf_{p\in\Delta(\Pi)}\sup_{M\in\cM}\En_{\pi\sim{}p}\brk*{\fm(\pim)-\fm(\pi)-\frac{\gamma}{4}\cdot\En_{M'\sim\nu}\Dhels{M(\pi)}{M'(\pi)}}
    \\
& = \sup_{\nu\in\Delta(\cM)}\decoregr[\gamma/4](\cM,\nu).    
  \end{align*}
  For the second inequality in \pref{eq:conv_equivalence_regret1}, it follows immediately
  from convexity of squared Hellinger distance that
  $\sup_{\nu\in\Delta(\cM)}\decoregr(\cM,\nu)\leq\sup_{\Mbar\in\conv(\cM)}\decoreg(\cM,\Mbar)$.

  We now prove \pref{eq:conv_equivalence_regret2}. Let $\veps>0$ be
  given. Since the strong regularity condition is satisfied relative to $\cMall$,
  \pref{prop:constrained-localization-allm} with $\MM' = \cMall$ implies that
  
  \begin{align*}
\sup_{\Mbar\in\cMall}\deccreg(\cMu,\Mbar)&\leq{}\Cloc\cdot\sup_{\Mbar\in\cMall}\deccreg[\Creg\cdot\veps](\cMloc[\alpha](\Mbar)\cup\crl{\Mbar},\Mbar), %
  \end{align*}
  where $\alpha\ldef{}\Creg^2\cdot\prn*{\veps +
  \sup_{\Mbar\in\cMall}\deccreg[\veps](\cMu,\Mbar)}$. Now, let
  \[
    \cMtil(\Mbar) = \crl*{M\in\cM\mid{} \fm(\pim)\leq\fmbar(\pimbar) +
      \alpha, \fmbar(\pimbar) \leq \fm(\pimbar) + \alpha}.
\]
Using
  \pref{lem:localization_one_sided}, along with the fact that
  $\alpha \geq \Creg^2 \cdot \vep\geq{}\sqrt{2} \Creg \cdot \vep$ since $\Creg \geq \sqrt 2$, we have that
  \begin{align*}
    \sup_{\Mbar\in\cMall}\deccreg[\Creg \cdot \veps](\cMloc[\alpha](\Mbar)\cup\crl{\Mbar},\Mbar)
    \leq{} \sup_{\Mbar\in\cMall}\deccreg[\sqrt{2}\Creg \cdot \veps](\cMtil[\alpha](\Mbar)\cmb,\Mbar) + \sqrt{2}\Creg \veps.
  \end{align*}
  Let $\Mtil$ be the model in $\cMall$ that attains the maximum in the
  \rhs above, and set $\cM'=\cMtil(\Mtil)$. Let $\gamma>0$ be fixed. Using \pref{prop:constrained_offset_union},
  we have
  \begin{align*}
    \deccreg[\sqrt{2}\Creg\veps](\cM'\cup\crl{\Mtil},\Mtil)
    \leq{}
    8\sup_{\Mbar\in\cMall}\decoreg(\cM',\Mbar)\vz
    + 16\Creg^2 \gamma\veps^2 + 7\sqrt{2}\Creg \veps.
  \end{align*}
  By \eqref{eq:conv_equivalence_regret1}, we have
  \begin{align*}
    \sup_{\Mbar\in\cMall}\decoreg(\cM',\Mbar)
    \leq{}     \sup_{\nu\in\Delta(\cM')}\decoregr[\gamma/4](\cM',\nu).
  \end{align*}
  Consider any $\nu\in\Delta(\cM')$ and let
  $\Mbarnu\ldef{}\En_{M'\sim\nu}\brk*{M'}\in\conv(\cM')\subseteq\conv(\cM)$. Observe that if
  $M\in\cM'=\cMtil(\Mtil)$, then
  \begin{align*}
    \fm(\pim) \leq f\sups{\Mtil}(\pi\subs{\Mtil}) + \alpha
    \leq \En_{M'\sim\nu}\brk*{\fmp(\pimtil)} + 2\alpha
    \leq \max_{\pi\in\Pi}\En_{M'\sim\nu}\brk*{\fmp(\pi)} + 2\alpha
    = \fmbarnu(\pimbarnu) + 2\alpha.
  \end{align*}
Hence, $\cM'\subseteq\cMloc[2\alpha](\Mbarnu)$, and we have the upper
bound
\begin{align*}
\sup_{\nu\in\Delta(\cM')}\decoregr[\gamma/4](\cM',\nu) \leq   \sup_{\nu\in\Delta(\cM)}\decoregr[\gamma/4](\cMloc[2\alpha](\Mbarnu),\nu).
\end{align*}
For any $\nu\in\Delta(\cM)$, \pref{prop:constrained_offset_localized_rand} implies that
\begin{align*}
  \decoregr[\gamma/4](\cMloc[2\alpha](\Mbarnu),\nu)
  \leq{} \deccregr[4\sqrt{\alpha/\gamma}](\cM\cup\crl{\Mbarnu},\nu) + \frac{2}{\gamma}.
\end{align*}
Putting everything together,
this establishes that for all $\gamma>0$,
\begin{align*}
  \sup_{\Mbar\in\cMall}\deccreg(\cMu,\Mbar)
  &\leq{}
    \Cloc\cdot\prn*{8\sup_{\nu\in\Delta(\cM)}\deccregr[4\sqrt{\alpha/\gamma}](\cM\cup\crl{\Mbarnu},\nu)
    +8\sqrt{2}\Creg\veps + 16\gamma\Creg^2\veps^2 + \frac{16}{\gamma}} \\
    &\leq{}
    \Cloc\cdot\prn*{8\sup_{\nu\in\Delta(\cM)}\deccregr[4\sqrt{\alpha/\gamma}](\cM\cup\crl{\Mbarnu},\nu)
      + 24\Creg^2 \gamma\veps^2 + \frac{24}{\gamma}}.
\end{align*}
We choose $\gamma=\frac{1}{24\Cloc\Creg^4}\cdot\frac{\alpha}{\veps^2}$. Since
$\veps^2/\alpha\leq\veps$, this gives
\begin{align*}
  \sup_{\Mbar\in\cMall}\deccreg(\cMu,\Mbar)
&\leq  c_1\cdot\sup_{\nu\in\Delta(\cM)}\deccregr[c_2\veps](\cM\cup\crl{\Mbarnu},\nu)
                                             +
                                             \frac{1}{2\Creg^2}\alpha
                                             + c_4\veps,\\
  &\leq  c_1\cdot\sup_{\nu\in\Delta(\cM)}\deccregr[c_2\veps](\cM\cup\crl{\Mbarnu},\nu)
                                             + \frac{1}{2}\sup_{\Mbar\in\cMall}\deccreg(\cMu,\Mbar) + c_3\veps,
\end{align*}
where $c_1,c_2,c_3,c_4>0$ are constants that depend only on $\Creg$ and
$\Cloc$. Rearranging yields the first inequality in
\eqref{eq:conv_equivalence_regret2}; the second inequality now follows
from Jensen's inequality.

\end{proof}

\section{Omitted Proofs from \creftitle{sec:related}}
\label{app:related}

  \begin{proof}[\pfref{prop:improvement}]
The lower bound is an immediate corollary of \pref{prop:equivalence},
so let us prove the upper bound. Let $\veps>0$ be fixed. Using
    \pref{prop:constrained_implies_localization_regret}, we have
    \begin{align*}
      \deccreg(\cM,\Mbar)
      \leq     \Cloc\cdot\deccreg[\Creg\veps](\cMloc(\Mbar),\Mbar),
    \end{align*}
    where $\alpha=\Creg^2\cdot\prn*{\veps +
      \deccreg[\veps](\cM,\Mbar)}$. Recall that we assume
    $\Creg,\Cloc=\bigoh(1)$. Hence, for all $\gamma>0$ be fixed,
    using \pref{prop:constrained_offset_union}, we have
    \[
      \deccreg[\Creg\veps](\cMloc(\Mbar),\Mbar) \leq
      \bigoh\prn*{\decoreg(\cMloc(\Mbar),\Mbar)\vee{}0 + \gamma\veps^2+\veps}.
    \]
\pref{prop:constrained_offset_union} also gives
    \begin{align*}
      \alpha = \bigoh(\veps + \deccreg[\veps](\cM,\Mbar))
      \leq \bigoh(\veps + \decoreg(\cM,\Mbar)\vee{}0 + \gamma\veps^2)
      \leq \bigoh(\decoreg(\cM,\Mbar)\vee{}0 + \gamma\veps^2 + \gamma^{-1})=\alphaupperabs,
    \end{align*}
    where the second inequality is AM-GM. This establishes the result.
  \end{proof}

\begin{proof}[Proof for \pref{ex:upper_improvement}]%
  
  We first lower bound the quantity in \eqref{eq:upper_old}, then
  prove an upper bound on the regret of \etdp.
\paragraphi{Lower bound on offset \CompShort and regret bound from \eqref{eq:upper_old}}
We start with a basic lower bound on the offset \CompShort for the
class $\cMab$.
\begin{lemma}
  \label{lem:ex_lb}
  Let $\alpha\in(0,1/4)$, $\beta\in(0,1)$ and $A\geq{}2$ be given. For
  all $\gamma>0$,
  \begin{align*}
    \decoreg(\cMab,\Mtil)
    = \decoreg(\cMab_{\alpha}(\Mtil),\Mtil)
    \geq{}
        \frac{\alpha}{2+8\gamma\beta}
    -  4\gamma/A.
  \end{align*}
\end{lemma}
We now prove a lower bound on the quantity
\[
R\ldef{} \min_{\gamma>0}\max\crl[\bigg]{\sup_{\Mbar\in\conv(\cM)}\decoreg(\cMloc[\alphaupper](\Mbar),\Mbar)\cdot{}T,\;
    \gamma\cdot{}\log\abs{\cM}}
\]appearing in
\eqref{eq:upper_old}. We begin by lower bounding the localization
radius
\begin{align*}
  \alphaupper=\bigom\prn*{\frac{\gamma}{T} + \decoreg(\cM)}.
\end{align*}
We choose $\alpha_1=1/2$ and $A=T^2$, so that whenever $T$ is a
sufficiently large constant, \pref{lem:ex_lb} gives
\begin{align*}
  \alphaupper\geq\bigom\prn*{\frac{\gamma}{T} +
               \decoreg(\cM^{\alpha_1,\beta},\Mtil)}
  \geq\bigom\prn*{\frac{\gamma}{T} +
               \frac{1}{1+\gamma\beta} - 4\gamma/A}
  \geq\bigom\prn*{\frac{\gamma}{T} +
  \frac{1}{1+\gamma\beta}}
  \geq{}   \bigom\prn*{
  \sqrt{\frac{1}{\beta{}T}}\wedge{}1
  }.
\end{align*}
It follows that as long as $\beta\geq{}1/T$, if we set
\[
\alpha_2 = c\cdot{}\sqrt{\frac{1}{\beta{}T}},
\]
where $c$ is a sufficiently small numerical constant, then regardless
of how $\gamma$ is chosen,
\begin{align*}
  \cM^{\alpha_2,\beta}(\Mtil)
  \subseteq\cM_{\alphaupper}(\Mtil),
\end{align*}
and
\begin{align*}
  R\geq{}
  \min_{\gamma>0}\max\crl[\bigg]{\decoreg(\cM^{\alpha_2,\beta}(\Mtil),\Mtil)\cdot{}T,\;
    \gamma}.
\end{align*}
Applying \pref{lem:ex_lb} once more, we have
\begin{align*}
  R\geq{}
  \bigom\prn*{\min_{\gamma>0}\crl[\bigg]{\frac{\alpha_2T}{1+\gamma\beta}+
  \gamma}}
  \geq{}  \bigom\prn*{\alpha_2T\wedge\sqrt{\frac{\alpha_2T}{\beta}}
  }
  \geq{} \bigom\prn*{
  \beta^{-1/2}T^{1/2}\wedge{}\beta^{-3/4}T^{1/4}
  }.
\end{align*}
We set $\beta=T^{-1/2}$, which gives
\[
  R \geq{} \bigom(T^{5/8}),
\]
as desired.

\paragraphi{Upper bound on constrained \CompShort and regret of \etdp}
We now bound the regret of \etdp via \pref{thm:regret_upper}. We first bound the constrained \CompShort.
\begin{lemma}
  \label{lem:ex_ub}
  Let $\beta\in(0,1)$ be given, and let $\cMcup\ldef\cup_{\alpha\in(0,1/2]}\cMab$.
  Then for all $\veps>0$, 
  \[
    \deccreg(\cM) \leq{} \deccreg(\cMcup)\leq {\bigoh\prn*{\frac{\veps^2}{\beta}}}.
  \]
\end{lemma}
Let $\beta\propto{}T^{-1/2}$ and $A\propto{}T^{2}$ as in the prequel. %
Plugging the bound from \pref{lem:ex_ub} into
\pref{thm:regret_upper} (cf. \eqref{eq:reg-bound-ugly}) gives
\begin{align*}
  \En\brk*{\RegDM}
  \leq \bigoht\prn*{
  \frac{\vepsupperT^2}{\beta}\cdot{}T%
  +\sqrt{T}}=\bigoht\prn*{
  \sqrt{T}
  },
\end{align*}
since, with the usual choice of estimation oracle, we can take $\vepsupperT\leq\bigoht\prn*{\sqrt{\frac{\log\abs{\cM}}{T}}}$
and $\log\abs{\cM}\leq\bigoh\prn*{\log(A)}\leq\bigoh(\log(T))$.

\end{proof}

\begin{proof}[\pfref{lem:ex_lb}]
  We first remark that $\cMab=\cMab_{\alpha}(\Mtil)$, since
  $\fmtil(\pi)=1/2$ for all $\pi\in\brk{A}$, and all $M\in\cMab$
  have $\fm(\pim)\leq{}1/2+\alpha$.

We now lower bound the value of the offset \CompShort. Consider any
distribution $p\in\Delta(\Pi)$, and let
$i\ldef{}\argmin_{i\in\brk{A}}p(i)$, so that $p(i)\leq{}1/A$. We have
\begin{align*}
  \Enp\brk*{\fmia(\pimia)-\fmia(\pi)}
  \geq \alpha\cdot(1-1/A-p(\picirc)) + \frac{1}{4}p(\picirc)
  \geq{} \frac{\alpha}{2},
\end{align*}
since $\alpha\leq{}1/4$ and $A\geq{}2$. We now bound the Hellinger
distance via

  \begin{align*}
    \En_{\pi\sim{}p}\brk*{\Dhels{\Mia(\pi)}{\Mtil(\pi)}}
    \leq{} p(\picirc)\cdot\Dhels{\Mia(\picirc)}{\Mtil(\picirc)} + 2p(i).
  \end{align*}
  Observe that $\Dhels{\Mia(\picirc)}{\Mtil(\picirc)}\leq{}2\beta$ and
  $p(i)\leq{}1/A$, so that
\begin{align*}
  \En_{\pi\sim{}p}\brk*{\Dhels{\Mia(\pi)}{\Mtil(\pi)}}
  \leq{} 2\beta{}\cdot{}p(\picirc) +  2/A.
\end{align*}

Combining the calculations so far gives
\begin{align*}
  \En_{\pi\sim{}p}\brk*{\fmia(\pimia)-\fmia(\pi)
  - \gamma\cdot{}\Dhels{\Mia(\pi)}{\Mtil(\pi)}}
  \geq{} \frac{\alpha}{2} 
  - 2\gamma\beta{}p(\picirc) -  4\gamma/A.
\end{align*}
On the other hand, by choosing $M=\Mtil$, we have
\begin{align*}
  \En_{\pi\sim{}p}\brk*{\fmtil(\pimtil)-\fmtil(\pi)
  - \gamma\cdot{}\Dhels{\Mtil(\pi)}{\Mtil(\pi)}}
=\frac{1}{2}p(\picirc),
\end{align*}
so that
\begin{align*}
  \decoreg(\cMab,\Mtil)
  &\geq\min_{p\in\Delta(\Pi)}
  \max\crl*{
\frac{\alpha}{2} 
  - 2\gamma\beta{}p(\picirc),
  \frac{1}{2}p(\picirc)
    } -  4\gamma/A,\\
  &\geq
    \frac{\alpha}{2+8\gamma\beta}
    -  4\gamma/A.
\end{align*}

\end{proof}

\begin{proof}[\pfref{lem:ex_ub}]
  Let $\Mbar\in\conv(\cMcup)$ and $\veps\leq{}1/10$ be given. Assume that
  $25\frac{\veps^2}{\beta}\leq{}1/2$, as the result is trivial
  otherwise.

  Let $i=\argmax_{i\in\brk{A}}\bbP_{o\sim{}\Mbar(\pir)}(o=i)$
and set $p=(1-25\frac{\veps^2}{\beta})q + 25\frac{\veps^2}{\beta}\indic_{\pir}$, where
$q\in\Delta(\brk{A})$ is another distribution whose value will be
chosen shortly. We first observe that if $\Mja\in\cMab\subseteq\cMcup$ for some $\alpha>0$
and $j\neq{}i$, then since
  $\bbP_{o\sim\Mbar(\pir)}(o=j)\leq{}\beta/2$, we have
  \begin{align*}
    \Dhels{\Mja(\pir)}{\Mbar(\pir)}
    \geq \prn*{\sqrt{\bbP_{o\sim{}\Mja(\pir)}(o=j)} -
    \sqrt{\bbP_{o\sim\Mbar(\pir)}(o=j)}}^2
    \geq (\sqrt{\beta}-\sqrt{\beta/2})^2\geq\frac{\beta}{20},
  \end{align*}
where we have used that
$\bbP_{o\sim{}\Mja(\picirc)}(o\neq\perp)=\bbP_{o\sim{}\Mbar(\picirc)}(o\neq\perp)=\beta$,
since $\Mbar\in\conv(\cMcup)$. It follows that regardless of how $q\in\Delta(\brk{A})$ is chosen, $\Mja\notin\cH_{p,\veps}(\Mbar)$, since
  \[
\Enp\brk*{\Dhels{\Mja(\pi)}{\Mbar(\pi)}}\geq{}\frac{25\beta}{20\beta}\veps^2 > \veps^2.
\]
Hence, if we define $\cM'_i = \crl{\Mia\in{\cMcup}\mid{}\alpha\in(0,1/2]}$, we
have $\cH_{p,\veps}(\Mbar)\cup\crl{\Mbar}\subseteq\cM'_i\cup\crl{\Mtil}\cup\crl{\Mbar}$,
and it remains to choose $q$ such that the regret on all of these
models is small. We note that
$\Enp\brk*{g\sups{\Mtil}}\leq{}25\frac{\veps^2}{\beta}$ regardless of how $q$
is chosen, so we restrict our attention to $\Mbar$ and $\cM'_i$ going forward.

Let $\Mstar=\argmin_{M\in\cM'_i}\Dhels{M(i)}{\Mbar(i)}$.  We will show that
\begin{equation}
  \label{eq:useful2}
  \fmbar(\pimbar) - \fmbar(i) \leq \hell{\Mstar(i)}{\Mbar(i)}.
\end{equation}
To establish this fact, first note that if $\pimbar = i$, then
\cref{eq:useful2} is immediate. Otherwise, let $\nu\subs{\Mbar} \in
\Delta(\cMcup)$ be such that $\Mbar(\pi) = \E_{M' \sim \nu\subs{\Mbar}}[M'(\pi)]$ for all $\pi \in \Pi$, and then
\begin{align}
\fmbar(\pimbar) - \fmbar(i) \ = \max_{\pi\in\brk{A}}\En_{M'\sim\nu\subs{\Mbar}}\brk*{\fmp(\pi)-\fmp(i)} \leq \frac12 \cdot \BP_{M' \sim \nu\subs{\Mbar}} \left( M' \not\in\MM_i'\right) \leq \frac 12 \cdot \BP_{r \sim \Mbar(i)} \left( r = 1/2 \right)\label{eq:use-nu-mbar},
\end{align}
where the final inequality follows since all models $M' \in \cMcup \backslash \MM_i'$ satisfy $r = 1/2$ a.s.~when $r \sim M'(i)$. Recall the
  elementary fact that for all events $A$ and distributions $\bbP$ and $\bbQ$.
  \begin{align}
    \label{eq:hellinger-triangular} \frac{(\bbP(A)-\bbQ(A))^2}{\bbP(A)+\bbQ(A)}\leq{}2\Dhels{\bbP}{\bbQ}.
  \end{align}
  Since $\Mstar \in \MM_i'$, we have $\BP_{r \sim \Mstar(i)}(r = 1/2) = 0$, and so, using \cref{eq:hellinger-triangular}, it follows that
\begin{align}
  \label{eq:r-mbar-mstar}
  \BP_{r \sim \Mbar(i)}(r = 1/2) \leq 2 \hell{\Mbar(i)}{\Mstar(i)}\nonumber,
\end{align}
and combining with \cref{eq:use-nu-mbar} establishes \cref{eq:useful2}.

To proceed, we choose $q\in\Delta(\brk{A})$ by setting $q(i) =
\frac{4\veps^2}{\Dhels{\Mstar(i)}{\Mbar(i)}}\wedge{}1$, and
$q(\pimbar)=1-q(i)$. We consider two cases
\begin{itemize}
\item If $q(i)=1$, it is immediate that for all $M\in\cM'_i$,
  $\En_{\pi\sim{}p}\brk*{g\sups{M}(\pi)}\leq{}25\frac{\veps^2}{\beta}$. In addition,
  \begin{align*}
\frac{4\veps^2}{\Dhels{\Mstar(i)}{\Mbar(i)}}    \geq{}1,
  \end{align*}
  so \pref{eq:useful2} implies that
  \begin{align*}
    \fmbar(\pimbar) - \fmbar(i)\leq{}{4\vep^2.}  \end{align*}
  It follows that $\En_{\pi\sim{}p}\brk*{\gmbar(\pi)}\leq{}25\frac{\veps^2}{\beta} +
  (\fmbar(\pimbar) - \fmbar(i)) \leq 25\frac{\veps^2}{\beta} + {4\vep^2}$.
  \item If $q(i)<1$, then for all $M\in\cM'_i$,
    \begin{align*}
      \En_{\pi\sim{}p}\brk*{\Dhels{M(\pi)}{\Mbar(\pi)}}
      \geq{} \frac{1}{2}q(i) \Dhels{M(i)}{\Mbar(i)}
      \geq{} \frac{1}{2}q(i) \Dhels{\Mstar(i)}{\Mbar(i)}
      = 2\veps^2,
    \end{align*}
    so $M\notin\cH_{p,\veps}(\Mbar)$. It follows that
    $\cH_{p,\veps}(\Mbar)\cap\cM'_i=\emptyset$. All that remains is to
    bound the regret under $\Mbar$, {which we do as follows:
      \begin{align}
        \E_{\pi \sim p}[\gmbar(\pi)] \leq q(i)(\fmbar(\pimbar) - \fmbar(\pi)) + 25 \frac{\vep^2}{\beta} &\leq q(i) \cdot \hell{\Mstar(i)}{\Mbar(i)} + 25 \frac{\vep^2}{\beta}\nonumber\\
        &=  4 \vep^2 + 25 \frac{\vep^2}{\beta}\nonumber.
      \end{align}
    }

\end{itemize}

Putting the cases above together, we conclude that
    \[
      \sup_{\Mbar\in\conv(\cMcup)}\deccreg({\cMcup}\cup\crl{\Mbar},\Mbar)\leq \bigoh\prn*{\frac{\veps^2}{\beta}}.
    \]

\end{proof}

\end{document}